\documentclass[aos]{imsart}

\RequirePackage{amsthm,amsmath,amsfonts,amssymb}
\RequirePackage[numbers]{natbib}
\RequirePackage[colorlinks,citecolor=blue,urlcolor=blue]{hyperref}
\RequirePackage{graphicx}

\usepackage{dsfont}
\usepackage{amsmath}
\usepackage{natbib}
\usepackage{amsfonts}
\usepackage{amssymb}
\usepackage{mathtools}
\usepackage{graphicx}
\usepackage{float}
\usepackage{wrapfig}
\usepackage[section]{placeins}
\usepackage{caption}
\usepackage{xcolor}
\usepackage{subcaption}
\usepackage{pgfplots}
\ifdefined\DeclareUnicodeCharacter
  \DeclareUnicodeCharacter{2212}{−}
\fi
\usepgfplotslibrary{groupplots,dateplot}
\usetikzlibrary{patterns,shapes.arrows}

\usepackage{algorithm}
\usepackage{algpseudocode}

\usepackage{booktabs}
\usepackage{multirow}

\usepackage{enumerate}
\usepackage{enumitem}
\newlist{inparaenum}{enumerate}{2}
\setlist[inparaenum]{nosep}
\setlist[inparaenum,1]{label=\bfseries\arabic*.}
\setlist[inparaenum,2]{label=\bfseries\roman*)}

\usepackage{url} 

\newcommand{\bX}{\boldsymbol{X}}
\newcommand{\bY}{\boldsymbol{Y}}
\newcommand{\bM}{\boldsymbol{M}}
\newcommand{\bE}{\boldsymbol{E}}

\newcommand{\bA}{\boldsymbol{A}}
\newcommand{\bB}{\boldsymbol{B}}

\newcommand{\bI}{\boldsymbol{I}}
\newcommand{\bU}{\boldsymbol{U}}
\newcommand{\bV}{\boldsymbol{V}}

\newcommand{\bN}{\boldsymbol{N}}
\newcommand{\bS}{\boldsymbol{S}}

\newcommand{\bH}{\boldsymbol{H}}
\newcommand{\bT}{\boldsymbol{T}}

\newcommand{\bG}{\boldsymbol{G}}

\newcommand{\sbeta}{\widehat{\beta}}
\newcommand{\pbeta}{{\beta}^*}

\newcommand{\veps}{\boldsymbol{\varepsilon}}

\newcommand{\hsvt}{\textsf{HSVT}}
\newcommand{\TE}{\textsf{TE}}

\newcommand{\Tr}{{\sf Tr}}

\newcommand{\per}{\textsf{PER}}
\newcommand{\tz}{\tilde{z}}
\newcommand{\tbH}{\widetilde{\bH}}
\newcommand{\Hankel}{\mathsf{H}}
\newcommand{\stackedHankel}{\mathsf{SH}}

\newcommand{\hrho}{\widehat{\rho}}

\newcommand{\cF}{\mathcal{F}}

\newcommand{\cH}{\mathcal{H}}

\newcommand{\cL}{\mathcal{L}}
\newcommand{\cE}{\mathcal{E}}

\newcommand{\mult}{\,\mathsf{x}\,}

\newcommand{\tensor}{{\mathbf T}}
\newcommand{\hightensor}{{\mathbf H}{\mathbf T}}
\newcommand{\Tensor}{{\mathbb T}}
\newcommand{\highTensor}{{\mathbb H}{\mathbb T}}

\newcommand{\sM}{\textsf{M}}
\newcommand{\btdU}{\widetilde{\bU}}
\newcommand{\btdV}{\widetilde{\bV}}

\newcommand{\bhM}{\widehat{\bM}}

\newcommand{\bSigma}{\boldsymbol{\Sigma}}
\newcommand{\btdSigma}{\widetilde{\boldsymbol{\Sigma}}}

\newcommand{\Ex}{\mathbb{E}}
\newcommand{\Pb}{\mathbb{P}}
\newcommand{\Rb}{\mathbb{R}}
\newcommand{\Zb}{\mathbb{Z}}
\newcommand{\Reals}{\mathbb{R}}
\newcommand{\Nb}{\mathbb{N}}
\newcommand{\Gb}{\mathbb{G}}
\newcommand{\Var}{\mathbb{V}\text{ar}}

\newcommand{\imp}{{\sf ImpErr}}
\newcommand{\fore}{{\sf ForErr}}

\newcommand{\oosfore}{{\sf TestForErr}}

\DeclarePairedDelimiter{\norm}{\lVert}{\rVert}
\DeclarePairedDelimiter{\abs}{\lvert}{\rvert}

\newtheorem{property}{Property}

\newcommand{\Page}{\mathsf{P}}

\newcommand{\hPage}{\widehat{\Page}}
\newcommand{\StackedPage}{\mathsf{SP}}
\newcommand{\hStackedPage}{\widehat{\StackedPage}}
\newcommand{\bOne}{\mathbf{1}}
\newcommand{\SP}{\mathsf{SP}}
\newcommand{\SPp}{\mathsf{SP}^\prime}

\newcommand{\hSPp}{\widehat{\SPp}}
\newcommand{\hSPpz}{\widehat{\SPp_0}}
\newcommand{\hSPpi}{\widehat{\SPp_i}}
\newcommand{\hSPpo}{\widehat{\SPp_1}}
\newcommand{\hSPpop}{{\widehat{\SPp_1}^{\perp}}}

\newcommand{\hSPpoi}{{\widehat{\SPp_i}^{\perp}}}

\newcommand{\hf}{\widehat{f}}

\usepackage{tikz}
\usetikzlibrary{positioning}

\usepackage{etoolbox}

\newtoggle{FIGURES}
\toggletrue{FIGURES}

\mathtoolsset{showonlyrefs}

\startlocaldefs
\theoremstyle{plain}

\newtheorem{theorem}{Theorem}[section]
\newtheorem{lemma}[theorem]{Lemma}

\newtheorem{proposition}{Proposition}[section]
\newtheorem*{prop*}{Proposition}
\newtheorem{corollary}{Corollary}[section]
\theoremstyle{remark}
\newtheorem{definition}[theorem]{Definition}

\endlocaldefs

\begin{document}

\begin{frontmatter}
\title{On Multivariate Singular Spectrum Analysis: \\ Tensor and Matrix Variants}
\runtitle{On Multivariate Singular Spectrum Analysis: Tensor and Matrix Variants}

\begin{aug}
\author[A]{\fnms{Anish} \snm{Agarwal}\ead[label=e1]{aa5194@columbia.edu}},
\author[B]{\fnms{Abdullah} \snm{Alomar}\ead[label=e2]{aalomar@mit.edu}}
\and
\author[B]{\fnms{Devavrat} \snm{Shah}\ead[label=e3]{devavrat@mit.edu}}

\address[A]{Department of Industrial Engineering and Operations Research,
Columbia University\\
\printead{e1}}

\address[B]{Department of Electrical Engineering and Computer Science,
Massachusetts Institute of Technology\\
\printead{e2}; \printead{e3}}
\end{aug}

\begin{abstract}
We introduce and analyze two extensions of Singular Spectrum Analysis (SSA) to the multivariate setting: a new variant of the well-known matrix-based method (mSSA), and a novel tensor-based approach (tSSA).
For mSSA, under a spatio-temporal factor model with $N$ time series and $T$ observations per series, we establish that prediction mean-squared-error for both imputation and out-of-sample forecasting effectively scales as $1 / \sqrt{\min(N, T )T}$. 
This improves over: 
(i) $1 /\sqrt{T}$ error scaling of SSA, the univariate restriction of mSSA; 
(ii) $1/\min(N, T)$ error scaling for matrix estimation methods that ignore temporal structure. 
The out-of-sample forecasting result of mSSA could be of independent interest for online learning under a spatio-temporal factor model. 
For tSSA, we characterize its imputation mean-squared-error and showcase its better sample complexity, compared to mSSA,  for certain regimes of $N$ and $T$.
We establish that our spatio-temporal model admits a broad range of
time series dynamics including harmonics, polynomials, differentiable periodic functions, and H\"{o}lder continuous functions.
This is further illustrated via the {\em Hankel Calculus}, which establishes that the set of time series the model represents is closed under component-wise addition and multiplication.
Empirically, on benchmark datasets, mSSA performs competitively with state-of-the-art neural-network time series methods (e.g. DeepAR, LSTM) and significantly outperforms classical methods such as vector autoregression (VAR).
Consistent with our theory, tSSA achieves improved imputation performance over mSSA in certain regimes of $N$ and $T$.
Finally, we introduce and analyze an additional variant of SSA to estimate the time-varying variance of a time series.  
To our knowledge, this is the first result providing provable finite-sample performance guarantees for estimating the time-varying variance of a time series.
\end{abstract}



\end{frontmatter}
\newpage
\tableofcontents
\newpage

\section{Introduction}\label{sec:introduction}
Multivariate time series data is of great interest across many application areas, including cyber-physical systems, finance, retail, healthcare to name a few. 
An important goal across these domains can be summarized as accurate imputation and forecasting of a multivariate time series in the presence of noisy and/or missing data. 

\vspace{2mm}
\noindent {\bf Setup.} 
We consider a discrete time setting with time indexed as $t \in \Zb$. 
For $N \in \Nb$, let the collection $f_n: \Zb \to \Rb, ~n \in [N] := \{1, \dots, N\}$ be the latent time series of interest.
For $t \in [T]$ and $n \in [N]$, we observe $X_n(t)$ where for $\rho \in (0, 1]$,
\footnote{The missingness probability can itself be a time series, i.e., we can have a different $\rho_t$ for each $t$. However, for simplicity, we shall leave it as a constant $\rho$ for all $t$.}
\begin{align}\label{eq:model}
X_n(t) & = 
\begin{cases}
f_n(t) + \eta_n(t) & \mbox{~with~probability~} \rho \\
\star & \mbox{~with~probability~} 1-\rho.
\end{cases}
\end{align}
Here $\star$ represents a missing observation and $\eta_n(t)$ represents the per-step noise,  which we assume to be an independent (across $t, n$) mean-zero random variable.
\footnote{Our results go through if the noise $\eta_n(t)$ is weakly correlated across $t \in [T], n \in [N]$, as long as the spectral norm induced by the stacked Page Matrix of $\eta_n(t)$ has $\| \cdot \|_2$-norm scaling on the order $O(\sqrt{N} + \sqrt{T})$.}
Though $\eta_n(t)$ is independent, we note that the underlying time series, $f_n(\cdot)$, is of course strongly dependent across $t, n$. 
Indeed the presence of per-step noise $\eta_n(t)$ and missing values (denoted by $\star$) represent an additional challenge of measurement error in our setup. 
The generic spatio-temporal factor model for $f_n(\cdot), n \in [N]$ described in Section \ref{sec:ts_model} {\em without} additional noise $\eta_n(\cdot)$ or missingness already provides an expressive model for a time series including any finite sum of products of harmonics and polynomials, any differentiable periodic function, and any H{\"o}lder continuous function.

\vspace{2mm}
\noindent {\bf Goal.} 
Our objective is two-folds, for $n \in [N]$: 
(i) imputation -- estimating $f_n(t)$ for all $t \in [T]$; 
(ii) out-of-sample forecasting -- predicting $f_n(t)$ for $t > T$. 

\subsection{Multivariate Singular Spectrum Analysis}\label{sec:mSSA}
{
Multivariate singular spectrum analysis (mSSA) is a known method to impute and forecast a multivariate time series (see \cite{Broomhead, Plaut, Ghil, mSSA3, mSSA1, mSSA2, bogalo2020understanding}). 
mSSA has been used for both imputation and forecasting, and signal extraction---decomposing a time series into a small number of simpler time series (e.g., periodic, trend, autoregressive component).
However, despite its heavy use in practice, the theoretical properties of mSSA are not well understood. 
Hence, in this work, we  introduce a variant of mSSA for which we provide a rigorous finite-sample analysis of its imputation and out-of-sample forecasting properties; such a finite-sample analysis of mSSA has been missing from the literature. 
We note that we  do not focus on the task of signal extraction which we leave as important future work.
The variant of mSSA we introduce is arguably much simpler to implement than the original mSSA method.
See Figure \ref{fig:algorithm_sketch} for a visual depiction of the key steps in this variant of mSSA\footnote{%
In Section \ref{appendix:lit_review}, we compare the original mSSA method with this variant and discuss key differences.}.

Moreover, existing formulations of mSSA preserve the matrix-based view of univariate SSA, implicitly flattening the multivariate time series.
While effective in practice, this representation may obscure the underlying multi-dimensional structure of the data.
Therefore, we propose and analyze a more natural tensor-based extension, which we refer to as tSSA.
See Figure \ref{fig:algorithm_sketch_tensor} for a visual depiction of the tensor representation used in tSSA.

We will begin by describing the two variants we are proposing in detail.
}

\vspace{2mm}
\noindent {\bf Singular spectrum analysis (SSA).} 
For ease of exposition and to build intuition, we start with $N = 1$, i.e. a univariate time series. 
There are two algorithmic parameters: $1 \leq L \leq T$ and $k \geq 1$.
For simplicity and without loss of generality assume that $T$ is an integer multiple of $L$, i.e. $T/L \in \Nb$ and $k \leq \min(L, T/L)$\footnote{When $T/L \notin \Nb$, by applying both the imputation and forecasting algorithms for two ranges, $1,\dots, \lfloor T/L \rfloor \mult L$ and $(T \mod L) +1, \dots, T$, this condition will be satisfied in each range and will provide imputation and forecasting for all $T$.
Here $\lfloor T/L \rfloor$ refers to the floor of $T/L$.}.
%

%
First, transform the time series $X_1(t), ~t \in [T]$ into an $L \times  T/L $ matrix where the entry of the matrix in row $i \in [L]$ and column $j \in [ T/L ]$ is $X_1(i+(j - 1) \mult L)$. 
This matrix induced by the time series is called the Page matrix, and we denote it as $\Page(X_1, T, L) \in \Rb^{L \times T/L}$. 

\vspace{2mm}
\sloppy\noindent
{\em Imputation.}
After replacing missing values (i.e. $\star$) in the matrix $\Page(X_1, T, L)$ by $0$, we compute its singular value decomposition, which we denote as 
\begin{align}\label{eq:svd.ssa}
\Page(X_1, T, L) & = \sum_{\ell=1}^{\min(L,  T/L )}  s_\ell u_\ell v_\ell^T, 
\end{align}
where $s_1 \geq s_2 \dots \geq s_{\min(L,  T/L )} \geq 0$ denote its ordered singular values, and $u_\ell \in \Reals^{L}, v_\ell \in \Reals^{ T/L }$ denote its left and right singular vectors, respectively, for $\ell \in [\min(L,  T/L )]$. 
Let $\hrho_1$ be the fraction of observed entries of $X_1$, precisely defined as $(\max(1, \sum_{t=1}^T \bOne(X_1(t) \neq \star)))/T$. 
Let the normalized, truncated version of $\Page(X_1, T, L)$ be
\begin{align}\label{eq:hsvt.ssa}
\hPage(X_1, T, L; k) & = \frac{1}{\hrho_1} \sum_{\ell=1}^k s_\ell u_\ell v_\ell^T,
\end{align}
i.e., we perform Hard Singular Value Thresholding (HSVT) with threshold $k$ on $\Page(X_1, T, L)$ to obtain $\hPage(X_1, T, L; k)$.
We then define the {\em de-noised and imputed estimate} of the original time series, denoted by $\hf_1$, as follows: 
for $t \in [T]$, $\hf_1(t)$ equals the entry of $\hPage(X_1, T, L; k)$ in row $(t-1 \mod\, L)+1$ and column $\lceil t/L\rceil$. 
Here $\lceil t/L\rceil$ refers to the ceiling of $t/L$.

\vspace{2mm}
\noindent
{\em Forecasting.}
\noindent
To forecast, we learn a linear model $\hat{\beta}(X_1, T, L; k) \in \Reals^{L-1}$, which is the solution to
\begin{align}\label{eq:ols.ssa}
{\sf minimize} \quad \sum_{m=1}^{ T/L} (y_m - \beta^T x_m)^2 \quad {\sf over} \quad\beta \in \Reals^{L-1},
\end{align}
where $y_m =  (1 / \hrho_1) X_1(L \mult m)$, $x_m = [\hf_1(L\mult (m-1)+1) \dots \hf_1(L \mult (m-1)+L-1)]$ for $m \in [T/L]$.
\footnote{
To establish theoretical results for the forecasting algorithm, we produce estimates $(\hf_1(L\mult (m-1)+1) \dots \hf_1(L\mult (m-1)+L-1))$ for $m \in [T/L]$ by applying the imputation algorithm on $\Page(X_1, T, L)$ {\em after} setting its $L$th row equal to $0$.
Also, $\hrho_1$ in the definition of $y_m$ is computed using only the first $L - 1$ rows of  $\Page(X_1, T, L)$. 
This avoids dependencies in the noise between $y_m$ and $x_m$ for $m \in [T/L]$.}
Note to define $y_m$ we impute missing values in $X_1$ by $0$.
We now describe how to use $\hat{\beta}(X_1, T, L; k)$ to produce both in-sample and out-of-sample forecasts.
(i) In-sample forecast: for time $ t = L \mult m$ and $m \in [T/L]$, the forecast is given by $\bar{f}_1(L \mult m) =  \hat{\beta}(X_1, T, L; k)^T x_m$ .
(ii) Out-of-sample forecast: for $m >T/L$,  i.e., for time $t>T$, the forecast is given by $\bar{f}_1(L \mult m) =  \hat{\beta}(X_1, T, L; k)^T x'_m$ where $x'_m = \frac{1}{\hrho_1}[X_1(L\mult (m-1)+1) \dots X_1(L \mult (m-1)+L-1)]$ after imputing missing values in $X_1$ by $0$.

\vspace{2mm}
\noindent {\bf Matrix-based Multivariate singular spectrum analysis (mSSA).} 
Below we describe the variant of mSSA we propose, which is an extension of the SSA algorithm described above, to when we have a multivariate time series, i.e., $N > 1$. 
The key change is in the first step where we construct the Page matrix---instead of considering the Page matrix of a single time series, we now consider a `stacked' Page matrix, which is obtained by a column-wise concatenation of the Page matrices induced by each time series separately.
{
Therefore, this variant maintains standard matrix-based approach used in prior work~\cite{mSSA3, mSSA1, mSSA2} for multivariate time series. 
}

Like SSA, it has two algorithmic parameters, $L \geq 1$ and $k \geq 1$. 
For each time series, $n \in [N]$, create its $L \times T/L$ Page matrix $\Page(X_n, T, L)$, where the entry in row $i \in [L]$ and column $j \in [ T/L ]$ is $X_n(i+(j - 1) \mult L)$. 
We then create a stacked Page matrix from these $N$ time series by performing a column wise concatenation of the $N$ matrices, $\Page(X_n, T, L), ~n \in [N]$. 
We denote this matrix as $\StackedPage((X_1,\dots, X_N), T, L)$, and note that it has $L$ rows and $N \mult T/L$ columns. 

\vspace{2mm}
\sloppy \noindent
{\em Imputation.}
We replace missing values (i.e. $\star$s) in $\StackedPage((X_1,\dots, X_N), T, L)$ by $0$.
Similar to \eqref{eq:hsvt.ssa}, we perform HSVT on $\StackedPage((X_1,\dots, X_N), T, L)$ and denote its normalized, truncated version as $\hStackedPage((X_1,\dots, X_N), T, L; k)$ (instead of $\hrho_1$, we now normalize by $\hrho \coloneqq (\max(1, \sum_{n=1}^{N}\sum_{t=1}^T  \bOne(X_n(t) \neq \star)))/NT$). 
From $\hStackedPage((X_1,\dots, X_N), T, L; k)$, like in SSA, we can {\em read off} $\hat{f}_n(t)$ for $n \in [N], ~t \in [T]$, the {\em de-noised and imputed estimate} of the $N$ time series over $T$ time steps.
In particular, let $\hPage(X_n, T, L; k)$ refer to sub-matrix of $\hStackedPage((X_1,\dots, X_N), T, L; k)$ induced by selecting only its $[(n-1) \mult (T / L)  +1,  \dots, n \mult T/L$] columns. 
Then for $t \in [T]$, $\hat{f}_n(t)$ equals the entry of $\hPage(X_n, T, L; k)$ in row $(t-1 \mod L)+1$ and column $\lceil t/L\rceil$.

\vspace{2mm}
\noindent
{\em Forecasting.}
Similar to SSA, to forecast, we learn a linear model $\hat{\beta}((X_1, \dots, X_N), T, L; k) \in \Reals^{L-1}$, which is the solution to
\begin{align}\label{eq:ols.mssa}
{\sf minimize} \quad \sum_{m=1}^{ N \mult T/L} (y_m - \beta^T x_m)^2 \quad {\sf over} \quad\beta \in \Reals^{L-1},
\end{align}

\sloppy \noindent where $y_m$ is the $m$th component of $(1 / \hrho) [X_1(L), \ X_1(2\mult L),  \dots, X_1(T), \ X_2(L), \dots, X_2(T), \dots, \allowbreak X_N(T)] \in \Reals^{N \mult T/L}$, and $x_m \in \Reals^{L -1}$ corresponds to the vector formed by the entries of the first $L-1$ rows in the $m$th column of $\hStackedPage((X_1,\dots, X_N), T, L; k)$\footnote{
Similar to the SSA forecasting algorithm, when creating a forecasting model in mSSA, we produce $\hStackedPage((X_1,\dots, X_N), T, L; k)$ by first setting the $L$th row of $\StackedPage((X_1,\dots, X_N), T, L; k)$ equal to zero before performing the SVD and the subsequent truncation.
Also, $\hrho$ in the definition of $y_m$ is computed only using the first $L-1$ rows of  $\StackedPage((X_1,\dots, X_N), T, L; k)$.
} for $m \in [N \mult T/L]$. 
Note, to define $y_m$, we impute missing values in $X_1, \dots, X_n$ by $0$.
(i) In-sample forecast: for time step $t = L \mult m'$ for $m' \in [T/L]$ and for time series $n \in [N]$, the forecast is given by $\bar{f}_n(L \mult m') =  \hat{\beta}((X_1, \dots, X_N), T, L; k)^T x_{m}$ where $m=m' + (n-1) \mult T/L$.
(ii) Out-of-sample forecast: for $m' >T/L$, i.e., for time $t>T$, and for time series $n \in [N]$, the forecast is given by $\bar{f}_n(L \mult m') =  \hat{\beta}((X_1, \dots, X_N), T, L; k)^T x'_{m'}$,  where $x'_{m'} = \frac{1}{\hrho}[X_n(L\mult m' - (L-1)) \dots X_n(L \mult m'-1)]$ after imputing missing values in $X_n$ by $0$.
See Figure \ref{fig:algorithm_sketch} for a visual depiction of the key steps above.

{
\noindent {\bf Tensor singular spectrum analysis (tSSA).} 
Below, we describe the tensor variant of multivariate SSA (tSSA) we propose. 
The key change in this algorithm is that we replace the stacked Page matrix 
representation with the  `Page tensor' representation which we introduce next. 

Given $N$ time series, with observations over $T$ time steps and hyper-parameter $L \geq 1$, define $\tensor \in \Rb^{N \times L \times T/L}$ such that
\begin{align}\label{eq:page_tensor_rep_2}
    \tensor_{n \ell s} & = f_n((s-1)\times L + \ell), ~~n \in [N], ~\ell \in [L], ~s \in [T/L].
\end{align}
The corresponding observation tensor, $\Tensor \in (\Rb \cup \{\star\})^{N \times L \times T/L}$, is
\begin{align}\label{eq:noisy_page_tensor_rep_2}
    \Tensor_{n \ell s} & = X_n((s-1)\times L + \ell), ~~n \in [N], ~\ell \in [L], ~s \in [T/L]. 
\end{align}

See Figure \ref{fig:algorithm_sketch_tensor} for a visual depiction of $\Tensor$.

\begin{figure}
\centering
  \includegraphics[width = 0.6\linewidth]{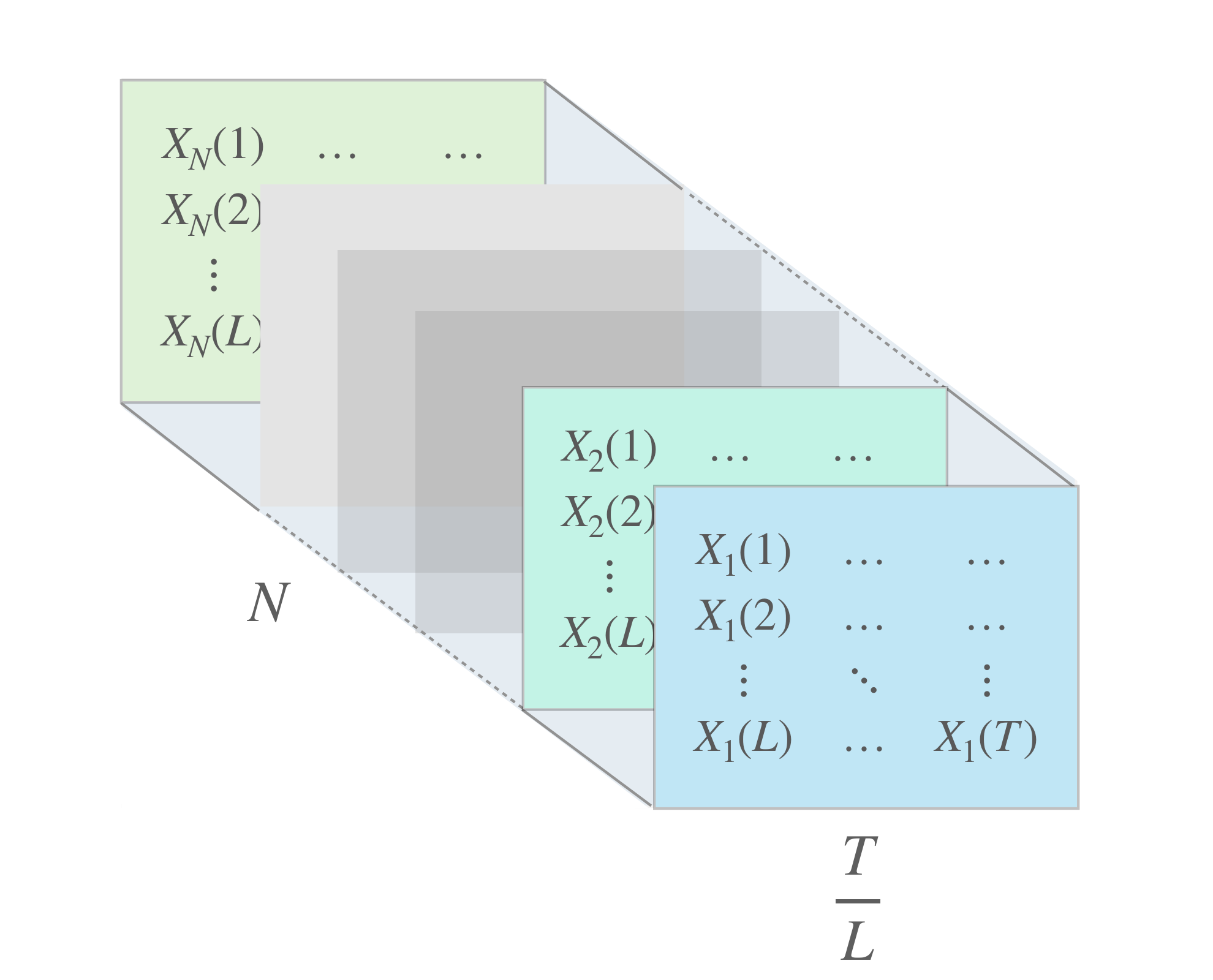}
  \caption{The observations Page tensor.}
      \label{fig:algorithm_sketch_tensor}
      \vspace{-2mm}
\end{figure}

%
\vspace{2mm}
\sloppy\noindent
{\em Imputation.}
This representation and the spatio-temporal model we propose collectively suggest that time series imputation can be reduced to low-rank tensor estimation, i.e., recovering a tensor of low CP-rank\footnote{
The CP-rank of an order-$d$ tensor $\bT \in \Rb^{n_1 \times n_2 \times \dots \times n_d}$ is the smallest value of $r \in \Nb$ such that $\bT_{i_1, \dots, i_d} = \sum^{r}_{k=1} u_{i_1, k} \dots u_{i_d, k}$, where $u_{i_\ell, \cdot}$ are latent factors for $\ell \in [d]$.
} from its noisy, partial observations. 
Over the past decade, the field of low-rank tensor (and matrix) estimation has received great empirical and theoretical interest, leading to a large variety of algorithms including spectral, convex optimization, and nearest neighbor based approaches.
We list a few works which have explicit finite-sample rates for noisy low-rank tensor completion \cite{barak_moitra, xia2018statistically, nonconvex_low_rank_tensor_noisy, yu2020tensor, shah2019iterative}).
As a result, we ``blackbox'' the tensor estimation algorithm used in tSSA as a pivotal subroutine.
Doing so allows one the flexibility to use the tensor estimation algorithm of their choosing within tSSA.
Consequently, as the tensor estimation literature continues to advance, the ``meta-algorithm'' of tSSA will continue to improve in parallel. 
We now define the ``meta'' tSSA imputation algorithm;
the two algorithmic hyper-parameters are $L \ge 1$ (defined in \eqref{eq:noisy_page_tensor_rep_2}) and the order-three tensor estimation algorithm one chooses $\TE_3$ (see Definition \ref{def:tensor_estimation}).
First, using $X_n(t)$ for $n \in [N], t \in [T]$, construct Page tensor $\Tensor$ as in \eqref{eq:noisy_page_tensor_rep_2}.
Second, obtain $\widehat{\tensor}$ as the output of $\TE_3(\Tensor)$ and read off $\hat{f}_n(t)$ by selecting appropriate entry in $\widehat{\tensor}$.

\vspace{2mm}
\noindent
{\em Forecasting.}
\noindent
For forecasting, one can apply the same algorithm as in mSSA, with the modification that each input vector $x_m \in \mathbb{R}^{L-1}$ is formed from the entries $ \left( \widehat{\tensor}_{n,1, i}, \widehat{\tensor}_{n,2, i}, \dots, \widehat{\tensor}_{n,L-1, i} \right)^\top$, where $i = \big((m - 1) \bmod T/L\big) + 1$ and $n = \left\lfloor \frac{mL - 1}{T} \right\rfloor + 1$.
We leave the analysis and empirical evaluation of tSSA forecasting to future work.
}

\vspace{2mm}

{
\noindent \textbf{Parameter selection.} Herein, we give guidance about the selection of the two algorithmic parameters $L$ and $k$ introduced above.
Guided by our finite-sample results in Sections \ref{sec:main_results} and \ref{sec:hankel_calculus}, the optimal value for $L$ in mSSA is $\sqrt{\min(N, T)T}$. 
{
Further, guided by Property \ref{property:te_error_rates}, the optimal value for $L$ in tSSA is $\sqrt{T}$. 
}
On the other hand, we select $k$ in a data-driven manner. In particular, using cross-validation, we either choose $k$ based on the thresholding procedure outlined in [5] or as the minimum number of singular values capturing $ > 90\%$ of its spectral energy. 
See Section \ref{ssec:evidence} and Appendix \ref{ssec:params} for details. 
}

\begin{figure}[h]
    \centering
	\includegraphics[width=0.8\linewidth]{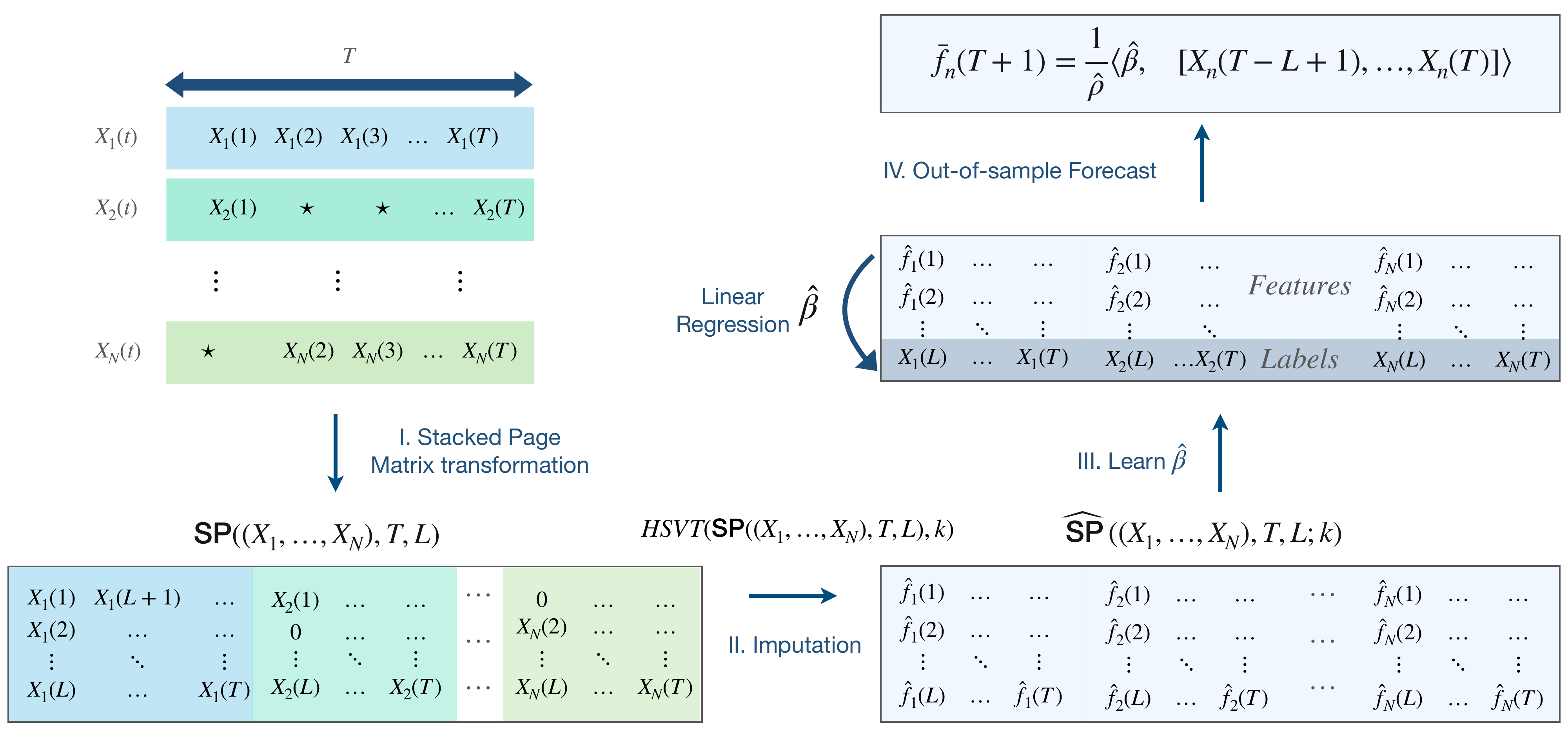}
	\caption{Key steps of our proposed variant of the mSSA algorithm.}
    \label{fig:algorithm_sketch}
    \vspace{-2mm}
\end{figure}

\vspace{2mm}
\noindent {\bf Page vs. Hankel mSSA.} 
See Appendix \ref{sec:hankel_vs_page} for a detailed discussion of the various benefits and drawbacks of the using the Page matrix representation as we propose in both our mSSA variants, instead of the Hankel representation used in the original mSSA.
%

\vspace{2mm}
{
\noindent {\bf Empirical performance.} 
A key question here is how well do our proposed variants (mSSA and tSSA) perform empirically?
In Table \ref{table:stat_nrmse}, we provide a summary comparison of mSSA's performance for
imputation and forecasting on benchmark datasets with respect to state-of-the-art time series algorithms. 
We find that by using the stacked Page matrix in mSSA, it greatly improves performance over SSA; indicating that mSSA is effectively utilizing information {\em across} multiple time series. 
Surprisingly, our variant of mSSA performs competitively or outperforms popular neural network based methods, such as LSTM and DeepAR---we note that these state-of-the-art neural network based methods have no associated theoretical analysis.
Further, it significantly outperforms classical multivariate forecasting methods such as VAR.

For tSSA, we evaluate its effectiveness in imputing/de-noising a multivariate time series relative to mSSA (i.e. the matrix-based approach).
Pleasingly, we find that the empirical results agree with our theoretical analysis, where tSSA outperforms mSSA  in the regime $T \ge N \ge T^{1/3}$, and mSSA performs better when $T \ge N^3$.
See Figure \ref{fig:tssaA} for more details.

Indeed, apart from its use in practice, the empirical performance of both variants strongly motivates a theoretical analysis of when and why mSSA and tSSA work.

}
\begin{table}[h!]
\fontsize{8.9pt}{8.9pt}\selectfont
\tabcolsep=0.07cm
\centering
\caption{mSSA statistically outperforms SSA and other state-of-the-art algorithms, including LSTMs and DeepAR across many datasets. We use the average normalized root mean squared error (NRMSE) as our metric. Details of experiments run to produce { results are in Section \ref{sec:experiments}.}
} 
\label{table:stat_nrmse}
\begin{tabular}{@{}ccccccccccc@{}}
\toprule
        & \multicolumn{5}{c}{\begin{tabular}[c]{@{}c@{}} Mean Imputation \\ (NRMSE)\end{tabular}}                                        & \multicolumn{5}{c}{\begin{tabular}[c]{@{}c@{}}Mean Forecasting  \\ (NRMSE)\end{tabular}}      \\ \cmidrule[0.5pt](lr){2-6} \cmidrule[0.5pt](lr){7-11} 
       & \multicolumn{1}{c}{Electricity} & \multicolumn{1}{c}{Traffic} & \multicolumn{1}{c}{Synthetic} & \multicolumn{1}{c}{Financial}   & \multicolumn{1}{c}{M5}  & \multicolumn{1}{c}{Electricity} & \multicolumn{1}{c}{Traffic} & \multicolumn{1}{c}{Synthetic} & \multicolumn{1}{c}{Financial}  & \multicolumn{1}{c}{M5}                              \\ \midrule
mSSA    &    \textbf{0.398} &   0.508 &       \textbf{0.416} &      \textbf{0.238} &       \textbf{0.883}        &            0.485 &  0.536 &   \textbf{0.281} &  \textbf{0.251} &  \textbf{1.021}                         \\
SSA  &     0.514&  0.713 &       0.675  &  0.467&  0.958        &    0.632 &  0.696 &  0.665 &  0.303 &  1.068 \\
LSTM    & NA                              & NA                          & NA       &                NA       & NA &      0.558 &  0.478 &  0.559 &  1.205 &  1.034 \\
DeepAR  & NA                              & NA                          & NA       &             NA          & NA &   \textbf{0.479} & \textbf{0.464} &   0.415 &  0.316 &  1.050 \\
TRMF &    0.641 &      \textbf{0.460} &       0.564 &  0.430 &  0.916   &  		 0.495 &  0.508 &  0.422 &  0.291 &  1.032 \\
Prophet & NA                              & NA                          & NA       &             NA          & NA   &                    0.569&  0.614 &  1.010 &  1.286 &  1.100 \\    
VAR    & NA                              & NA                          & NA       &             NA          & NA    &     1.291 &  1.092 &  2.987 &  1.218 &  1.120 \\\bottomrule
\end{tabular}
\end{table}

\subsection{Our Contributions}\label{sec:contribution}
As our primary contribution, we provide an answer to the question posed above---under a spatio-temporal factor model that we introduce, the finite-sample analysis we carry out of mSSA's estimation error for imputation and out-of-sample forecasting establishes consistency, as well as its ability to effectively utilize both the spatial and temporal structure in a multivariate time series. 
{
We take this analysis a step further to understand how the tensor-based variant (tSSA) would perform in imputation relative to mSSA.
}
Below, we detail the various aspects of our contribution with respect to the:
(a) spatio-temporal factor model; 
(b) finite sample analysis of mSSA, which leads to both novel and tighter theoretical bounds compared to previous works;
{
(c)  analysis of tSSA that shows a better imputation error convergence rate compared to mSSA for certain relative scalings of $N$ and $T$.
}
(d) algorithmic extensions (and associated theoretical analysis) of mSSA to do time-varying variance estimation, which is the first of the its kind.

{
\vspace{2mm}
\noindent {\bf Spatio-temporal factor model \& Hankel calculus.} 
}
Note that the collection of latent multivariate time series $f_n(t)$, for $n \in [N], ~t \in [T]$ can be collectively viewed as a $N \times T$ matrix. 
To capture the spatial structure, i.e. the relationship across rows, we model this matrix to be low-rank---there exists a low-dimensional latent factor (or feature) associated with each of $N$ time series; analogously, there exists a low-dimensional latent factor associated with each of the $T$ time steps. 
To capture the temporal structure, we further assume that each component of the latent temporal factor has an {\em approximately low-rank Hankel matrix}  representation (see Definition \ref{def:hankel} for the Hankel matrix induced by a time series), i.e., the Hankel---and therefore Page---matrix induced by each component of the latent temporal factor is approximately low-rank. 
This additional structure imposed on the temporal factors is what motivates using the stacked Page matrix representation in mSSA, which is of dimension $L \times (N \mult T/L)$, where $L$ is a hyper-parameter.
We note that for $N=1$ this subsumes the model considered to explain the success of SSA in \cite{SSA_Sigmetrics} as a special case. 

{
As we alluded to earlier, we establish that our factor model is expressive in that it includes harmonics, polynomials, any differentiable periodic function, and any H{\"o}lder continuous function.
In particular, we show that functions in these important and rich classes of time series dynamics have an approximate low-rank Hankel matrix representation. 
We further characterize the representational strength of our spatio-temporal factor model by introducing the {\em Hankel calculus}.  The Hankel calculus establishes that the set of time series that have an approximately low-rank Hankel representation is {\em closed} under component-wise addition and multiplication. 
{
This thorough characterization of our model expands the applicability of our mSSA variant, the proposed tSSA algorithm,  as well as the various SSA/mSSA variants in the literature, which is of interest in its own right. 
}}
 
\vspace{2mm}
\noindent {\bf Finite sample analysis of mSSA.} 
Under the spatio-temporal factor model, we establish that mean squared imputation error {of mSSA} scales as $1/\sqrt{\min(N, T) T}$ (see Theorem \ref{thm:mean_estimation_imputation_simplified}) and the out-of-sample forecasting error scales as $\max(1/\sqrt{NT}, \ N/T^2)$ (see Theorem \ref{thm:mean_estimation_forecasting_simplified_oos},  and Corollary \ref{cor:mean_estimation_forecasting_simplified_oos}).
When $N < T$, the error rate is $1/\sqrt{NT}$. 
When $N > T$, one can simply divide the various time series into sets of size $O(T)$; this will result in a mean squared error rate of $1/T$.
Hence, effectively both the imputation \& forecasting error is of order $1/\sqrt{\min(N, T) T}$. 
For exact details on the relative scaling of $N$ and $T$, please refer to Theorem \ref{thm:mean_estimation_forecasting_simplified_oos}.
For $N = 1$, it implies that the SSA algorithm described above has imputation and forecasting error scaling as $1/\sqrt{T}$. 
That is, mSSA improves performance by a $\sqrt{N}$ factor over SSA by utilizing information across the $N$ time series. 
This also improves upon the prior work of \cite{SSA_Sigmetrics} which established the weaker result that SSA has imputation error scaling as $1/T^{\frac14}$ (i.e., when $N = 1$).
Further \cite{SSA_Sigmetrics} {\em does not} establish a result for the out-of-sample forecasting error of SSA. 
We note that the asymmetry in our finite-sample analysis between $N$ and $T$ is to be expected as we impose further structure on the latent temporal factors; they satisfy a low-rank Hankel representation, which is not assumed of the spatial factors.

Further, existing matrix estimation based methods applied to the $N \times T$ matrix of time series observations (i.e, without first performing the Page matrix transformation as done in mSSA) establish that the imputation prediction error scales as $1/\min(N, T)$.
This is indeed the primary result of the works \cite{TRMF, NIPS2015_5938}, as seen in Theorem 2 of \cite{NIPS2015_5938}\footnote{There seems to be a typo in Corollary 2 of \cite{TRMF} in applying Theorem 2: square in Frobenius-norm error is missing.}. 
That is, while the algorithm stated in \cite{TRMF, NIPS2015_5938} utilizes the temporal structure in addition to the spatial structure, the theoretical guarantees do not reflect it---the guarantees provided by such methods are weaker (since $1/\min(N,T) \geq 1/\sqrt{\min(N, T) T}$) than that obtained by mSSA. 
Again, we emphasize that the existing analysis of SSA and matrix estimation-based methods (for example \cite{SSA_Sigmetrics,TRMF, NIPS2015_5938}) {\em do not establish (finite-sample) bounds for out-of-sample forecasting error}, as is done in this work.

{
\noindent{\bf Introduction and analysis of tSSA.}   
{
As presented earlier, we propose tSSA, a novel tensor variant of SSA, which exploits recent developments in the tensor estimation literature.
}
In Proposition \ref{prop:imputation_comparisons_new}, with respect to imputation error, we characterize the relative performance of tSSA, mSSA, and ``vanilla'' matrix estimation (ME). 
We find that when $N = o(T^{1/3})$, mSSA outperforms tSSA; when $T^{1/3} = o(N), \ N = o(T)$ tSSA outperforms mSSA; when $T = o(N)$, standard matrix estimation methods are equally as effective as mSSA and tSSA.
In addition, as alluded to earlier, we provide empirical results that corroborate this characterization. 
See Figure \ref{fig:regimes} for a graphical depiction of the theoretical characterization.
}

\vspace{2mm}
\noindent{\bf Algorithmic extensions: variance estimation.}   
We extend mSSA to estimate the latent time-varying variance, i.e. $\Ex[\eta_n^2(t)], ~n \in [N], ~t \in [T]$. 
We establish the efficacy of such an extension when the time-varying variance is also modeled through a spatio-temporal factor model. 
To the best of our knowledge, this is the first result that provides provable finite-sample performance guarantees for estimating the time-varying variance of a time series. 

\begin{figure}[h!tb]
\tikzset{every picture/.style={line width=0.75pt}} 
\centering
\begin{tikzpicture}[x=0.75pt,y=0.75pt,yscale=-0.7,xscale=0.7]
\draw [color={rgb, 255:red, 0; green, 0; blue, 0 }  ,draw opacity=1 ][line width=1.5]  (42.09,292.12) -- (400,292.12)(72.04,42.5) -- (72.04,326.23) (393,287.12) -- (400,292.12) -- (393,297.12) (67.04,49.5) -- (72.04,42.5) -- (77.04,49.5)  ;
\draw [color={rgb, 255:red, 179; green, 35; blue, 24 }  ,draw opacity=1 ][line width=1.5]  [dash pattern={on 5.63pt off 4.5pt}]  (72.87,295.03) -- (246.83,130.52) -- (333.6,51.7) ;
\draw [color={rgb, 255:red, 24; green, 124; blue, 179 }  ,draw opacity=1 ][line width=1.5]  [dash pattern={on 5.63pt off 4.5pt}]  (72.87,295.03) -- (242.21,228.38) -- (309.95,200.01) -- (388.51,169.85) ;
\draw  [draw opacity=0][fill={rgb, 255:red, 179; green, 35; blue, 24 }  ,fill opacity=0.05 ] (333.59,51.69) -- (76.73,292.85) -- (388.51,169.85) -- cycle ;
\draw  [draw opacity=0][fill={rgb, 255:red, 24; green, 124; blue, 179 }  ,fill opacity=0.05 ] (387.88,171.5) -- (71.87,294.04) -- (387.65,294.54) -- cycle ;
\draw  [draw opacity=0][fill={rgb, 255:red, 0; green, 0; blue, 0 }  ,fill opacity=0.05 ] (332.6,50.98) -- (71.87,295.03) -- (72.94,50.02) -- cycle ;
\draw (369.45,21.63) node  [font=\small,color={rgb, 255:red, 179; green, 35; blue, 24 }  ,opacity=1 ,rotate=-317.72]  {$N\ =\ T$};
\draw (437.97,145.33) node  [font=\small,color={rgb, 255:red, 15; green, 73; blue, 190 }  ,opacity=1 ,rotate=-336.1]  {$N\ =\ T^{\frac{1}{3}}$};
\draw (366.91,317.72) node  [font=\Large,color={rgb, 255:red, 0; green, 0; blue, 0 }  ,opacity=1 ,rotate=-359.92]  {$T$};
\draw (51.32,97.91) node  [font=\Large,color={rgb, 255:red, 0; green, 0; blue, 0 }  ,opacity=1 ,rotate=-359.92]  {$N$};
\draw (113.42,201.28) node [anchor=north west][inner sep=0.75pt]  [font=\scriptsize,rotate=-316.68] [align=left] {tSSA = mSSA = ME};
\draw (148.36,230.2) node [anchor=north west][inner sep=0.75pt]  [font=\scriptsize,rotate=-330.65] [align=left] {tSSA $\gg$ mSSA $\gg$ ME};
\draw (160.52,260.83) node [anchor=north west][inner sep=0.75pt]  [font=\scriptsize,rotate=-347.83] [align=left] {mSSA $\gg$ tSSA $\gg$ ME};
\end{tikzpicture}
\captionsetup{format=plain}

\caption{Relative effectiveness of tSSA, mSSA, ME for varying $N, T$.}
\label{fig:regimes}
\end{figure}

\vspace{2mm}
\noindent{\bf Summary of contributions.} 
We now briefly summarize our contributions:

\begin{enumerate}
    \item A novel spatio-temporal factor model to analyze mSSA. We show that a large family of time series dynamics fall within our factor model including any differentiable periodic function.
    \item Finite-sample analysis for mSSA's imputation and out-of-sample forecasting. The tools we use for imputation borrow from the existing literature on matrix estimation. However, our out-of-sample forecasting requires novel technical contributions. We believe these tools might be of interest for online learning with a spatio-temporal factor model.{
    \item A novel tensor variant called tSSA, which exploits recent developments in the tensor estimation literature. 
    We find that when $T^{1/3} = o(N)$ and $N = o(T)$, tSSA has better sample complexity compared to mSSA.
    We believe this tensor variant opens a direction to future work to understand the appropriate statistical and computational trade offs for time series analysis.  
    }
    \item A novel time-varying variance estimation algorithm with theoretical guarantees. To the best of our knowledge, neither such an algorithm nor an associated theoretical analysis exists.
\end{enumerate}

\section{Related Work}\label{appendix:lit_review}

Given the ubiquity of multivariate time series analysis, it will not be possible to do justice to the entire literature. 
We focus on a few techniques most relevant to compare against, either theoretically or empirically.

\vspace{2mm}
\noindent {\bf Classical SSA and mSSA.} 
A good overview of the literature on SSA can be found in \cite{SSA_book}.
As alluded to earlier, the original SSA method differs from the variant discussed in \cite{SSA_Sigmetrics} and in this work. 
The key steps of the original SSA method are:
Step 1--create a Hankel matrix from the time series data; 
Step 2--do a Singular Value Decomposition (SVD) of it; 
Step 3--group the singular values based on user belief of the model that generated the process; 
Step 4--perform diagonal averaging to “Hankelize" the grouped rank-1 matrices outputted from the SVD to create a set of time series; and
Step 5--learn a linear model for each “Hankelized" time series for the purpose of forecasting.
The theoretical analysis of this original SSA method has been focused on proving that many univariate time series have a low-rank Hankel representation, and secondly on defining sufficient {\em asymptotic} conditions for when the singular values of the various time series components are separable, thereby justifying Step 3 of the method.
Step 3 of the original SSA method requires user input and Steps 4 and 5 are not robust to noise and missing values due to the strong dependence across entries of the Hankel representation of the time series. 
To overcome these limitations, in \cite{SSA_Sigmetrics} a simpler and practically useful version as described in Section \ref{sec:mSSA} was introduced.
The original mSSA method, like the original SSA method, involves the five steps described above, but first the Hankel matrices induced by each of the $N$ time series are stacked either column-wise (horizontal mSSA) or row-wise (vertical mSSA); see \cite{hassani2018singular}.

We note given the popularity of mSSA, there are many algorithmic variants of it proposed in the literature motivated by different applications: see \cite{Broomhead, Plaut, Ghil, mSSA3, mSSA1, mSSA2, bogalo2020understanding}. 
A significant focus of these works is signal extraction, i.e., decomposing the observed time series into a small number of simpler time series (e.g., periodic, trend, autoregressive component);
these extracted signals are then subsequently utilized for imputation and forecasting as described in the preceding paragraph.
As stated earlier, despite the popularity of the mSSA framework, a rigorous finite-sample analysis of its imputation and out-of-sample forecasting properties are missing in the literature; the challenge in such an analysis is exacerbated with missing data and measurement error.
In this work, as described in Section \ref{sec:mSSA}, we introduce a simpler variant of mSSA that uses the Page instead of the Hankel matrix representation.
This variant is simpler as it focuses only on the task of imputation and forecasting, and not signal extraction.
We do a finite-sample analysis of our variant of mSSA and establish its consistency with respect to imputation and forecasting, which so far has been missing from the mSSA literature. 
In Appendix \ref{sec:hankel_vs_page}, we compare our variant to the original version of mSSA which uses the Hankel matrix, both with respect to their theoretical and practical properties.

{
\noindent {\bf SSA variant in \cite{SSA_Sigmetrics}.}  
As we have alluded to earlier, our work extends the simpler and more practical variant of SSA introduced in \cite{SSA_Sigmetrics} and the analysis therein to the case of multivariate time series. 
However, even with respect to the univariate case, our work improves upon the analysis of \cite{SSA_Sigmetrics} in three substantial ways. %
First, it provides a stronger bounds ($1/\sqrt{T}$) for the imputation prediction error of SSA (i.e. in the univariate case $N =1$) compared to the weaker result established in \cite{SSA_Sigmetrics} which has imputation error scaling as $1/T^{\frac14}$.
Second, \cite{SSA_Sigmetrics} {\em does not} establish a result for the out-of-sample forecasting error of SSA as done in this work. 
Third, our rich spatio-temporal model expands the applicability of SSA (and mSSA) to new classes of time series dynamics. 
For example, this work establishes that smooth periodic functions admit an approximate low-rank Hankel matrix representation. See Proposition \ref{def:C_k_smoothness}  for a concrete definition.  
We further expand the richness of the model, and subsequently the applicability of SSA and mSSA, by introducing the Hankel calculus, which establishes that the set of time series with an approximately low-rank Hankel representation is {\em closed} under component-wise addition and multiplication.
}

\vspace{2mm}
{
\noindent {\bf Principal component regression (PCR).}
Two recent works on the analysis of PCR are closely related to our work. 
The first of these works (\cite{PCR_NeurIPS}) analyzes PCR under the random i.i.d. covariate design where both the in- and out-of-sample covariates are given during the training phase. 
Under this setting, \cite{PCR_NeurIPS} establishes bounds on the out-of-sample prediction error which decays as
$1/\sqrt{n}$, where $n$ is the number of training samples. 
A subsequent work of \cite{agarwal2020principal} analyzes PCR in a setting that is more closely related to ours. 
Specifically, \cite{agarwal2020principal} considers the {\em fixed design}  setting and assumes the more typical supervised learning setup where out-of-sample covariates are unavailable during training. 
Under this setting, \cite{agarwal2020principal} establishes a better decay rate for the out-of-sample prediction error ($1/n$), and proves that PCR correctly identifies the model.

It is worth noting that while our setup is similar to that of \cite{agarwal2020principal}, it differs in a few crucial ways.
This results in a few unique theoretical challenges, which we address in our analysis. 
In particular, the setup we consider differs in three ways (see the differences illustrated in Figure \ref{fig:pcr_vs_mssa}).  
\begin{figure}
    \centering
    \includegraphics[width=0.8\linewidth]{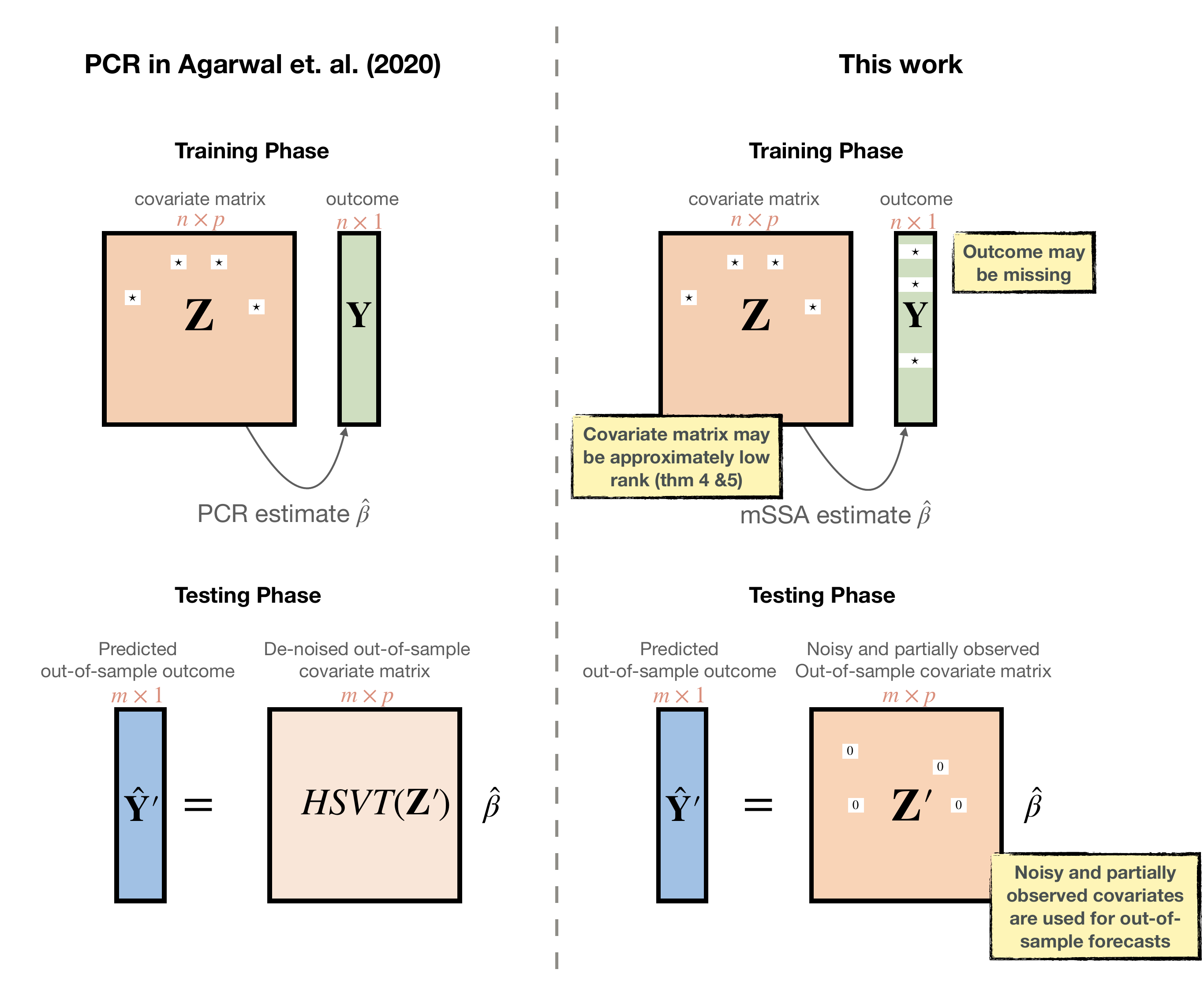}
    \caption{Our setup differs from that of \cite{agarwal2020principal} in three crucial ways, which results in a few unique theoretical challenges }
    \label{fig:pcr_vs_mssa}
\end{figure}

\vspace{2mm}
\noindent  \textit{ (i) No access to the full out-of-sample covariate matrix when forecasting.} The algorithm in \cite{agarwal2020principal} makes use of the fact that they have full access to the out-of-sample covariates matrix (therein denoted by $\boldsymbol{Z}^{\prime}$) when making out-of-sample predictions. Specifically, HSVT is applied to the out-of-sample covariate matrix to produce a de-noised estimate. Then, the predictions are produced using the de-noised estimate of the covariates.
In contrast, recall that in our out-of-sample forecasting problem, we forecast $f_n(t')$ having access to only the past observations $\{X_n(t)\}_{n\in[N], ~t<t'}$. That is, we only have partial access to the out-of-sample covariates matrix.
Since we cannot guarantee a reasonable estimate of the out-of-sample covariates using HSVT with partial access to the matrix,  our out-of-sample forecasts will be produced {\em directly from the noisy covariates}.
Deriving bounds for the out-of-sample forecasting error with such an estimate is a unique challenge we address in this work. 

\vspace{2mm}
\noindent  \textit{(ii) The outcome in the training data may be missing.} While \cite{agarwal2020principal} does not address this possibility, which is natural in the settings of time series analysis, we address the scenario where the outcome may be missing in the training data. 

\vspace{2mm}
\noindent  \textit{(iii) The covariate matrix
may be approximately low-rank.} In Theorems \ref{thm:mean_estimation_imputation_generalized} and \ref{thm:mean_estimation_forecasting_generalized} we generalize Theorems \ref{thm:mean_estimation_imputation_simplified} and \ref{thm:mean_estimation_forecasting_simplified}  to the extended spatio-temporal model where the Hankel/Page matrices are only approximately low-rank. This setting is not addressed in \cite{agarwal2020principal}.  

}

\vspace{2mm}
\noindent {\bf Matrix factorization based methods for multivariate time series.}
There is a rich line of work in econometrics, signal processing, and statistics on viewing multiple time series as a matrix, and where some form of matrix factorization is performed to learn the spatial and temporal factors induced by the matrix; such models have also been called dynamic factor models.
Some representative papers (and by no means exhaustive) include \cite{stock_watson, Forni_Hallin_Lippi_Reichlin, Hallin_Marc_Liska, Doz_Giannone_Reichlin, Marta_Modungo, barigozzi2019quasi, deistler2008generalized, anderson2012autoregressive}.
\cite{stock_watson}  consider the estimation by principal components of this $N \times T$ matrix. 
They use the model for signal extraction and forecasting. 
Also, they proposed an expectation-maximization (EM) algorithm to handle missing data and imputation.
\cite{Forni_Hallin_Lippi_Reichlin, Hallin_Marc_Liska} also estimate principal components and restrict the singular vectors to be related to the Fourier basis.
\cite{Doz_Giannone_Reichlin, barigozzi2019quasi} consider maximum likelihood estimation based on Kalman filtering and also consider forecasting and signal extraction. 
\cite{Marta_Modungo} show how to handle missing data and imputation.
Similar to the mSSA literature, the general focus of these works has been signal extraction, which can then be subsequently used for imputation and forecasting. 
The theoretical analysis of these methods has generally been asymptotic in nature, and has focused on recovery of the spatial and temporal factors, i.e., signal extraction.
Our work complements this literature as we focus directly on finite-sample analysis for imputation and out-of-sample forecasting (without the need for signal extraction), and establish consistency for the variant of mSSA we propose.
To the best of our knowledge, finite-sample consistency results such as ours are limited in the literature.

Additionally, there is a recent line of work from the machine learning literature which also employs matrix factorization based methods (see \cite{Matrix_Time2, TRMF}).
Most such methods make strong prior model assumptions on the underlying time series and the algorithm changes based on the assumptions made on the time series dynamics that generated the data. 
Further, finite sample analysis, especially with respect to forecasting error, of such methods is usually lacking. 
We highlight one method, Temporal Regularized Matrix Factorization (TRMF) (see \cite{TRMF}), which we directly compare against due to its popularity, and as it achieves state-of-the-art empirical imputation and forecasting performance. 
The authors in \cite{TRMF} provide finite sample imputation analysis for an instance of the model 
considered in this work, but forecasting analysis is absent. 
As discussed earlier, they establish that imputation error scales as $1/\min(N, T)$. 
This is a direct consequence of the low-rank structure of the original $N \times T$ matrix. 
But they fail to utilize, at least in the theoretical analysis, the temporal structure. 
Indeed, our analysis captures such temporal structure and hence our imputation error scales as $1/\sqrt{\min(N, T)T}$ which is a stronger guarantee. 
For example, for $N =\Theta(1)$, their error bound remains $\Theta(1)$ for any $T$, suggesting that TRMF \cite{TRMF} fails to utilize the temporal structure for better estimation, while the error for mSSA would vanish as $T$ grows. 

\vspace{2mm}
\noindent {\bf Other relevant literature.} 
We take a brief note of some popular time series methods in the recent literature. 
In particular, recently neural network (NN) based approaches have been popular and empirically effective.
Some industry standard neural network methods include LSTMs, from the Keras library (a standard NN library, see \cite{Keras}) and DeepAR (an industry leading NN library for time series analysis, see \cite{DeepAR}).
Though they have no theoretical guarantees, which is the focus of our work, we compare with them empirically. 

\section{Model}\label{sec:ts_model}
\subsection{Spatio-Temporal Factor Model}\label{sec:ts_model.defn} 

Below, we introduce the spatio-temporal factor model we use to explain the success of mSSA. 
In short, the model requires that the underlying latent multivariate time series satisfies Properties \ref{prop:low_rank_mean_} and \ref{prop:low_rank_mean_hankel}, which capture the ``spatial'' and ``temporal'' structure within it, respectively.

\vspace{2mm}
\noindent {\bf Spatial structure in data.}
Consider the matrix $\bM \in \Rb^{N\times T}$, where its entry in row $n$ 
and column $t$, $\bM_{nt}$ is equal to $f_n(t)$, the value of the latent time series $n$ at time $t$. 
We posit that the matrix $\bM$ is low-rank. Precisely, 
\begin{property}\label{prop:low_rank_mean_}
Let $\text{rank}(\bM) = R$. 
That is, for any $n \in [N], t \in [T]$,
$
\bM_{nt}  = \sum^{R}_{r=1} U_{nr} \ W_{rt},
$
where $| U_{nr} | \le \Gamma_1$, $| W_{rt} | \le \Gamma_2$ for constants $\Gamma_1, \Gamma_2 > 0$.
\end{property}
Property \ref{prop:low_rank_mean_} effectively captures the ``spatial'' structure amongst the $N$ time series.
Similar to the dynamic factor model literature, we can interpret this model as there existing $R$ latent time series $W_{r \cdot}$ for $r \in [R]$, and each time series $f_n(\cdot)$ is a linear combination of these $R$ time series, where the weights are given by $U_{n \cdot}$.
{\em We emphasize that though the latent factors $U_{n \cdot}$ and $W_{\cdot t}$, which determine $f_n(t)$ are bounded, the observed time series $X_n(t) = f_n(t) + \eta_n(t)$ can be be unbounded, as $\eta_n(t)$ is only assumed to be a sub-Gaussian random variable.}

\vspace{2mm}
\noindent {\bf Temporal structure in data.}
To explicitly capture the temporal structure in the data, we impose additional structure on $W_{r \cdot}$.
To that end, we introduce the notion of the Hankel matrix induced by a time series.
\begin{definition}[Hankel Matrix]\label{def:hankel}
Given a time series $g: \Zb \to \Rb$, its Hankel matrix associated with observations over $T$ time steps, $\{1,\dots, T\}$, is given by the matrix $H \in \Rb^{\lfloor T/2 \rfloor \times \lfloor T/2 \rfloor}$ with $H_{ij} = g(i+j-1)$ for $i, j \in [\lfloor T/2 \rfloor]$. 
\end{definition}
Now, for a given $r \in [R]$, consider the time series $W_{rt}$ for $t \in [T]$. 
Let $H(r) \in \Rb^{\lfloor T/2 \rfloor \times \lfloor T/2 \rfloor}$ denote its Hankel matrix restricted to $[T]$, i.e. $H(r)_{ij} = W_{r (i+j-1)}$ for $i, j \in [\lfloor T/2 \rfloor]$. 
\begin{property}\label{prop:low_rank_mean_hankel}
For each $r \in [R]$ and for any $T \ge 1$, the Hankel Matrix $H(r) \in \Rb^{\lfloor T/2 \rfloor \times \lfloor T/2 \rfloor}$ associated with time series $W_{rt}, \ t \in [T]$ has rank at most $G$. 
\end{property}
Property \ref{prop:low_rank_mean_hankel} captures the temporal structure within the latent factors associated with time; indeed, such a low-rank Hankel representation includes a rich family of time series dynamics as noted in Proposition \ref{prop:lowrank_LRF_example} below.
\begin{proposition}[Proposition 5.2, \cite{SSA_Sigmetrics}]
\label{prop:lowrank_LRF_example}
Consider a time series $f: \mathbb{Z} \to \Rb$ with its element at time $t$ denoted as
\begin{align}\label{eq:ex.poly.harmonic.exp}
f(t)& = \sum_{a=1}^A \exp(\alpha_a t) \cdot \cos(2\pi \omega_a t + \phi_a) \cdot P_{m_a}(t),
\end{align}
where $\alpha_a, \omega_a, \phi_a \in \Rb$ are parameters, $P_{m_a}$ is a degree 
$m_a \in \mathbb{N}$ polynomial in $t$. 
Then $f(\cdot)$ satisfies Property \ref{prop:low_rank_mean_hankel}.
In particular, consider the Hankel matrix of $f$ over $[T]$, denoted as $H(f) \in \Rb^{\lfloor T/2 \rfloor \times \lfloor T/2 \rfloor}$ with $H(f)_{ij} = f(i+j-1)$ for $i, j \in [\lfloor T/2 \rfloor]$. 
For any $T$, the rank of $H(f)$ is at most $G = A(m_{\max} + 1)(m_{\max} + 2)$, where 
$m_{\max} = \max_{a \in A} m_a$.
\end{proposition}
Proposition \ref{prop:lowrank_LRF_example} states any finite sum of (products of) harmonics, polynomials, and exponentials has a low-rank Hankel representation.
Each of these functions are popular to model various aspects of a time series such as periodicity and trend.
Further, we note that the spectral representation of generic stationary processes, which includes autoregressive processes, implies that {\em any} sample-path of a stationary process can be decomposed into a weighted sum (precisely an integral) 
of harmonics, where the weights in the sum are sample path dependent---see Property 4.1, Chapter 4 of \cite{VAR}. 
That is, a finite (weighted) sum of harmonics provides a good model representation for stationary processes with the model becoming more expressive as the number of harmonics grows. 
In Section~\ref{sec:hankel_calculus}, we extend this model when Property  \ref{prop:low_rank_mean_hankel} is only approximately satisfied. 
In particular, we quantify the approximation error based on the smoothness of the underlying time series and the number of harmonics used in the summation to approximate it.
 
\vspace{2mm}
\noindent {\bf Spatio-temporal model implies stacked Page matrix is low-rank.}
Recall that the primary representation utilized by mSSA, as described in Section \ref{sec:mSSA}, is the stacked Page matrix (with parameter $L$). 
Observe that the Page matrix of a univariate time series for any $L \leq \lfloor T/2 \rfloor$ is simply the sub-matrix of the associated Hankel matrix: 
precisely, the Page matrix can be obtained by restricting to the top $L$ rows and columns $1, L+1, \dots$ of the Hankel matrix. 
Therefore, the rank of the Hankel matrix is a bound on the rank of the Page matrix. 
%
Under the spatio-temporal factor model satisfying Properties \ref{prop:low_rank_mean_} and \ref{prop:low_rank_mean_hankel}, we establish the following low-rank property of the Page 
matrix of any particular time series as well as that of the stacked Page matrix. 
\begin{proposition} \label{prop:flattened_mean_low_rank_representation}
Let Properties \ref{prop:low_rank_mean_} and  \ref{prop:low_rank_mean_hankel} hold. 
Then for any $L \leq \lfloor T/2 \rfloor$ with any $T \geq 1$, the rank of the Page matrix induced by the univariate time series $f_n(\cdot)$ for $n \in [N]$ is at most $R \mult G$. 
Further, the rank of the stacked Page matrix induced by all $N$ time series $f_1(\cdot), \dots, f_N(\cdot)$ is also at most $R \mult G$. 
\end{proposition}
The proof is in Appendix \ref{sec:proof_prop_flattened_page_low_rank} where a more general result is established in Proposition \ref{prop:approx_flattened_mean_low_rank_representation}.   

{
\noindent {\bf Spatio-temporal model implies Page tensor has low CP-rank.}
Similarly,  recall that the primary representation utilized by tSSA, as described in Section \ref{sec:mSSA}, is the Page tensor representation (see Figure \ref{fig:algorithm_sketch_tensor}). 
Let the CP-rank of an order-$d$ tensor $\bT \in \Rb^{n_1 \times n_2 \times \dots \times n_d}$ be the smallest value of $r \in \Nb$ such that $\bT_{i_1, \dots, i_d} = \sum^{r}_{k=1} u_{i_1, k} \dots u_{i_d, k}$, where $u_{i_\ell, \cdot}$ are latent factors for $\ell \in [d]$.
Recall the definitions of $\tensor$  in \eqref{eq:page_tensor_rep_2},
Then, under the model described above, we have the following property. 
\begin{proposition}\label{prop:tensor}
Let Properties \ref{prop:low_rank_mean_} and \ref{prop:low_rank_mean_hankel} hold. 
Then, for any $1 \leq L \leq \sqrt{T}$, $\tensor$, the Page tensor induced by time series $f_1(\cdot), \dots, f_N(\cdot)$,  has canonical polyadic (CP)-rank
at most $R \mult G$. 
%
%
\end{proposition}
Proof of Proposition \ref{prop:tensor} can be found in Appendix \ref{sec:tSSA_proofs}.
}

\subsection{A Diagnostic Test for the Spatio-Temporal Model}\label{ssec:evidence} 
In Sections \ref{sec:main_results} and \ref{sec:hankel_calculus}, under the model described above, we theoretically establish the efficacy of mSSA.  
Beyond this model though, our work does not provide any guarantees for mSSA. 
Therefore, to utilize the guarantees of this work, it would be useful to have a data-driven diagnostic test that can help identify scenarios when the model of Section \ref{sec:ts_model} may or may not hold. 
We discuss one such test in this section.

In particular, Proposition  \ref{prop:flattened_mean_low_rank_representation} suggests a ``data driven diagnosis test'' to verify whether mSSA is likely to succeed as per the results of this work. 
Specifically, if the (effective) rank---defined as the minimum number of singular values capturing $ > 90\%$ of its spectral energy---of the Page matrix associated with any of the univariate components $f_n(\cdot)$ and the (effective) rank of stacked Page matrix associated with the multivariate time series with $N$ component are {\em very different, then mSSA may not be effective} compared to SSA, but if they are {\em very similar then mSSA is likely to be more effective} compared to SSA. 
Our finite-sample results in Sections \ref{sec:main_results} and \ref{sec:hankel_calculus} indicate that the optimal value for $L$ is $\sqrt{\min(N, T)T}$.
Thus as a further test, if the effective rank of the stacked Page matrix does not scale much slower than $L$ for $L \sim \sqrt{\min(N, T)T}$, then SSA (and mSSA) are unlikely to be effective methods. 

\vspace{2mm}
\noindent
Table \ref{table:rank} compares the (effective) rank of the stacked Page matrices for different benchmark time series data sets.  
The value of $T$ equals {$3993$, $26304$, and $10560$} for the Financial, Electricity, and Traffic 
datasets respectively (see Appendix \ref{appendix:experiments} for details on the datasets).  
We set $L = \lfloor\sqrt{\min(N, T)T}\rfloor $ for all datasets. When $N =1$, this corresponds to  $L$ equals {$63$, $162$,  and $102$}  for the Financial, Electricity, and Traffic datasets respectively. 
Table \ref{table:rank} shows the effective rank in each dataset as we vary $N$. 
As can be seen, for $N = 1$, the effective rank is much smaller than $L$ (or $T$) suggesting that SSA is likely to be effective. 
For Electricity and Financial datasets, the rank does not change by much as we increase $N$. 
However, relatively the rank does increase substantially for the Traffic dataset. 
{This might explain why mSSA is relatively less effective for the Traffic dataset in contrast to the Financial and Electricity datasets as noted in Table \ref{table:stat_nrmse}.  }

\begin{table}[!htb]
\centering
\caption{ 
	Effective rank of stacked Page matrix across benchmarks as we vary $N$. 
	%
	}
\label{table:rank}
\begin{tabular}{@{}lllll@{}}
\toprule
Dataset     & N = 1 & N =10 & N = 100 & N = 350 \\ \midrule
Electricity & 19 & 37 & 44   & 31   \\
Financial   & 1  & 3 & 3  & 6     \\
Traffic     & 14 & 32 & 69  & 116 \\ \bottomrule
\end{tabular}
\end{table}

\section{Main Results}\label{sec:main_results}
We now provide bounds on the imputation and forecasting prediction error under the spatio-temporal model introduced in Section \ref{sec:ts_model}. 
We start by defining the metric by which we measure prediction error. 
For imputation, we define prediction error as
\begin{align}\label{eq:imp.error}
\imp(N, T) & = \frac{1}{NT} \sum_{n=1}^N \sum_{t=1}^T \Ex\big[ (f_n(t) - \hat{f}_n(t))^2\big].
\end{align}
Here, the imputed estimate $\hat{f}_n(\cdot), ~n \in [N]$ are produced by the imputation algorithm of Section \ref{sec:mSSA}.  
For forecasting, we define the in-sample prediction error as
\begin{align}\label{eq:fore.error}
\fore(N, T, L) & = \frac{L}{NT} \sum_{n=1}^N \sum_{m'=1}^{T/L} \Ex\big[ (f_n(L \mult m') - \bar{f}_n(L\mult m'))^2\big].
\end{align}
Further,  let $T_1 \in \Zb$ such that $T_1 \ge L$. Then, we define the out-of-sample prediction error  as  
\begin{align}\label{eq:fore.error_oos}
\oosfore(N, T ,T_1, L) & = \frac{L}{NT_1} \sum_{n=1}^N \sum_{m'=1}^{T_1/L} \Ex\big[ (f_n(T + L \mult m') - \bar{f}_n(T +  L\mult m'))^2\big].
\end{align}
Again, the forecasted estimate $\bar{f}_n(\cdot), ~n \in [N]$ are produced by the forecasting algorithm of Section \ref{sec:mSSA}. 
In \eqref{eq:imp.error}, \eqref{eq:fore.error}, and \eqref{eq:fore.error_oos}, the expectation is with respect to the randomness in observations due to noise and missingness. 

{\noindent \textbf{Remark.} Note that the two forecasting error metrics we use measure the accuracy of only a subset of the forecasts one would naturally be interested in. For example, $\oosfore(N, T ,T_1, L)$ only measures the forecasting error for the $NT_1/L$ entries $\{\bar{f}_n\left(T+L \times m^{\prime}\right) \}_{n \in [N], m \in [T_1/L]}$. 
However,  with a slight change to the algorithm, one can use the same analysis to establish a similar bound for all $NT_1$ forecasts. 
Specifically, the algorithm is modified such that  $L$ linear forecaster $\sbeta_1, \dots, \sbeta_L$ are learned. To learn the $\ell$-th forecaster, mSSA is applied on the stacked Page matrix induced by the time series observations starting from time step $\ell$.   
In this algorithm, the forecast $\bar{f}_n(T+t')$ for $t' \in [T_1]$ is made by $\sbeta_{((T+t') \mod L)+ 1}$. 
Lastly, while this analysis would imply that learning $L$ linear forecasters is needed, our empirical results strongly suggest that learning one $\beta$ is sufficient. 
Thus, we believe that requiring $L$ linear forecasters is a consequence of our proof technique, and is not a fundamental limitation to mSSA. 
}

\subsection{Assumptions}
To state the main results, we make the following assumptions. 
Recall from \eqref{eq:model} that for each $n \in [N]$ and $t \in [T]$, we observe $f_n(t) + \eta_n(t)$ with probability $\rho \in (0,1]$ independently.
We shall assume that noise $\eta_n(\cdot), n \in [N]$ satisfy the following property.
\begin{property} \label{prop:bounded_noise}
For $n \in [N], t \in [T]$, $\eta_{n}(t)$ are independent sub-gaussian random variables, with $\Ex[\eta_{n}(t)] = 0$ and
 $\| \eta_{n}(t) \|_{\psi_2} \le \gamma$.
\end{property}
Recall that the  $\|\cdot\|_{\psi_\alpha}$ denotes the sub-gaussian norm, which is defined as follows: for a random variable $X$, the sub-gaussian norm is $\|X\|_{\psi_{2}}=\inf \left\{t>0: \mathbb{E} \exp \left(X^{2} / t^{2}\right) \leq 2\right\}$, see \cite{vershynin2010introduction} for more details.
{Note that while independent noise (across $t$ and $n$) is assumed, the results in this paper can indeed be extended to setting with time-dependent noises. This is already evident in subsequent work \cite{samossa} that builds upon our analysis to establish bounds when the noise is a stable autoregressive process. 
}.
\begin{property}\label{property:spectra}
\textbf{(Balanced spectra)}. Denote the $L \times (NT/L)$ stacked Page matrix associated with all $N$ time series $f_1(\cdot), \dots, f_N(\cdot)$ as $\SP(f) \coloneqq \StackedPage((f_1,\dots, f_N), T, L)$.
Under the setup of Proposition \ref{prop:flattened_mean_low_rank_representation}, $\text{rank}(\SP(f)) = \ell \geq 1$ and $\ell \leq  R \mult G$. 
Then, for 
%
$L = \sqrt{\min(N, T) T}$,
$\SP(f)$ is such that $\sigma_\ell(\SP(f)) \geq c \sqrt{NT}/\sqrt{\ell}$ for some absolute constant $c > 0$, where $\sigma_\ell$ is the $\ell$-th largest singular value of $\SP(f)$.
\end{property} 
Note that if $\sigma_\ell(\SP(f)) = \Theta(\sigma_1(\SP(f)))$, then one can verify that Property \ref{property:spectra} holds.
Indeed, assuming that the non-zero singular values are `well-balanced' is standard in the matrix/tensor estimation literature.
To state our results for out-of-sample forecasting error, let $\SP_1(f)$ be the $L \times (NT_1/L)$ stacked Page matrix associated with all $N$ time series $f_1(t), \dots, f_N(t)$ entries for $t \in [T+1, T+T_1]$. 
We assume an analogous condition on  $\SP_1(f)$ as we do for $\SP(f)$.
\begin{property}\label{property:spectra2}
\textbf{(Balanced spectra (out-of-sample))}. 
Under the setup of Proposition \ref{prop:flattened_mean_low_rank_representation}, we have that $\text{rank}(\SP_1(f)) = \ell \geq 1$ and $\ell \leq  R \mult G$. 
Then, for 
$L = \sqrt{\min(N, T) T}$,
$\SP_1(f)$ is such that $\sigma_\ell(\SP_1(f)) \geq c \sqrt{NT_1}/\sqrt{\ell}$ for some absolute constant $c > 0$, where $\sigma_\ell$ is the $\ell$-th largest singular value of $\SP_1(f)$.
\end{property}  
Again, note that if $\sigma_\ell(\SP_1(f)) = \Theta(\sigma_1(\SP_1(f)))$, then one can verify that Property \ref{property:spectra2} holds.

Lastly, we shall impose some restrictions on the complexity of the $N$ time series $f_1(t), \dots, f_N(t)$ for $t>T$.
Let $\SPp(f)$ denote the $(L-1) \times (NT/L)$ matrix formed using the top $L-1$ rows of $\SP(f)$.
Define $\SPp_1(f)$ analogously with respect to $\SP_1(f)$.
Let $\text{colspan}(\SPp(f))$ and $\text{colspan}(\SPp_1(f))$ denote the subspace of $\Rb^{L-1}$ spanned by the columns of $\SPp(f)$ and $\SPp_1(f)$, respectively. 
We assume the following property.
\begin{property}\label{property:subspaceinclusion}
\textbf{(Subspace inclusion)}. 
$\text{colspan}(\SPp_1(f))  \subseteq \text{colspan}(\SPp(f))$. 
\end{property} 
Intuitively, this requires that to effectively forecast, the associated stacked Page matrix of the out-of-sample time series $\text{colspan}(\SPp_1(f))$ is only as ``rich'' as that of $\SPp(f)$, its in-sample analog.

\noindent {\em Picking hyper-parameter $L$.}
The proof of Theorems \ref{thm:mean_estimation_imputation_simplified}, \ref{thm:mean_estimation_forecasting_simplified}, and \ref{thm:mean_estimation_forecasting_simplified_oos} imply the optimal choice of $L$ is to set it to $\sqrt{\min(N, T) T}$.
Intuitively, this choice of $L$ leads to the stacked Page matrix $\SP(f)$ to be as square as possible, and our analysis implies that the error rate is inversely proportional to the minimum of the number of the rows and columns of $\SP(f)$.
Hence, for the remainder of the paper, we state our results for $L = \sqrt{\min(N, T) T}$.

\noindent {\em Picking hyper-parameter $k$.}
For our theoretical result, we assume that we pick $k = \ell$, where $\ell$ is the rank of $\SP(f)$.
Empirically, we pick $k$ to equal the ``effective rank'' of the observed Page matrix as defined in Section \ref{ssec:evidence}.

\subsection{Finite-sample Analysis for Imputation and Forecasting for mSSA}
{
Now we state the main results for the matrix-based variant (mSSA). 
}
In what follows, we let $C(c, \Gamma_1, \Gamma_2, \gamma)$ denote a constant that depends only (polynomially) on model parameters $c, \Gamma_1, \Gamma_2, \gamma$.
We also remind the reader that $R, \Gamma_1, \Gamma_2$ are defined in Property \ref{prop:low_rank_mean_}, $G$ in \ref{prop:low_rank_mean_hankel}, $\gamma$ in Property \ref{prop:bounded_noise} and $c$ in Property \ref{property:spectra}.

\vspace{2mm}
\noindent {\bf Imputation.}
We begin with our imputation result.
\begin{theorem}[Imputation]\label{thm:mean_estimation_imputation_simplified}
Let Properties \ref{prop:low_rank_mean_}, \ref{prop:low_rank_mean_hankel}, \ref{prop:bounded_noise} and \ref{property:spectra} hold. 
For a large enough absolute constant $C > 0$, let $\rho \geq C \frac{\log NT}{\sqrt{NT}}$. 
Then with hyper-parameters $L =\sqrt{\min(N, T) T}$ and $k = \ell$, 
\begin{equation}
\begin{aligned}
\imp(N, T) & \leq 
C(c, \Gamma_1, \Gamma_2, \gamma)
\bigg(
\frac{R^{3}G \log NT}{\rho^{4} \sqrt{\min(N, T) T}} 
\bigg).
\end{aligned}
\end{equation} 
\end{theorem}

\vspace{2mm}
\noindent {\bf In-sample forecasting.}
Recall from \eqref{eq:ols.mssa} that in mSSA, we learn a linear model between the last row of $\StackedPage((X_1,\dots, X_N), T, L)$ and the $L-1$ rows above it (after de-noising the sub-matrix induced these $L - 1$ rows via HSVT). 
Hence, we first establish that in the idealized scenario (no noise, no missing values), there does indeed exist a linear model between the last row and the $L-1$ rows above of $\SP(f)$.  
Let $\SP(f)_{L\cdot}$ denote the $L$-th row of $\SP(f)$ and recall $\SPp(f) \in \Rb^{(L-1)\times (NT/L)}$ denotes the sub-matrix of $\SP(f)$ formed by selecting top $L-1$ rows.
In the proposition below, we show there exists a linear relationship between $\SP(f)_{L \cdot}$ and $\SPp(f)$.
\begin{proposition}\label{prop:exact_low_rank_linear}
Let Properties \ref{prop:low_rank_mean_}  and \ref{prop:low_rank_mean_hankel} hold. 
Then there exists $\pbeta \in \Rb^{L-1}$ such that 
$
\SP(f)_{L \cdot}^T = {\SPp(f)}^T \pbeta.
$
Further, $\| \pbeta \|_0 \le R G$.
\end{proposition}
\begin{theorem}[In-sample forecasting]\label{thm:mean_estimation_forecasting_simplified}
Let the conditions of Theorem \ref{thm:mean_estimation_imputation_simplified} hold. 
Then, with $\beta^{*}$ defined in Proposition \ref{prop:exact_low_rank_linear}, we have
\begin{equation}
\begin{aligned}
\fore(N, T, L) 
& \leq 
C(c, \gamma, \Gamma_1, \Gamma_2) 
\max(1,  \|\pbeta\|^2 _1)
\Big(
\frac{R^{3}G \log NT}{\rho^{4} \sqrt{\min(N, T) T}} 
\Big).
\end{aligned}
\end{equation}
\end{theorem}

\vspace{2mm}
\noindent {\bf Out-of-sample forecasting.}
\begin{theorem}[Out-of-sample Forecasting]\label{thm:mean_estimation_forecasting_simplified_oos}

Let Properties \ref{prop:low_rank_mean_}, \ref{prop:low_rank_mean_hankel}, \ref{prop:bounded_noise}, \ref{property:spectra}, \ref{property:spectra2}, and \ref{property:subspaceinclusion} hold. 
Let the hyper-parameters $L =\sqrt{\min(N, T) T}$ and $k = \ell$.
For a large enough absolute constant $C > 0$, let $\rho \geq C \max\Big(\frac{\log NT}{\sqrt{NT}}, (\gamma +R \Gamma_1\Gamma_2)  \sqrt{\frac{RG}{L}} \Big) $.
%
%
Then, with  $\beta^{*}$ defined in Proposition \ref{prop:exact_low_rank_linear}, we have
\begin{equation}
\begin{aligned}
\oosfore(N, T,T_1, L)
 &\leq  C(\gamma, \Gamma_1, \Gamma_2, c)  \max(1,  \|\pbeta\|_1^2) \\
 &\qquad \times
 \Bigg( \frac{R^9  G^3  \log ({N\max(T, T_1)})}{\rho^4  \sqrt{\min(N, T) T} }   \left(\max\left(1, \frac{N}{T}\right)  + \frac{T}{T_1}\right) \Bigg).
\end{aligned}
\end{equation}
\end{theorem}

\begin{corollary} \label{cor:mean_estimation_forecasting_simplified_oos}
Let the conditions of Theorem \ref{thm:mean_estimation_forecasting_simplified_oos} hold. 
Then, with $ T_1  = \Theta(T)$, we have
\begin{equation}
\begin{aligned}
\oosfore(N, T,T_1, L) 
 &\leq  C(\gamma, \Gamma_1, \Gamma_2, c)  \max(1,  \|\pbeta\|_1^2)
 \Bigg( \frac{R^9  G^3  \log ({NT}) \max\left(1, \frac{N}{T}\right)}{\rho^4  \sqrt{\min(N, T) T} }   \Bigg).
\end{aligned}
\end{equation}
\end{corollary}
Corollary \ref{cor:mean_estimation_forecasting_simplified_oos} implies that when $N = o(T)$, then the error scales as $\sim 1/\sqrt{NT}$. 
When $T = o(N)$, then one can simply divide the $N$ time series up into sets of size $T$. 
Corollary \ref{cor:mean_estimation_forecasting_simplified_oos} implies that this will result in  error scaling as $\sim 1/T$. 
Thus effectively, the error rate scales as $\sim 1/\sqrt{\min(N, T)T}$.

We note that Theorems \ref{thm:mean_estimation_imputation_simplified}, \ref{thm:mean_estimation_forecasting_simplified}  and Proposition \ref{prop:exact_low_rank_linear}
are special cases of Theorems \ref{thm:mean_estimation_imputation_generalized}, \ref{thm:mean_estimation_forecasting_generalized}  and Proposition \ref{prop:approx_low_rank_linear} stated in the next section, respectively.
Their proofs are in Appendices \ref{sec:proof_mean_estimation_imputation}, \ref{appendix:forecasting}, and \ref{sec:prop:flattened_mean__approx_linear_representation}, respectively. 
The proof of Theorem \ref{thm:mean_estimation_forecasting_simplified_oos} is in Appendix \ref{appendix:oos_forecast}.

{
\noindent \textbf{Remark.}  While we do not provide lower bounds for the three metrics above, we believe that our bounds are tight with respect to $N$ and $T$ when $N \geq T$.
Consider the following example, for $n \in[N], t\in[T]$ and $\theta_1, \dots, \theta_N \in \Rb$, let
$f_n(t) = \theta_n$. As in our setup, assume that we observe the noisy observation 
$X_n(t) = f_n(t) + \eta_n(t)$, 
where for $n \in [N], t \in [T]$, $\eta_{n}(t)$ are independent mean-zero Gaussian random variables. 
Both imputation and forecasting questions boil down to estimating  $\theta_n$ for $n \in[N]$.
The maximum likelihood estimate $\hat{\theta}_n$ of the parameter $\theta_n$ is the empirical mean of $ X_n(1), \dots, X_n(T)$. 
Indeed, one can easily show that for all $n \in [N]$, $\Ex[(\theta_n -\hat{\theta}_n)^2] \sim 1/T$. This is exactly the imputation (and forecasting) scaling w.r.t to $N$ and $T$ that we find in our settings when $N \ge T$. Note that $f_n(t)$ can be represented by our spatio-temporal model with $R = G = 1$.
We consider establishing formal lower bounds for all three metrics of interest an interesting direction for future work.
}

{
\subsection{Analysis of tSSA: Impact of Tensor Representation}

First, recall that the algorithm as described in Section \ref{sec:mSSA} make use of a black-box tensor estimation algorithm therein referred to as $\TE_3$.
For our analysis, we will start with a definition of such tensor estimation algorithms for a generic order-$d$ tensor.
\begin{definition}[Matrix/Tensor Estimation]\label{def:tensor_estimation}
For $d \ge 2$, denote $\TE_d: \{\star, \Rb\}^{n_1 \times n_2 \times \dots n_d} \to \Rb^{n_1 \times n_2 \times \dots n_d}$ as an order-$d$ tensor estimation algorithm. 
It takes as input an order-$d$ tensor $\Gb$ with noisy, missing entries, where $\Ex[\Gb] = \rho \bG$ and $\rho \in (0, 1]$ is the probability of each entry in $\Gb$ being observed.  
$\TE_d$ then outputs an estimate of $\bG$ denoted as $\widehat{\bG} =  \TE_d(\Gb)$.
\end{definition}
We assume the following `oracle' error convergence rate for $\TE_d$; for ease of exposition, we restrict our attention to the setting where $\rho = 1$.
\begin{property}\label{property:te_error_rates}
For $d \ge 2$, assume $\TE_d$ satisfies the following: the estimate $\widehat{\bG} \in \Rb^{n_1 \times n_2 \times \dots n_d}$, which is the output of $TE_d(\Gb)$ with $\Ex[\Gb] = \bG$, satisfies
\begin{flalign}
&\frac{1}{n_1 \dots n_d} \|\widehat{\bG} - \bG \|^{2}_F = \tilde{\Theta}\left(1 / \min(n_1, \dots, n_d)^{\lceil d/2 \rceil} \right).
\end{flalign}
Here, $\tilde{\Theta}(\cdot)$ 
suppresses dependence on noise, i.e., $\bE = \Gb - \Ex[\Gb]$, $\log(\cdot)$ factors, and CP-rank of $\bG$.
\end{property}
Property \ref{property:te_error_rates} holds for a variety of matrix/tensor estimation algorithms.
For $d = 2$, it holds for HSVT as we establish in the proof of Theorem \ref{thm:mean_estimation_imputation_simplified} for mSSA of $\tilde{O}(1 / \sqrt{\min(N, T), T})$.
It is straightforward to show that this is the best rate achievable for $\TE_2$.
For $d \ge 3$, it has recently been shown that Property \ref{property:te_error_rates} provably holds for a spectral gradient descent based algorithm  \cite{nonconvex_low_rank_tensor_noisy} (see Corollary 1.5 of \cite{nonconvex_low_rank_tensor_noisy}), conditioned on certain standard ``incoherence'' conditions imposed on the latent factors of $\bG$; another spectral algorithm that achieved the same rate was furnished in \cite{xia2018statistically}, which the authors also establish is minimax optimal.

Finally, to apply Property \ref{property:te_error_rates} in our context, we first establish the following routine property.
\begin{proposition}\label{prop:tesnor_exp}
Let Property \ref{prop:bounded_noise} hold, and $\Tensor$ and $\tensor$ be defined as in \eqref{eq:noisy_page_tensor_rep_2} and \eqref{eq:page_tensor_rep_2}, respectively.
Then, all entries of $\Tensor$ are independent random variables with each entry observed with probability $\rho \in (0,1]$, and $\Ex[\Tensor] = \rho \tensor$. 
\end{proposition}

\noindent{\bf Algorithmic comparison: tSSA vs. mSSA vs. ME.} 
We now provide a unified view of tSSA, mSSA, and ``vanilla'' ME (which we describe below) to do time series imputation.
All three methods have two key steps: 
(i) data transformation -- converting the observations $X_n(t)$ into a particular data representation/structure; 
(ii) de-noising-- applying some form of matrix/tensor estimation to de-noise the constructed data representation.
\vspace{-2mm}
\begin{itemize}[leftmargin=*,noitemsep]
	\item tSSA --  
		using $X_n(t)$, create the Page tensor $\Tensor \in \Rb^{N \times L \times T/L}$ as in \eqref{eq:noisy_page_tensor_rep_2};
		apply $\TE_3(\Tensor)$ to get $\widehat{\tensor}$ (e.g. using the method in \cite{nonconvex_low_rank_tensor_noisy}); 
		read off $\hat{f}_n(t)$ by selecting appropriate entry in $\widehat{\tensor}$.
	\item  mSSA -- 
		using $X_n(t)$, create the stacked Page matrix $\StackedPage((X_1,\dots, X_N), T, L) \in \Rb^{L \times (N \mult T/L)}$ as detailed in Section \ref{sec:mSSA};
		apply $\TE_2(\StackedPage((X_1,\dots, X_N), T, L))$ to get $\hStackedPage((X_1,\dots, X_N), T, L)$ (where we use HSVT for $\TE_2(\cdot)$); 
		read off $\hat{f}_n(t)$ by selecting appropriate entry in $\hStackedPage((X_1,\dots, X_N), T, L)$.
	\item ME --  
		using $X_n(t)$, create $\bX \in \Rb^{N \times T}$, where $\bX_{nt}$ is equal to $X_n(t)$; 
		apply $\TE_2(\bX)$ (e.g. using HSVT as in mSSA) to get $\widehat{\bX}$; 
		read off $\hat{f}_n(t)$ by selecting appropriate entry in $\widehat{\bX}$.
\end{itemize}
This perspective also suggests that one can use any ``blackbox'' matrix estimation routine to de-noise the constructed stacked Page matrix in mSSA; HSVT is one such choice that we analyze.

\vspace{2mm}
\noindent{\bf Theoretical comparison: tSSA vs. mSSA vs. ME.} 
We now do a theoretical comparison of the relative effectiveness of tSSA, mSSA, and ME in imputing a multivariate time series $X_n(t)$ for $n \in [N], t \in [T]$, as we vary $N$ and $T$.
To that end, let $\imp(N,T; \text{tSSA})$, $\imp(N,T; \text{mSSA})$, and $\imp(N,T; \text{ME})$ denote the imputation error for tSSA, mSSA, and ME, respectively.
\begin{proposition}\label{prop:imputation_comparisons_new}
For tSSA and mSSA, pick hyper-parameter $L = \sqrt{T}, L = \sqrt{\min(N, T)T}$, respectively.
Let Property \ref{property:te_error_rates} hold.
Then,
{\small
\begin{itemize}[leftmargin=5mm,noitemsep]
\item[(i)] $T = o(N)$: $\imp(N,T; \text{tSSA}), \ \imp(N,T; \text{mSSA}) = \tilde{\Theta}(\imp(N,T; \text{ME}))$;
\item[(ii)] $T^{1/3} = o(N), \ N=o(T)$: $\imp(N,T; \text{tSSA}) = \tilde{o}(\imp(N,T; \text{mSSA}))$, $\imp(N,T; \text{mSSA}) = \tilde{o}(\imp(N,T; \text{ME})$;
\item[(iii)] $N = o(T^{1/3})$: $\imp(N,T; \text{mSSA}) = \tilde{o}(\imp(N,T; \text{tSSA}))$, $ \imp(N,T; \text{tSSA}) = \tilde{o}(\imp(N,T; \text{ME}))$, 
 \end{itemize}
}
%
	%
where $\tilde{o}(\cdot)$, $\tilde{\Theta}(\cdot)$ suppresses dependence on noise parameters, CP-rank, poly-logarithmic factors.
\end{proposition}
We note given Property \ref{property:te_error_rates}, $L = \sqrt{T}$  is optimal for tSSA and $L = \sqrt{\min(N, T)T}$ is optimal for mSSA.

We also remind the reader that,  as implied by Properties \ref{prop:low_rank_mean_} and \ref{prop:low_rank_mean_hankel}, the rank (or CP-rank) of the matrix/tensor constructed in the three methods above is fixed as we increase $N$ and $T$. 
If, for example, Property \ref{prop:low_rank_mean_hankel} does not hold, and $G$ grows with $T$, e.g. $G = \Theta(T)$, then the characterization above does not hold. 
Indeed, in this case, ME will be more effective than both tSSA and mSSA (with the choice $L = \sqrt{\min(N, T)T}$). This is to expected as the Page matrix of any time series could be full rank when Property \ref{prop:low_rank_mean_hankel} does not hold. 

See Figure \ref{fig:regimes} in Section \ref{sec:introduction} for a graphical depiction of the different regimes in Proposition \ref{prop:imputation_comparisons_new}. 
Proofs of Proposition \ref{prop:tesnor_exp} and \ref{prop:imputation_comparisons_new} can be found in Appendix \ref{sec:tSSA_proofs}.

}

\section{Approximate Low-Rank Hankel Representation}\label{sec:hankel_calculus}
In this section, we extend the model presented in Section \ref{sec:ts_model} by relaxing Property \ref{prop:low_rank_mean_hankel} to only hold approximately. 
We establish a `calculus' for this extended model -- the set of time series functions which have this approximate low-rank Hankel representation is closed under component-wise addition and multiplication.
We show important examples of time series dynamics studied in the literature have an approximate low-rank Hankel representation.
Lastly, we present generalizations of Theorems \ref{thm:mean_estimation_imputation_simplified} and \ref{thm:mean_estimation_forecasting_simplified} for this extended model. 

\subsection{Approximate Low-rank Hankel Representation and Hankel Calculus}
We first introduce the definition of the approximate rank of a matrix. 

\begin{definition}[$\epsilon$-approximate rank] A matrix $M \in \mathbb{R}^{a \times b}$ is said to have   $\epsilon$-approximate rank at most $s \geq 1$, for a given $0<\epsilon\leq\norm{M}_{\infty}$, if there exists a rank $s$ matrix $M_s \in \mathbb{R}^{a \times b}$ such that $\left\|M-M_s\right\|_{\infty}\leq \epsilon$.
\end{definition}

\begin{definition}[$(G, \epsilon)$-Hankel Time Series]\label{def:hankel.aprx}
For a given $\epsilon \geq 0$ and $G \geq 1$, a time series $f: \Zb \to \Rb$ is called a $(G, \epsilon)$-Hankel time series if for any $T \geq 1$, its Hankel matrix has $\epsilon$-approximate rank $G$. 
\end{definition}
We extend the model of Section \ref{sec:ts_model} by replacing Property \ref{prop:low_rank_mean_hankel} by the following.
\begin{property}\label{prop:approx_low_rank_mean_hankel}
For each $r \in [R]$ and for any $T \ge 1$, the Hankel Matrix $H(r) \in \Rb^{\lfloor T/2 \rfloor \times \lfloor T/2 \rfloor}$ associated with time series $W_{rt}, ~t \in [T]$ has $\epsilon$-approximate rank at most $G$ for $\epsilon > 0$. 
That is, for each $r \in [R]$, $W_{r\cdot}$ is a $(G, \epsilon)$-Hankel time series.
\end{property}
We state an implication of the above stated properties on the stacked Page matrix.
\begin{proposition} \label{prop:approx_flattened_mean_low_rank_representation}
Let Properties \ref{prop:low_rank_mean_} and  \ref{prop:approx_low_rank_mean_hankel} hold. 
For any $L \leq \lfloor T/2 \rfloor$ with any $T \geq 1$, the stacked Page matrix induced by the $N$ time series $f_1(\cdot), \dots, f_N(\cdot)$ has $\epsilon'$-rank at most $R \mult G$ for $\epsilon' = R \Gamma_1 \epsilon$. 
\end{proposition}
%

\vspace{2mm}
\noindent {\bf Hankel calculus.}
We present a key property of the model class satisfying Property \ref{prop:approx_low_rank_mean_hankel}, i.e. time series that have an approximate low-rank Hankel matrix representation. 
To that end, we define `addition' and `multiplication' for time series.  
Given two time series $f_1, ~f_2: \Zb \to \Rb$, define their addition, denoted $f_1 + f_2: \Zb \to \Rb$ as
$
(f_1 + f_2)(t)  = f_1(t) + f_2(t),
$
for all $t \in \Zb$. 
Similarly, their multiplication, denoted $f_1 \circ f_2: \Zb \to \Rb$ as 
$
(f_1 \circ f_2)(t)  = f_1(t) \mult f_2(t), ~\mbox{for~all}~t \in \Zb.
$
Now, we state a key property for the model class satisfying Property \ref{prop:approx_low_rank_mean_hankel} (proof in Appendix \ref{appendix:proof_ts_model_hankel_examples}).
\begin{proposition}\label{prop:Hankel_algebra}
For $i \in \{1, 2\}$, let $f_i$ be a $(G_i, \epsilon_i)$-Hankel time series for $G_i \geq 1, ~\epsilon_i \geq 0$. 
Then,  $f_1 + f_2$ is a $(G_1 + G_2, \epsilon_1 + \epsilon_2)$-Hankel time series and
$f_1 \circ f_2$ is a $\Big( G_1 G_2, 3 \max(\epsilon_1, \epsilon_2) \cdot \max(\norm{f_1}_\infty, \norm{f_2}_\infty) \Big)$-Hankel time series.
\end{proposition}

\subsection{Examples of $(G, \epsilon)$-Hankel Time Series}\label{ssec:examples}

We establish that many important classes of time series dynamics studied in the literature are instances of $(G,\epsilon)$-Hankel time series, i.e. they satisfy Property \ref{prop:approx_low_rank_mean_hankel}. 
In particular, any differentiable periodic function (Proposition \ref{prop:C_k_smoothness_TS}), and any time series with a H\"older continuous latent variable representation (Proposition \ref{prop:LVM_implies_Hankel}).
%
%
Proofs of Propositions  \ref{prop:low_rank_Hankel_matrices}, \ref{prop:C_k_smoothness_TS}, and \ref{prop:LVM_implies_Hankel} can be found in Appendix \ref{appendix:proof_ts_model_hankel_examples}. 

\medskip
\noindent{\bf Example 1. $(G, \epsilon)$-LRF time series.} 
We start by defining a linear recurrent formula (LRF), which is a standard model for linear time-invariant systems.
\begin{definition}[$(G, \epsilon)$-LRF]\label{def:LRF}
For $G \in \Nb$ and $\epsilon \geq 0$, a time series $f$ is said to be a $(G, \epsilon)$-Linear Recurrent Formula (LRF) if for all $T \in \Zb$ and $t \in [T]$, there exists $g: \Zb \to \Rb$ such that  
\[
	f(t) = g(t) + h(t),
\]
where for all $ t \in \Zb$, (i) $g(t) = \sum_{l = 1}^{G} \alpha_l g(t-l)$ with constants $\alpha_1, \dots, \alpha_G$, and  (ii) $|h(t)| \le \epsilon$. 
\end{definition}
Now we establish a time series $f$ that is a $(G, \epsilon)$-LRF is also $(G, \epsilon)$-Hankel. 
\begin{proposition}\label{prop:low_rank_Hankel_matrices}
If $f$ is $(G, \epsilon)$-LRF representable, then it is $(G, \epsilon)$-Hankel representable.
\end{proposition}
LRF's cover a broad class of time series functions, including any finite sum of products of harmonics, polynomials and exponentials. 
In particular, it can be easily verified that a time series described by \eqref{eq:ex.poly.harmonic.exp} is a $(G, 0)$-LRF, where $G \le A(m_{\max} + 1)(m_{\max} + 2)$ with $m_{\max} = \max_{a \in A} m_a$.

\medskip
\noindent{\bf Example 2. ``smooth'' and periodic time series.} 
We establish that any differentiable periodic function is $(G, \epsilon)$-LRF and hence $(G, \epsilon)$-Hankel for appropriate choices of $G$ and $\epsilon$. 
\begin{definition}[$C^k(R, \per)$]\label{def:C_k_smoothness}
For $k \geq 1$ and $R > 0$,  we use $C^k(R, \per)$ to denote the class of all time series $f: \Rb \to \Rb$ such that it is $R$ periodic, i.e. $f(t+R) = f(t)$ for all $t \in \Rb$ and the $k$-th derivative of $f$, denoted $f^{(k)}$, exists and is continuous.
\end{definition}
%
%
\begin{proposition}\label{prop:C_k_smoothness_TS}
Any $f \in C^k(R, \per)$ is
$$
\Big(4G, C(k, R) \frac{\norm{f^{(k)}} }{G^{k - 0.5}}\Big)-\text{Hankel representable},
$$
for any $G \geq 1$.  Here $C(k, R)$ is a term that depends only on $k, R$ and 
$\norm{f^{(k)}}^2 = \frac{1}{R} \int_0^{R} (f^{(k)}(t))^2 dt$.
\end{proposition}

\medskip
\noindent{\bf Example 3. time series with latent variable model (LVM) structure.} 
We now show that if a time series has a LVM representation, and the latent function is H\"older continuous, then it has a $(G, \epsilon)$-Hankel representation for appropriate choice of $G \geq 1$ and $\epsilon \geq 0$. 
We first define the H\"older class of functions; this class of functions is widely adopted in the non-parametric regression literature \cite{non_Param_stats}. 
Given a function $g: [0, 1)^K \to \Rb$, and a multi-index $\kappa \in \Nb^K$, let the partial derivate of $g$ at $x \in [0, 1)^K$, if it exists, be denoted as $\triangledown_\kappa g(x) = \frac{\partial^{|\kappa|} g(x)}{(\partial x)^\kappa}$ {where $|\kappa| = \sum_{j=1}^K \kappa_j$ and $(\partial x)^\kappa = \partial^{\kappa_1} x_1 \cdots \partial^{\kappa_K} x_K$.}
\begin{definition}[\textbf{$(\alpha, \cL)$-H\"older Class}]\label{def:holder}
Given $\alpha, \cL > 0$, the H\"older class $\cH(\alpha, \cL)$ on $[0, 1)^K$
%
%
is defined as the set of functions $g: [0, 1)^K \to \Rb$ whose partial derivatives satisfy for 
all $x, x' \in [0, 1)^K$,
$
\sum_{\kappa: |\kappa| = \lfloor \alpha \rfloor} \frac{1}{\kappa !} |\triangledown_\kappa g(x) - \triangledown_\kappa g(x') | \le \cL \norm{x - x'}_\infty^{\alpha - \lfloor \alpha \rfloor}.
$
Here $\lfloor \alpha \rfloor$ refers to the greatest integer strictly smaller than $\alpha$ {and $\kappa! = \prod_{j=1}^K \kappa_j!$.}
\end{definition}
Note that if $\alpha \in (0, 1]$, then the definition above is equivalent to the $(\alpha, \cL)$-Lipschitz condition, i.e.,
$
|g(x) - g(x')| \le \cL \norm{x - x'}_\infty^{\alpha},
$
for $x, x' \ \in [0, 1)^K$.
Given a time series $f: \Zb \to \Rb$, for any $T \geq 1$, recall the Hankel matrix $H \in \Rb^{\lfloor T/2 \rfloor \times \lfloor T/2 \rfloor}$ is defined such that its entry in row $i \in [\lfloor T/2 \rfloor]$ and column $j \in [\lfloor T/2 \rfloor]$ is given by $H_{ij} = f(i+j-1)$. 
We call a time series $f$ to have $(\alpha, \cL)$-H\"older smooth LVM representation for $\alpha, \cL > 0$ if for any given $T \geq 1$, the corresponding Hankel matrix $H$ satisfies: for $i, j \in [\lfloor T/2 \rfloor]$,
$
	\bH_{ij} = g(\theta_i, \omega_j),
$
where $\theta_i, \omega_j  \in [0, 1)^K$ are latent parameters and  $g(\cdot, \omega) \in \cH(\alpha, \cL)$ for any $\omega \in [0, 1)^K$. 
It can be verified that a $(G,0)$-Hankel time series is an instance of such a LVM representation with corresponding $g(x, y) = x^T y$. 
Thus in a sense, this model is a natural generalization of the $(G, 0)$-Hankel matrix representation. 
The following proposition connects this LVM representation to the $(G, \epsilon)$-Hankel representation for appropriately defined $G \geq 1, \epsilon > 0$. 
\begin{proposition}\label{prop:LVM_implies_Hankel}
Given $\alpha, \cL > 0$, let $f$ have $(\alpha, \cL)$-H\"older 
smooth LVM representation. 
Then for all $\epsilon > 0$, $f$ is 
$$
(C(\alpha, K) \Big(\dfrac{1}{\epsilon}\Big)^K , \cL\epsilon^\alpha)-\text{Hankel representable}.
$$
Here $C(\alpha, K)$ is a term that depends only on $\alpha$ and $K$.
\end{proposition}

\subsection{Extending Main Results}
Below, we provide generalizations of the imputation and in-sample forecasting results stated in Section \ref{sec:main_results}.
To do so, we utilize Property \ref{property:spectra_approx} which is analogous to Property \ref{property:spectra} but for the approximate low-rank setting. %

\begin{property}\label{property:spectra_approx}
(Approximately balanced spectra). 
Under the setup of Proposition \ref{prop:approx_flattened_mean_low_rank_representation}, we can represent the $L \times (NT/L)$ stacked Page matrix associated with all $N$ time series $f_1(\cdot), \dots, f_N(\cdot)$ as $\SP(f) = \tilde{\bM}+ \bE$ with $\text{rank}(\tilde{\bM}) = \ell \geq 1$ and $\ell \leq  R \mult G$ and $\|\bE\|_\infty \leq R \Gamma_1 \epsilon$. 
Then, for 
%
$L = \sqrt{\min(N, T) T}$,
$\tilde{\bM}$ is such that $\sigma_\ell(\tilde{\bM}) \geq c \sqrt{NT}/\sqrt{\ell}$ for some absolute constant $c > 0$, where $\sigma_\ell$ is the $\ell$-th largest singular value of $\tilde{\bM}$.
\end{property} 
\begin{theorem}[Imputation]\label{thm:mean_estimation_imputation_generalized}
Let Properties \ref{prop:low_rank_mean_}, \ref{prop:approx_low_rank_mean_hankel}, \ref{prop:bounded_noise} and \ref{property:spectra_approx} hold. 
For a large enough absolute constant $C > 0$, let  $\rho \geq C \frac{\log NT}{\sqrt{NT}}$. 
Then, with hyper-parameters $L =\sqrt{\min(N, T) T}$ and $k = \ell$, 
\begin{align}
\imp(N, T) & \leq 
C(c, \Gamma_1, \Gamma_2, \gamma)
\bigg(
\frac{R^{3}G \log NT}{\rho^{4} \sqrt{\min(N, T) T}} 
+ \frac{ R^{4}G ( \epsilon +  \epsilon^{3} )  }{\rho^2}
\bigg)
\end{align}
where $C(c, \Gamma_1, \Gamma_2, \gamma) $ is a positive constant dependent on model parameters including $\Gamma_1, \Gamma_2, \gamma$. 
\end{theorem}
We remind the reader that $R, \Gamma_1, \Gamma_2$ are defined in Property \ref{prop:low_rank_mean_}, $G$ in \ref{prop:low_rank_mean_hankel}, $\gamma$ in Property \ref{prop:bounded_noise} and $c$ in Property \ref{property:spectra_approx}. 

\medskip
\noindent{\bf Existence of Linear Model.} 
We now state Proposition \ref{prop:approx_low_rank_linear}, which is analogous to Proposition \ref{prop:exact_low_rank_linear}, but for the approximate low-rank setting.
\begin{proposition}\label{prop:approx_low_rank_linear}
Let Properties \ref{prop:low_rank_mean_}  and \ref{prop:approx_low_rank_mean_hankel} hold. 
Then, there exists $\pbeta \in \Rb^{L-1}$,  such that 
$
\| \SP(f)_{L \cdot}^T - {\SPp(f)}^T \pbeta \|_\infty  \leq R\Gamma_1 (1 + \| \pbeta \|_{1}) \epsilon.,
$
Further $\| \pbeta \|_0 \le R G$.
\end{proposition}

\begin{theorem}[In-sample forecasting]\label{thm:mean_estimation_forecasting_generalized}
Let the conditions of Theorem \ref{thm:mean_estimation_imputation_generalized} hold. 
Then with $\beta^{*}$ defined in Proposition \ref{prop:approx_low_rank_linear}, we have
\begin{align}
\fore(N, T, L) 
& \leq 
C(c, \gamma, \Gamma_1, \Gamma_2) \max(1,  \|\pbeta\|^2 _1)
\Big(
\frac{R^{3}G \log NT}{\rho^{4} \sqrt{\min(N, T) T}} 
+ \frac{ R^{4}G ( \epsilon +  \epsilon^{3} )  }{\rho^2}
\Big).
\end{align}
\end{theorem}

\vspace{-2mm}
\section{Experiments}\label{sec:experiments}

{
We describe experiments supporting our theoretical results. 
In Section \ref{ssec:imputation} and \ref{ssec:forecasting}, we expand on the summary results described earlier in Table \ref{table:stat_nrmse} and report the imputation and forecasting results of mSSA. 
In Appendix \ref{appendix:experiments}, we describe the datasets utilized and the various algorithms we compare with as well as the procedure for selecting the  hyper-parameters in each algorithm. 
In Section \ref{sec:tssa_empirics}, we empirically evaluate the effectiveness of tSSA in the imputation task relative to mSSA and ME. 
}

Note that in all experiments, we use the Normalized Root Mean Squared Error (NRMSE) as out accuracy metric. 
That is, we normalize all the underlying time series to have zero mean and unit variance before calculating the root mean squared error.  
We use this metric as it weighs the error on each time series equally.

\subsection{Imputation}\label{ssec:imputation}
%
\noindent {\bf Setup.} 
We test the robustness of mSSA's imputation performance by adding two sources of corruption to the data - varying the percentage of observed values and varying the amount of  noise we perturb the observations by.
We test imputation performance by how accurately we recover missing values.
We compare the performance of mSSA with TRMF, a method which achieves state-of-the-art imputation performance.
Further, to analyze the added benefit of exploiting the spatial structure in a multivariate time series using mSSA, we compare with the SSA variant introduced in \cite{SSA_Sigmetrics} .

\vspace{2mm}
\noindent{\bf Results.} 
Figures  \ref{fig:111}, \ref{fig:121}, \ref{fig:131}, \ref{fig:141}, and \ref{fig:151} show the imputation error in the aforementioned datasets as we vary the fraction of missing values, while Figures  \ref{fig:112}, \ref{fig:122}, \ref{fig:132}, \ref{fig:142}, and \ref{fig:152}    show the imputation error as we vary $\sigma$, the standard deviation of the gaussian noise.  
We see that as we vary the fraction of missing values and noise  levels, mSSA outperforms both TRMF and SSA in {$\sim$ 75\%} of experiments run. 
It is noteworthy the large empirical gain in mSSA over SSA, giving credence to the spatio-temporal model we introduce.
The average NRMSE across all experiments for each dataset is reported in Table \ref{table:stat_nrmse}, where mSSA outperforms every other method across all datasets except for the Traffic dataset.

\subsection{Forecasting}\label{ssec:forecasting}

\noindent{\bf Setup.} 
We test the forecasting accuracy of the proposed mSSA variant against several state-of-the-art algorithms. 
For each dataset, we split the data into training, validation,  and testing datasets as outlined in Appendix \ref{ssec:dataset}.
As was done in the imputation experiments,  we vary how much each dataset is corrupted by varying the percentage of observed values and the noise levels. 

\vspace{2mm}
\noindent{\bf Results.} 
Figures \ref{fig:211}, \ref{fig:221}, \ref{fig:231}, \ref{fig:241}, and \ref{fig:251}  show the forecasting accuracy of mSSA and other methods in the aforementioned datasets as we vary the fraction of missing values, while Figures  \ref{fig:212}, \ref{fig:222}, \ref{fig:232}, \ref{fig:242}, and \ref{fig:252}  show the forecasting accuracy as we vary the standard deviation of the added gaussian noise.  
We see that as we vary the fraction of missing values and noise level, mSSA is the best or comparable to the best performing method in {$\sim$ 80\%} of experiments.
In terms of the average NRMSE across all experiments, we find that mSSA performs similar to or better than
every other method across all datasets except for the traffic dataset as was reported in Table \ref{table:stat_nrmse}.
{ 
\subsection{Empirical Evaluation of tSSA}\label{sec:tssa_empirics}

\noindent {\bf Setup.} 
To corroborate our theoretical findings, we empirically evaluate the effectiveness of tSSA, mSSA, and ME in imputing/de-noising a multivariate time series $X_n(t)$ for $n \in [N], t \in [T]$, as we vary $N$ and $T$.
For mSSA and ME, we use HSVT as the matrix/tensor estimation algorithm, while for tSSA, we use Alternating Least Squares (ALS).
We note that we use ALS for tSSA as it is a practical and a widely used tensor estimation algorithm.

Using the synthetic dataset described in Appendix \ref{ssec:dataset}, we evaluate the imputation performance of the three algorithms listed above as we vary $N$ and $T$. 
Specifically, we corrupt the time series with i.i.d. zero-mean Gaussian noise with a standard deviation of $0.5$ and use the three algorithms to recover the underlying mean.  
We evaluate the performance of each algorithm based on the NRMSE metric.

\noindent {\bf Results.} 
In Figure \ref{fig:tssaA}, each point represents an experiment performed with the corresponding value of $T$ and $N$ in the X-axis and Y-axis, respectively. 
The color represents how mSSA error compares to tSSA's (in percentage). That is,  a positive/negative value indicates tSSA/mSSA is performing better.
Figure \ref{fig:tssaB}, shows a similar comparison between mSSA and ME.
Pleasingly, we find that the empirical results agree with our theoretical analysis, where all three methods perform similarly when $N \ge T$, tSSA performs best when $T \ge N \ge T^{1/3}$, and mSSA performs best when $T \ge N^3$.
}

\begin{figure}[!h]
    \centering
    \begin{subfigure}[b]{0.45\textwidth}
        \centering
        \input{figures/tssa}

        \caption{\footnotesize{tSSA's NRMSE improvement over mSSA. A negative value (blue) indicates mSSA is better.}}
        \label{fig:tssaA}
    \end{subfigure}
    \hfill 
    \begin{subfigure}[b]{0.45\textwidth}
        \centering
        \input{figures/ME}

        \caption{\footnotesize{ME's NRMSE improvement over mSSA. A negative value (blue) indicates mSSA is better.}}
        \label{fig:tssaB}
    \end{subfigure}
    \caption{The empirical evaluation corroborates our theoretical findings. As Proposition \ref{prop:imputation_comparisons_new} suggests: (i) tSSA outperforms mSSA when $T \ge N \ge T^{1/3}$ (see red dots in Figure \ref{fig:tssaA}); (ii) mSSA starts to outperforms tSSA as we increase $T$ such that that it exceeds $N^3$  (see blue dots in Figure \ref{fig:tssaA}); (iii) All methods perform similarly in the regime $N \ge T$ (see gray dots above the blue line in figures \ref{fig:tssaA} and \ref{fig:tssaB}).  }
    \label{fig:combined}
\end{figure}

\section{Algorithmic Extensions of mSSA}\label{sec:extension}

\subsection{Variance Estimation}\label{ssec:var}

We extend the mSSA algorithm to estimate the time-varying variance of a time series by making the following simple observation.
If we apply mSSA to the squared observations, $X_n^2(t)$, we will recover an estimate of $\Ex[X_n^2(t)]$ (for $\rho = 1$). 
However, observe that $\Var[X_n(t)] = \Ex[X_n^2(t)] - \Ex[X_n(t)]^2$. 
Therefore, by applying mSSA twice, once on $X_n(t)$ and once on $X^{2}_n(t)$ for $n \in [N]$ and $t \in [T]$,  and subsequently taking the component-wise difference of the two estimates will lead to an estimate of the variance. 
This suggests a simple algorithm which we describe next. 
We note this observation suggests {\em any} mean estimation algorithm (or imputation) in time series analysis can be converted to estimate the time varying variance -- this ought to be of interest in its own right. 

\vspace{2mm}
\noindent{\bf Algorithm.}
As described in Section \ref{sec:mSSA}, let $L \geq 1$ and $k, k' \geq 1$ be algorithm parameters. 
First, apply mSSA on observations $X_n(t), ~n \in [N], ~t \in [T]$ to produce imputed estimates $\hat{f}_n(t)$ using the hyper-parameters $L$ and $k$. 
Next, apply mSSA on observations $X_n^2(t), ~n \in [N], ~t \in [T]$ to produce imputed estimates $\hat{g}_n(t)$ using the hyper-parameters $L$ and $k'$. 
Lastly, we denote $\hat{\sigma}_n^2(t) = \max(0, \hat{g}_n(t)-\hat{f}_n(t)^2), ~n \in [N], ~t \in [T]$ as our estimate of the time-varying variance.

\vspace{2mm}
\noindent{\bf Model.} 
For $n \in [N], ~t \in [T]$, let $\sigma_n^2(t) = \Ex[\eta_n^2(t)]$ be the time-varying variance of the time series observations, i.e., if $\rho = 1$ then $\sigma_n^2(t) = \Var[X_n(t)] = \Ex[X_n^2(t)] - f_n^2(t)$.  
Let $\Sigma \in \Rb^{N \times T}$ be the matrix induced by the latent time-varying variances of the $N$ time series of interest, i.e., the entry in row $n$ at time $t$ in $\bSigma$ is $\bSigma_{nt} = \sigma^2_n(t)$. 
To capture the ``spatial'' and ``temporal'' structure across the $N$ latent time-varying variances, we assume the latent variance matrix $\bSigma$ satisfies Properties \ref{prop:low_rank_var_} and \ref{prop:low_rank_var_hankel}.
These properties are analogous to those assumed about the latent mean matrix $\bM$ (defined in Section \ref{sec:ts_model}); in particular, Properties \ref{prop:low_rank_mean_} and \ref{prop:low_rank_mean_hankel}. 
We state them next.
\begin{property}\label{prop:low_rank_var_}
Let $R' = \text{rank}(\bSigma)$, i.e, for any $n \in [N], t \in [T]$, 
$
\bSigma_{nt}  = \sum^{R^\prime}_{r=1} U^\prime_{nr} \ W^\prime_{rt},
$
where the factorization is such that $| U^\prime_{nr} | \le \Gamma^\prime_1$, $| W^\prime_{rt} | \le \Gamma^\prime_2$ for $\Gamma^\prime_1, \Gamma^\prime_2 > 0$.
\end{property}
Like Property \ref{prop:low_rank_mean_}, the above property captures  the ``spatial'' structure within $N$ time series of variances. 
To capture the ``temporal'' structure, next we introduce an analogue of Property \ref{prop:low_rank_mean_hankel}.
To that end, for each $r \in [R^\prime]$, define the $\lfloor T/2 \rfloor \times \lfloor T/2 \rfloor$ Hankel matrix of each time series $W^\prime_{rt}, ~t \in [T]$ as $H^\prime(r) \in \Rb^{\lfloor T/2 \rfloor \times \lfloor T/2 \rfloor}$, where $H^\prime(r)_{ij} = W^\prime_{r (i+j-1)}$ for $i, j \in [\lfloor T/2 \rfloor]$. 
\begin{property}\label{prop:low_rank_var_hankel}
For each $r \in [R^\prime]$, the Hankel Matrix $H^\prime(r) \in \Rb^{\lfloor T/2 \rfloor \times \lfloor T/2 \rfloor}$ associated with time series $W^\prime_{rt}, t \in [T]$ has rank at most $G^\prime$. 
\end{property}
\noindent{\em Result.} 
To establish the estimation error for the variance estimation algorithm under the spatio-temporal model above, we need the following additional property (analogous to Property \ref{property:spectra}).
\begin{property}\label{property:spectra_var}
\textbf{(Balanced spectra (variance))}. Denote the $L \times (NT/L)$ stacked Page matrix associated with all $N$ time series $\sigma^{2}_1(\cdot), \dots, \sigma^{2}_N(\cdot)$ as $\SP(\sigma^{2}) \coloneqq \StackedPage((\sigma^{2}_1,\dots, \sigma^{2}_N), T, L)$.
Due to Properties \ref{prop:low_rank_var_} and \ref{prop:low_rank_var_hankel}, and a simple variant of Proposition \ref{prop:flattened_mean_low_rank_representation}, we have $\text{rank}(\SP(\sigma^{2})) = \ell' \geq 1$ and $\ell' \leq  R' \mult G'$. 
Then, for 
%
$L = \sqrt{\min(N, T) T}$,
$\SP(\sigma^{2})$ is such that $s_\ell(\SP(\sigma^{2})) \geq c \sqrt{NT}/\sqrt{\ell'}$ for some absolute constant $c > 0$, where $s_\ell(\cdot)$ denotes the $\ell$-th largest singular value of its matrix argument.\footnote{We adopt $s_\ell(\cdot)$ here (instead of the paper-wide $\sigma_k(\cdot)$) to avoid notational clash with the variance $\sigma^2$.}
\end{property} 
%

%
\begin{theorem}[Variance Estimation]\label{thm:var_estimation_imputation_simplified}
Let Properties \ref{prop:low_rank_mean_}, \ref{prop:low_rank_mean_hankel}, \ref{prop:bounded_noise}, \ref{property:spectra}, 
\ref{prop:low_rank_var_}, \ref{prop:low_rank_var_hankel},  and \ref{property:spectra_var} hold.
Additionally let $| \hat{f}_n(t) | \le \Gamma_3$ for all $n \in [N], t \in [T]$.
Lastly, let hyper-parameters $L =\sqrt{\min(N, T) T}$, $k = \ell$, $k' = \ell'$.
Let $\rho = 1$.  
Then the variance prediction error is bounded above as 
\begin{align}
\frac{1}{NT} \sum_{n =1}^N \sum_{t=1}^T \Ex[\big(\sigma_n(t)^2 - \hat{\sigma}_n^2(t)\big)^2 ]
& \le \tilde{C} 
\left(
\frac{  (G^{2} + G') \log^{2} NT}{ \sqrt{\min(N, T)T} }.
\right).
\end{align} 
where $\tilde{C}$ is a constant dependent (polynomially) on  model parameters $\Gamma_1$, $\Gamma_2$, $\Gamma_3$, $\Gamma'_1$, $\Gamma'_2$, $\gamma$, $R$, $R'$.
\end{theorem}
Proof of Theorem \ref{thm:var_estimation_imputation_simplified} can be found in Appendix \ref{sec:proof_variance_estimation_imputation}. 

\medskip 

\subsection{ Application to Time-varying Recommendation Systems}
In Appendix~\ref{ssec:appendix_recc_systems}, we discuss the extension of our spatio-temporal model and tSSA to time-varying recommendation systems.

\section{Conclusion}\label{sec:conclusion}
{
We introduce and study two extensions of Singular Spectrum Analysis (SSA) to the multivariate setting: a  practical, simple variant the well-known matrix-based method (mSSA), and a novel tensor-based approach (tSSA).
We show how to extend mSSA to estimate time-varying variance and introduce a tensor variant, tSSA, which builds upon recent advancements in tensor estimation.
We hope this work motivates future inquiry into the connections between the classical field of time series analysis and the modern field of matrix/tensor estimation, and to promote broader adoption of tensor-based methods within the time-series community.
}

\begin{appendix}

\section{Page vs. Hankel mSSA}\label{sec:hankel_vs_page}
This section discusses the benefits and drawbacks of using the Page matrix representation, as we propose in our variant, instead of the Hankel representation used in the original mSSA.
Recall the key steps of the original SSA method in Section \ref{appendix:lit_review}.
The extension to mSSA is done by stacking the Hankel matrices induced by each of the $N$ time series either column-wise (horizontal mSSA) or row-wise (vertical mSSA) \cite{hassani2018singular}.  
In this section, we will use mSSA to denote our mSSA variant, and hSSA/vSSA to denote the original horizontal/vertical mSSA.
In what follows, we will compare our mSSA variant with hSSA/vSSA in terms of their:
(i) theoretical analysis; 
(ii) computational complexity; and
(iii) empirical performance. 

\vspace{2mm}

\noindent 
\textbf{Theoretical analysis.}
We re-emphasize that to the best of our knowledge, the theoretical analysis of the mSSA algorithm, both hSSA and vSSA, have been absent from the literature, despite their popularity.
We do a comprehensive theoretical analysis of the variant of mSSA we propose.
By utilizing the Page matrix, it allows us to invoke results from random matrix theory to prove our imputation and forecasting results.
However, extending our analysis to the Hankel matrix representation is challenging as the Hankel matrix has repeated entries of the same time series observation. 
This leads to correlation in the noise in the observation of the entries of the Hankel matrix, which prevents us from invoking the results from random matrix theory in a  straightforward way.
The Page matrix representation does not have repeated entries of the same observation, and thus allows us to circumvent this issue in our theoretical analysis.

\vspace{2mm}

\noindent 
\textbf{Computational complexity.}
Our mSSA variant is computationally far more efficient than both hSSA and vSSA.
This is because the Page matrix representation of a multivariate time series with N time series and T time steps is a matrix of dimension $ \sqrt{NT} \times \sqrt{NT}$ (with $L = \sqrt{NT}$)., i.e., it has a total of $\mathcal{O}(NT)$ entries. 
In contrast, the Hankel matrix representation is of dimension $ {T}/{4} \times {3NT}/{4}$ for hSSA and $ {NT}/{4} \times {3T}/{4}$ for vSSA (we set the parameter $L$ to $T/4$ as recommended in \cite{hassani2018singular}), i.e., both variants of the Hankel matrix have $\mathcal{O}(NT^2)$ entries. 
 This makes computing the SVD (the most computationally intensive step of mSSA) prohibitive for hSSA and mSSA even for the standard time series benchmarks we consider in Section \ref{sec:experiments}.  

\vspace{2mm}

To empirically demonstrate the computational efficiency of our variant of mSSA, we compare its training time to that of hSSA and vSSA. 
 Specifically, we measure the training time for mSSA, hSSA, and vSSA  as we increase the number of time steps $T \in [400,10000]$.
 We perform this experiment on two datasets: 
 (i) the synthetic dataset; 
 (ii) a subset of the electricity dataset, where we choose only 50 of the available 370 time series.  
 Both datasets are described in details in Appendix \ref{appendix:experiments}.
 Figure \ref{fig:mssa_vs_hankelmssa_training_time} shows that in both datasets, the training time of both hSSA and vSSA can be as 600-1000x as high as the training time of our mSSA variant as we increase $T$.

\vspace{2mm}

\noindent 
\textbf{Empirical performance.}  
Here, we compare the forecasting performance of  mSSA to that of hSSA and vSSA.
We report performance in terms of the NRMSE of the three methods as we increase the number of time steps $T \in [400,10000]$ in the aforementioned synthetic and electricity dataset.
The goal in the synthetic dataset is to predict the next 50 time steps using one step ahead forecasts, while the goal in the electricity dataset is to predict the next three days using day-ahead forecasts.
For hSSA and vSSA, we choose $L = T/4$ as recommended in \cite{hassani2018singular}; and for mSSA, we choose $L = \lfloor\sqrt{NT}\rfloor$. 
For all three methods, we choose the number of retained singular values based on the thresholding procedure outlined in \cite{Donoho}. 

\vspace{2mm}
 Figures \ref{fig:mssa_vs_hankelmssa} shows the performance of the three methods in both datasets.
 We find that initially, with few data points ($T< 600$ in the synthetic data and $T<4000$ in the electricity data), both hSSA and vSSA outperform mSSA. 
 As we increase $T$, mSSA performance significantly improves and eventually outperforms vSSA.
 In the electricity dataset, mSSA performs similar to hSSA for $T = 10000$. 
 These experiments suggest that if only a few observations were available, hSSA and vSSA might provide better performance. 
 However, if the number of observations were relatively large, then the performance of mSSA is superior to vSSA and relatively similar to hSSA. 
 
 \vspace{2mm}
 
Importantly, the electricity dataset experiment illustrates a critical advantage of our mSSA variant.  
 Specifically, when $T$ is large such that running hSSA or vSSA is computationally infeasible, then one can achieve  better accuracy using  mSSA.
 For example, while we could not run the hSSA and vSSA on the electricity dataset with $T=20000$ due to memory constraints, we were able to run mSSA and achieve a lower NRMSE.
 This suggests that our mSSA variant is the more practical mSSA algorithm when it comes to efficiently utilizing large multivariate time series.

\section{Experiment Details}\label{appendix:experiments}
In Appendix~\ref{ssec:dataset}, we describe the datasets utilized.
In Appendix~\ref{ssec:params}, we describe the various algorithms we compare with as well as the choice of hyper-parameters used for  each of them. 

\subsection{Datasets} \label{ssec:dataset}
We use four real-world datasets and one synthetic dataset.  
The description and preprocessing we do for each of these datasets are as follows.

\medskip
\noindent {\bf Electricity Dataset. } 
This is a public dataset obtained from the UCI repository which shows the 15-minutes electricity load of 370 households \cite{uci_elec}.  
As was done in \cite{TRMF},\cite{Amazon},\cite{DeepAR}, we aggregate the data into hourly 
intervals and use the first 25824 time-points for training, the next 288 points for validation, and the last 168 points for testing in the forecasting experiments.  
Specifically, in our testing period, we do 24-hour ahead forecasts for the next seven days (i.e. 24-step ahead forecast). 
See Table \ref{table:data_details} for more details.

\medskip
\noindent {\bf  Traffic Dataset. } 
This public dataset obtained from the UCI repository shows the occupancy rate of traffic lanes in San Francisco \cite{uci_elec}.  
The data is sampled every 15 minutes but to be consistent with previous work in \cite{TRMF}, \cite{Amazon}, we aggregate the data into hourly data and use the first 10248 time-points for training, the next 288 points for validation, and the last 168 points for testing in the forecasting experiments.
Specifically, in our testing period,  we do 24-hour ahead forecasts for the next seven days (i.e. 24-step ahead forecast). 
See Table \ref{table:data_details} for more details.

\medskip
\noindent {\bf Financial Dataset.} 
This dataset is obtained from the Wharton Research Data Services (WRDS) and contains the average daily stocks prices of 839 companies from October 2004 till November 2019 \cite{WRDS}.  
The dataset was preprocessed to remove stocks with any null values, or those with an average price below 30\$ across the aforementioned period.   
This was simply done to constrain the number of time series for ease of experimentation and we end up with 839 time series (i.e. stock prices of listed companies) each with 3993 readings of daily stock prices.  
In our forecasting experiments, we train on the first 3693 time points, validate on the next 120 time points, while for testing we consider the task of predicting 180 time-points ahead one point at a time.
That is, the goal here is to do one-day ahead forecasts for the next 180 days (i.e. 1-step ahead forecast). We choose to do so as this is a standard goal in finance.
See Table \ref{table:data_details} for more details.

\medskip
\noindent {\bf  M5 Dataset. } 
This public dataset obtained from Kaggle's M5 Forecasting competition include daily sales data of 30490 items across different Walmart stores for 1941 days \cite{m5}. 
The dataset was preprocessed to only include items that has more than zero sales in at least 500 days. 
For forecasting, as is the goal in the Kaggle competition, we consider the task of predicting the sales for the next 28 days (i.e. 28-step ahead forecast). 
We use the first 1829 points for training, the next 84 points for cross validation, and the last 28 points for testing.

\medskip
\noindent {\bf  Synthetic Dataset. } 
We generate the observation tensor  $X \in \mathbb{R}^{ n \times m \times T} $  by first randomly generating the two matrices $U \in \mathbb{R}^{r \times n} =[u_1,\dots, u_n]$ and $V \in \mathbb{R}^{r \times m} =[v_1, \dots, v_m]$; we do so by randomly sampling each coordinate of $U, V$ independently from a standard normal.
Then, we generate $r$ mixtures of harmonics where each mixture $g_k(t), k \in [r],$ is generated as: $g_k(t) = \sum_{h=1}^4 \alpha_h \cos(\omega_h t/T)$ where the parameters $\alpha_h, \omega_h$ are selected uniformly at randomly from the ranges $[-1,10]$ and $[1,1000]$, respectively.
Then each value in the observation tensor is constructed as follows:
$
X_{i,j}(t) = \sum_{k=1}^r u_{ik}  v_{jk} g_k(t),
$ 
where $r$ is the tensor rank, $i \in [n]$, $ j\in [m]$. 
In our experiment, we select $n = 5$, $m = 10$, $T = 15000$, and $r =4$.  This gives us $N = n \ \mult \ m = 50$ time series each with $15000$ observations per time series.
In the forecasting experiments, we use the first $13700$ points for training, the next 300 points for validation, while for testing,  we do $10$-step ahead forecasts for the final $1000$ points.
See Table \ref{table:data_details} for more details.

\begin{table}[h]
\caption{Dataset and training/validation/test split details.}
\label{table:data_details}
\tabcolsep=0.1cm
\fontsize{8.0pt}{8.0pt}\selectfont
\begin{tabular}{@{}lllllllll@{}}
\toprule
Dataset     & \begin{tabular}[c]{@{}l@{}}No.time\\  series\end{tabular} & \begin{tabular}[c]{@{}l@{}}Observations\\ per time series\end{tabular} & \begin{tabular}[c]{@{}l@{}}Forecast\\ horizon ($h$)\end{tabular} & \begin{tabular}[c]{@{}l@{}}Training \\ period\end{tabular} & \begin{tabular}[c]{@{}l@{}}No. validation \\ windows $W_{val}$ \end{tabular} & \begin{tabular}[c]{@{}l@{}}Validation \\ period\end{tabular} & \begin{tabular}[c]{@{}l@{}}No. test\\ windows\end{tabular} & \begin{tabular}[c]{@{}l@{}}Test\\ period\end{tabular} \\ \midrule
Electricity & 370                                                       & 26136                                                                  & 24                                                               & 1 to 25824                                                 & 2                                                                 & 25825 to 25968                                               & 7                                                          & 25969 to 26136                                        \\
Traffic     & 963                                                       & 10560                                                                  & 24                                                               & 1 to 10248                                                 & 2                                                                 & 10249 to 10392                                               & 7                                                          & 10393 to 10560                                        \\
Synthetic   & 50                                                        & 15000                                                                  & 10                                                               & 1 to 13700                                                 & 10                                                                & 13701 to 14000                                               & 100                                                        & 14001 to 15000                                        \\
Financial   & 839                                                       & 3993                                                                   & 1                                                                & 1 to 3693                                                     & 40                                                                & 3694 to 3813                                                    & 180                                                        & 3814 to 3993                                          \\
M5          & 15678                                                     & 1941                                                                   & 28                                                               & 1 to 1829                                                  & 1                                                                 & 1830 to 1913                                                 & 1                                                          & 1914 to 1941                                          \\ \bottomrule
\end{tabular}
\end{table}

\subsection{Algorithms.}\label{ssec:params}
In this section, we describe the algorithms used throughout the experiments in more detail and the hyper-parameters/implementation used for each method.

\medskip
\noindent {\bf mSSA \& SSA. }  
Note that since the SSA's variant described in \cite{SSA_Sigmetrics} is a special case of our proposed mSSA algorithm, we use our mSSA's implementation to perform the SSA experiments; key difference in SSA is that we do not ``stack'' the various Page matrices induced by each time series.
For all experiments we choose the parameters through the cross validation process detailed in Appendix \ref{ssec:cv_parameters}, where we perform a grid search for the following parameters:

\begin{inparaenum}
\item  \textit{The number of retained singular values, $k$.} This parameter is chosen using one of the following data-driven methods: (i)  we choose $k$ based on the thresholding procedure outlined in \cite{Donoho}, where 
the threshold is determined by the median of the singular values and the shape of the matrix; (ii) we choose $k$ as the minimum number of singular values that capture a fraction $\tau$ of the spectral energy, with $\tau \in \{0.7, 0.8, 0.9, 0.95\}$; (iii) we additionally try fixed ranks $k \in \{1, 3, 5, 10, 25\}$.
\item \textit{The shape of the Page matrix.} The Page matrix shape is controlled either by setting $L$ directly or through a column-to-row ratio $\rho = M / L$.
For mSSA, $L \in \{10, 500, 700, 800, 1000, 1250, 2000\}$ (dataset-dependent) or $\rho \in \{2, 5, 10, 20, 500\}$.
For SSA, $L \in \{10, 30, 40, 50, 80, 100, 150\}$ (dataset-dependent) or $\rho \in \{2, 5, 10\}$.
\item \textit{Missing values initialization.} Initializing the missing values is done according to one of two methods: (i)  set the missing values to zero; (ii)  perform forward filling where each missing value is replaced by the nearest preceding observation, followed by backward filling to accommodate the situation when the first observation is missing. 
\end{inparaenum}
\medskip
\noindent {\bf DeepAR. } 
We use the ``DeepAREstimator'' algorithm provided by the GluonTS package.
We choose the parameters through a grid search for the following parameters:

\begin{inparaenum}
\item  \textit{Context length}. This parameter determines the number of steps to unroll the RNN for before computing predictions. We choose this from the set $\{h \text{ (default)}, 2h, 3h\}$, where $h$ is the prediction horizon.
\item  \textit{Number of Layers.} This parameter determines the number of RNN layers. We choose this from the set $\{2 \text{ (default)},3\}$.
\end{inparaenum}

\medskip
\noindent {\bf TRMF. } 
We use the implementation provided by the authors in the Github repository associated with the paper (\cite{TRMF}).  
We choose the parameters through a grid search, as suggested by the authors  in their codebase, for the following parameters:
\begin{inparaenum}
\item  \textit{Matrix rank $k$}. This parameter  represents the chosen rank for the $T \times N$ time series matrix, we choose $k$ from the set $\{5,10,20,40,60\}$.
\item  \textit{Regularization parameters $\lambda_f, \lambda_x,\lambda_w$.} We choose these parameters from $\{0.05,0.5,5, 50\}$ as suggested in the authors repository. 
\end{inparaenum}
\noindent For the lag indices , we include the last day and the same weekday in the last week for the traffic and  electricity data,   the last 30 points for the  financial and synthetic dataset, and the last 10 points for the M5 dataset.

\medskip
\noindent {\bf LSTM.} 
Across all datasets, we use an LSTM network with $H \in \{2,3,4\}$ hidden layers each, with $45$ neurons per layer, as is done in \cite{Amazon}. We use the Keras implementation of LSTM.
As with  other methods' parameters, $H$ is chosen via cross validation. 

\medskip
\noindent {\bf Prophet. }  
We used Prophet’s Python library with the parameters selected using a grid search of the following parameters as suggested in~\cite{Prophet}:
\begin{inparaenum}
\item  \textit{Changepoint prior scale.} This parameter determines how much the trend changes at the detected trend changepoints. We choose this parameter from $\{0.001,0.05, 0.2\}$.
\item  \textit{Seasonality prior scale.} This parameter controls the magnitude of the seasonality. We choose this parameter from $\{0.01,10\}$.
\item  \textit{Seasonality Mode.}  Which is chosen to be either 'additive` or 'multiplicative`.
\end{inparaenum}

\noindent {\bf VAR. } We used the VAR estimator in the python package ``statsmodels'' (\cite{statsmode}). We apply the method on the first difference of the time series and verify that the series are not non-stationary using a unit root test (specifically, Augmented Dickey–Fuller test). For all datasets except M5, we choose the best value for the parameter \texttt{max\_lag} $\in\{1,2,5,10,20,50\}$. This parameter corresponds to the maximum number of lags used in fitting the VAR process. For M5, we choose   \texttt{max\_lag} $\in\{1,2,5\}$, as fitting the model for larger values is computationally infeasible. 

\medskip

\noindent {\bf tSSA, mSSA and ME in Section \ref{sec:tssa_empirics}.}
In this experiment, we use HSVT as the the matrix estimation subroutine for both ME in mSSA. 
For both methods, we choose the number of singular components retained based on the  the thresholding procedure outlined in \cite{Donoho}. 
For tSSA, we use ALS as the tensor estimation subroutine. Therein, we choose the best performing rank among the follwoing options: 
(i) the rank suggested by the thresholding procedure outlined in \cite{Donoho} for the stacked Page matrix used in mSSA; 
(ii) the rank suggested by the same procedure for the Page matrix of one of the time series (specifically the first);
(iii) the fixed values $\lfloor L/2 \rfloor $ and $\lfloor L/3 \rfloor $.

\subsection{Parameters Selection}\label{ssec:cv_parameters}
In all experiments, we choose the hyperparameters for out method and for the baselines by using cross-validation. 
Below, we detail the procedure for both imputation and forecasting experiments.

\medskip
\noindent {\bf Imputation Experiments. } 
To select the parameters in our imputation experiments, we additionally mask 10\% of the observed data uniformly at random.
Then, we evaluate the performance of each  parameter choice in recovering these additionally masked observations.
This process is repeated 3 times, and the choice of parameters that achieves the best performance (in NRMSE) across these runs is selected.
In our results, we report the accuracy of the selected parameters in recovering the original missing values. 

\medskip
\noindent {\bf Forecasting Experiments. } 
For parameters selection in the forecasting experiments, we use cross-validation on a rolling basis as typically used in time-series forecasting models~\cite{book_cv}. 
In this procedure, there are multiple validation sets. For each validation set, we train the model only on previous observations. 
That is, no future observations can be used in training the model, which will occur when a typical cross-validation procedure is followed for time series data. 
In our experiments, we start with a subset of the data used for training, then we forecast the first validation set using $h$-step ahead forecasts for $W_{val}$ windows , where the horizon $h$ and the number of validation windows $W_{val}$ are detailed in Table \ref{table:data_details}.  
We do this for three validation sets, each of length $h\times W_{val}$, and select the choice of parameters that achieves the best performance (in NRMSE) for evaluation on the test set. 
When evaluating on the test set, both the training and validation periods are used for training.

\section{Time-varying Recommendation Systems}\label{ssec:appendix_recc_systems}
In tSSA, we considered the setting where the $N \times T$ matrix $\bM$ induced by the latent time series $f_1(\cdot), \dots, f_N(\cdot)$ is low-rank; in particular, Property \ref{prop:low_rank_mean_} captures this spatial structure across these $N$ time series.
However, in many settings there is {\em additional} spatial structure across the $N$ time series.

\vspace{2mm}
\noindent
{\em Recommendation systems -- time-varying matrices/tensors.}
For example, in recommendation systems, for each $t \in T$, there is a $N_1 \times N_2$ matrix, $\bM^{(t)} \in \Rb^{N_1 \times N_2}$ of interest.
The $n_1$-th row and $n_2$-th column of $\bM^{(t)}$ denotes the latent rating user $n_1$ has for product $n_2$, i.e., $\bM^{(t)}_{n_1, n_2}$ denotes the value of the latent time series $f_{n_1, n_2}(\cdot)$ at time step $t$.
To capture the latent structure across users and products, one typically assumes that each $\bM^{(t)}$ is low-rank.
More generally, at each time step $t$, $\bM^{(t)} \in \Rb^{N_1 \times N_2, \dots, \times N_d}$ could be an order-$d$ tensor.
That is, $\bM^{(t)}_{n_1, \dots, n_d}$ denotes the value of the latent time series $f_{n_1, \dots, n_d}(\cdot)$ at time step $t$ for $n_1, \dots, n_d \in [N_1] \times \dots \times [N_d]$.
For example, if $d=3$, $\bM^{(t)}$ might represent the $t$-th measurement for a collection of $(x, y, z)$-spatial coordinates.
Let $\bN \in  \Rb^{N_1 \times N_2, \dots, \times N_d \times T}$ denote the $d + 1$ order tensor induced by viewing each order-$d$ tensor $\bM^{(t)}$ as the $t$-th `slice' of $\bN$, for $t \in [T]$.
Again, to capture the spatial and temporal structure of these latent time series, we posit the following spatio-temporal model for $\bN$, which is a higher-order analog of the model assumed in Property \ref{prop:low_rank_mean_}.
\begin{property}\label{prop:low_rank_mean_tensor}
Let $\bN$ have CP-rank at most $R$. 
That is, for any $n_1, \dots, n_d \in [N_1] \times \dots \times [N_d]$
\begin{align}
\bN_{n_1, \dots, n_d, t} & = \sum^{R}_{r=1} U_{n_1, r} \dots U_{n_d, r} \ W_{rt},
\end{align}
where the factorization is such that $| U_{n_1, r} |,  \dots | U_{n_d, r} | \le \Gamma_1$, $| W_{rt} | \le \Gamma_2$ for constants $\Gamma_1, \Gamma_2 > 0$.
\end{property}
As before, to explicitly model the temporal structure, we continue to assume Property \ref{prop:low_rank_mean_hankel} holds for the latent time factors $W_{r \cdot}$ for $r \in [R]$.

\vspace{2mm}
\noindent
{\bf Order-$d + 2$ Page tensor representation.}
We now consider the following order-$d + 2$ Page tensor representation of $\bN$.
In particular, given the hyper-parameter $L \geq 1$, define $\hightensor \in \Rb^{N_1 \times \dots \times N_d \times L \times T/L}$ such that for $n_1, \dots, n_d \in [N_1] \times \dots \times [N_d], ~\ell \in [L], ~s \in [T/L]$,
\begin{align}
    \hightensor_{n_1, \dots, n_d, \ell, s} & = f_{n_1, \dots, n_d }((s-1)\times L + \ell).
\end{align}
The corresponding observation tensor, $\highTensor \in (\Rb \cup \{\star\})^{N_1 \times \dots \times N_d \times L \times T/L}$, is
\begin{align}\label{eq:noisy_page_tensor_rep_higher}
    \highTensor_{n_1, \dots, n_d, \ell, s} & = X_{n_1, \dots, n_d }((s-1)\times L + \ell). 
\end{align}
Recall from \eqref{eq:model} that $X_{n_1, \dots, n_d }(t)$ is the noisy, missing observation we get of $f_{n_1, \dots, n_d }(t)$.
$\hightensor$ and $\highTensor$ then have the following property:
\begin{proposition}\label{prop:high_tensor}
Let Properties \ref{prop:low_rank_mean_tensor}, \ref{prop:low_rank_mean_hankel}, and \ref{prop:bounded_noise} hold. 
Then, for any $1 \leq L \leq \sqrt{T}$, $\hightensor$ has CP-rank at most $R \mult G$. 
Further, all entries of $\highTensor$ are independent random variables with each entry observed with probability $\rho \in (0,1]$, and $\Ex[\highTensor] = \rho \hightensor$. 
\end{proposition}
Analogous to Proposition \ref{prop:tensor}, Proposition \ref{prop:high_tensor} also establishes that order-$d + 2$ Page tensor representation of the various latent time series $f_{n_1, \dots, n_d}(\cdot)$ has CP-rank that continues to be bounded by $R \mult G$.
Proof of Proposition \ref{prop:high_tensor} can be found in Appendix \ref{sec:tSSA_proofs}.

\vspace{2mm}
\noindent
{\bf Higher-order tensor singular spectrum analysis (htSSA).}
Proposition \ref{prop:high_tensor} motivates the following algorithm, which exploits the further spatial structure amongst the $N$ time series.
We now define the ``meta'' htSSA algorithm.
The two algorithmic hyper-parameters are $L \ge 1$ (defined in \eqref{eq:noisy_page_tensor_rep_2}) and $\TE_{d+2}$ (the order-$d + 2$ tensor estimation algorithm one chooses).
First, using the observations $X_{n_1, \dots, n_d }(t)$ for $n_1, \dots, n_d \in [N_1] \times \dots \times [N_d], t \in [T]$ we construct the higher-order Page tensor $\highTensor$ as in \eqref{eq:noisy_page_tensor_rep_higher}.
Second, we obtain $\widehat{\hightensor}$ as the output of $\TE_{d+2}(\highTensor)$, and read off $\hat{f}_{n_1, \dots, n_d }(t)$ by selecting the appropriate entry in $\widehat{\hightensor}$.

\vspace{2mm}
\noindent
{\bf Relative effectiveness of mSSA, htSSA, and tensor estimation (TE).}
Again, for ease of exposition, we consider the case where $\rho = 1$.
We now briefly discuss the relative effectiveness of htSSA, mSSA,and ``vanilla'' tensor estimation (TE) in imputing $X_{n_1, \dots, n_d }(\cdot)$ to estimate $f_{n_1, \dots, n_d }(\cdot)$.
mSSA and htSSA have been previously described.
In TE, one directly de-noises the original order-$d+1$ tensor induced by the noisy observations, which we denote  $\bX \in  \Rb^{N_1 \times N_2, \dots, \times N_d \times T}$, where $\bX_{n_1, \dots, n_d, t} = X_{n_1, \dots, n_d }(t)$.
In particular, one produces an estimate of $\widehat{\bN} = \TE_{d + 1}(\bX)$, and then produces the estimates $\hat{f}_{n_1, \dots, n_d}(t)$ by reading off the appropriate entry of $\widehat{\bN}$.
Let $\imp(N,T; \text{htSSA})$, $\imp(N,T; \text{mSSA})$, and $\imp(N,T; \text{TE})$ denote the imputation error for htSSA, mSSA, and TE, respectively.
Now if we assume Property \ref{property:te_error_rates} holds, we have
\begin{align}
	\imp(N,T; \text{htSSA})  &= \tilde{\Theta}\left(\frac{1}{\min\left(N_1, \dots, N_d, \sqrt{T}\right)^{\lceil \frac{d+ 2}{2} \rceil}}\right), 
	\\ \imp(N,T; \text{mSSA})  &=  \tilde{\Theta}\left(\frac{1}{\sqrt{\min(N, T) T}}\right),
	 \\ \imp(N,T; \text{TE})  &= \tilde{\Theta}\left(\frac{1}{\min\left(N_1, \dots, N_d, T\right)^{\lceil \frac{d+ 1}{2} \rceil}}\right).
\end{align}
Then just as was done in the proof of Proposition \ref{prop:imputation_comparisons_new}, for any given $d$, one can reason about the relative effectiveness of htSSA, mSSA, and TE for different asymptotic regimes of the relative ratio of $N$ and $T$.

\section{Proof of Proposition \ref{prop:approx_flattened_mean_low_rank_representation}}\label{sec:proof_prop_flattened_page_low_rank}
Below, we present the proof of Proposition \ref{prop:approx_flattened_mean_low_rank_representation}. 
First we define the stacked Hankel matrix of $N$ time series over $T$ time steps. 
Precisely, given $N$ latent time series $f_1,\dots, f_N$, consider the stacked Hankel matrix induced by each of them over $T$ time steps, $[T]$, defined as follows. 
It is $\stackedHankel \in \Rb^{\lfloor T/2\rfloor \times N \lfloor T/2\rfloor}$ where its entry in row $i \in [\lfloor T/2\rfloor]$ and column $j \in [N \lfloor T/2\rfloor]$, $\stackedHankel_{ij}$, is given by
\begin{align}
\stackedHankel_{ij} & = f_{n(i,j)}(i + (j \mod \lfloor T/2\rfloor)  - 1 ), ~~\text{where}~n(i,j) = \Big\lceil \frac{j}{\lfloor T/2\rfloor}\Big\rceil.
\end{align}
We now establish Proposition \ref{prop:approx_low_rank_hankel_}, which  immediately implies Proposition \ref{prop:approx_flattened_mean_low_rank_representation} -- the stacked Page matrix can be viewed as a sub-matrix of $\stackedHankel$, by selecting the appropriate columns. 
\begin{proposition} \label{prop:approx_low_rank_hankel_}
Let Properties \ref{prop:low_rank_mean_} and  \ref{prop:approx_low_rank_mean_hankel} hold for $N$ 
latent time series of interest, $f_1,\dots, f_N$. 
Then for any $T \geq 1$, the stacked Hankel Matrix of these $N$ time series has $\epsilon'$-approximate rank $R \mult G$ with $\epsilon' = R \Gamma_1 \epsilon$.
\end{proposition}
\begin{proof}{}
We have $N$ latent time series $f_1,\dots, f_n$ satisfying Properties \ref{prop:low_rank_mean_} and  
\ref{prop:approx_low_rank_mean_hankel}. Consider their stacked Hankel matrix over $[T]$, 
$\stackedHankel \in \Rb^{\lfloor T/2\rfloor \times N \lfloor T/2\rfloor}$. 
By definition for $i \in [\lfloor T/2\rfloor]$ and $j = (n-1) \mult \lfloor T/2\rfloor + j'$ for 
$j' \in [\lfloor T/2\rfloor]$, we have
\begin{align}
\stackedHankel_{i j'} & = f_{n}(i + j' - 1).
\end{align}
That is, 
\begin{align}
\stackedHankel_{i j} &= f_{n} (i + j' - 1) \nonumber \\
&= \sum^{R}_{r=1} U_{nr}  W_{r (i+j' -1)}. \label{eq:propx.1}
\end{align}
Let $H(r) \in \Rb^{\lfloor T /2\rfloor \times \lfloor T /2\rfloor }$ be the Hankel matrix associated with $W_{r \cdot}$ over $[T]$.  
Due to Property \ref{prop:approx_low_rank_mean_hankel}, there exists a low-rank matrix $M(r) \in \Rb^{\lfloor T/2\rfloor \times \lfloor T/2\rfloor }$ such that (a) $\text{rank}(M(r)) \leq G$, (b) $\| H(r) - M(r) \|_\infty \leq \epsilon$. 
That is, for any $i, j' \in [\lfloor T/2 \rfloor]$, we have that $M(r)_{ij'} =  \sum_{g=1}^G a^r_{ig} b^r_{j'g}$ for some $a^r_{i \cdot}, b^r_{j' \cdot} \in \Rb^G$.  
Therefore, for any $i, j' \in [\lfloor T/2 \rfloor]$, we have that
\begin{align}
W_{r (i+j' -1)} & = H(r)_{ij'}~=~M(r)_{ij'} + (H(r)_{ij'} - M(r)_{ij'} )\nonumber \\
& = \sum_{g=1}^G a^r_{ig} b^r_{j'g} + (H(r)_{ij'} - M(r)_{ij'} ). \label{eq:propx.2}
\end{align}
From \eqref{eq:propx.1} and \eqref{eq:propx.2}, we conclude that 
\begin{align}
\stackedHankel_{i j} & = \sum_{r=1}^R \sum_{g=1}^G U_{nr} a^r_{ig} b^r_{j'g} + \sum^{R}_{r=1} U_{nr} (H(r)_{ij'} - M(r)_{ij'} ) \nonumber \\
& = \sum_{(r,g) \in [R] \times [G]} a^r_{ig} \mult (U_{nr} b^r_{j' g}) + \sum^{R}_{r=1} U_{nr} (H(r)_{ij'} - M(r)_{ij'} ). \label{eq:propx.3}
\end{align}
Define matrix $\sM \in \Rb^{\lfloor T/2\rfloor \times N \lfloor T/2\rfloor}$ with its entry for row $i \in [\lfloor T/2\rfloor]$ and column $j = (n-1) \mult \lfloor T/2\rfloor + j'$ for $j' \in [\lfloor T/2\rfloor]$ given by
\begin{align}
\sM_{ij} & = \sum_{(r,g) \in [R] \times [G]} a^r_{ig} \mult (U_{nr} b^r_{j' g}) \nonumber \\
& = \sum_{(r,g) \in [R] \times [G]} \alpha_{i (r,g)} \beta_{j (r,g)}, 
\end{align}
where $\alpha_{i (r,g)} = a^r_{ig}$ and $\beta_{j (r,g)} = U_{nr} b^r_{j' g}$. 
Further,
\begin{align}
| \stackedHankel_{i j} - \sM_{ij} | & \leq \sum^{R}_{r=1} | U_{nr}| |(H(r)_{ij'} - M(r)_{ij'} )| \nonumber \\
 & \leq \sum_{r=1}^R \Gamma_1 \|H(r) - M(r)\|_\infty ~\leq~ R \Gamma_1 \epsilon.
\end{align}
That is, the stacked Hankel matrix $\stackedHankel$ of $N$ time series of $[T]$ has 
$\epsilon'$-approximate rank $G \mult R$ with $\epsilon' = R \Gamma_1 \epsilon$. 
This completes the proof.
\end{proof}

\section{Proofs For Section \ref{sec:hankel_calculus}}\label{appendix:proof_ts_model_hankel_examples}

\subsection{Proof of Proposition \ref{prop:Hankel_algebra}}\label{sec:proof_Hankel_algebra}
\begin{proof}{}
Let  $f_1, f_2$ have a $(G_1, \epsilon_1)$ and $(G_2, \epsilon_2)$-Hankel representation, respectively.
For any $T \geq 1$, let $\bH_1, \bH_2 \in \Reals^{\lfloor T/2 \rfloor \times \lfloor T/2 \rfloor}$ be the Hankel matrices of $f_1, f_2$, respectively, over the time interval $[T]$. 
By definition, there exists matrices $M_1, M_2 \in \Reals^{\lfloor T/2 \rfloor \times \lfloor T/2 \rfloor}$ such that $\text{rank}(\bM_1) \leq G_1$, $\| \bM_1 - \bH_1 \|_\infty \leq \epsilon_1$ and 
$\text{rank}(\bM_2) \leq G_2$, $\| \bM_2 - \bH_2 \|_\infty \leq \epsilon_2$.

\vspace{2mm}
\noindent
{\bf Component-wise addition.}
Note the Hankel matrix of $f_1+f_2$ over $[T]$ is $\bH_1 + \bH_2$. 
Then, matrix $\bM = \bM_1 + \bM_2$ has rank at most $G_1 + G_2$ since for 
any two matrices $\bA$ and $\bB$, it is the case that $\text{rank}(\bA + \bB) \le \text{rank}(\bA) + \text{rank}(\bB)$. 
Further, $\| \bH_1 + \bH_2 - (\bM_1 + \bM_2) \|_\infty \leq \epsilon_1 + \epsilon_2$. 
Therefore it follows that $f_1 + f_2$ has $(G_1 + G_2, \epsilon_1 + \epsilon_2)$-Hankel representation. 

\vspace{2mm}
\noindent
{\bf Component-wise multiplication.}
For $f_1 \circ f_2$, its Hankel over $[T]$ is given by $\bH_1 \circ \bH_2$ where we abuse notation of $\circ$ in the context of matrices as the Hadamard product of matrices. 
Let $\bM = \bM_1 \circ \bM_2$. 
Then $\text{rank}(M) \leq G_1 \mult G_2$ since for any two matrices $\bA$ and $\bB$, $\text{rank}(\bA \circ \bB) \le \text{rank}(\bA)\text{rank}(\bB)$. 
Now
\begin{align}\label{eq:last.prop.4.1}
\| \bH_1 \circ \bH_2 - \bM_1 \circ \bM_2 \|_\infty & \leq \| \bH_1 \circ \bH_2 - \bH_1 \circ \bM_2 \|_\infty + \| \bH_1 \circ \bM_2 - \bM_1 \circ \bM_2 \|_\infty  \nonumber \\
& \leq \| \bH_1\|_\infty \| \bH_2 - \bM_2 \|_\infty + \| \bM_2 \|_\infty \| \bH_1 - \bM_1 \|_\infty \nonumber \\
& \leq \|f_1\|_\infty \epsilon_2 + (\|\bM_2 - \bH_2 \|_\infty + \|\bH_2\|_\infty) \epsilon_1 \nonumber \\
& \leq \|f_1\|_\infty \epsilon_2 + (\|f_2\|_\infty + \epsilon_2) \epsilon_1 \nonumber \\
& = \|f_1\|_\infty \epsilon_2 + \|f_2\|_\infty \epsilon_1 +  \epsilon_1 \epsilon_2 ~\leq~3 \max(\epsilon_1, \epsilon_2) \max( \|f_1\|_\infty,  \|f_2\|_\infty).
\end{align}
This completes the proof of Proposition \ref{prop:Hankel_algebra}.
\end{proof}

\subsection{Proof of Proposition \ref{prop:low_rank_Hankel_matrices}} \label{appendix:LRFs_proof}
\begin{proof}{}
Proof is immediate from Definitions \ref{def:hankel.aprx} and \ref{def:LRF}.
\end{proof}

\subsection{Proof of Proposition \ref{prop:C_k_smoothness_TS}}\label{sec:proof_C_k_smoothness_TS}

\subsubsection{Helper Lemmas for Proposition  \ref{prop:C_k_smoothness_TS}}
We begin by stating some classic results from Fourier Analysis. To do so, we introduce some notation.
Throughout, we have $R > 0$. 

\vspace{2mm}
\noindent
\textbf{$C[0, R]$ and $L^2[0, R]$ functions.} 
$C[0, R]$ is the set of real-valued, continuous functions defined on $[0, R]$. 
$L^2[0, R]$ is the set of square integrable functions defined on $[0, R]$, i.e. $\int_0^{R} f^2(t) dt \le \infty$

\vspace{2mm}
\noindent
\textbf{Inner Product of functions in $L^2[0, {R}]$.} $L^2[0, {R}]$ is a space 
endowed with inner product defined as 
$\langle f, g \rangle := \frac{1}{{R}} \int^{R}_0 f(t) g(t) dt$,
and associated norm as 
$\norm{f} := \sqrt{\frac{1}{{R}} \int^{R}_0 f^2(t) dt}$.

\vspace{2mm}
\noindent
\textbf{Fourier Representation of functions in $L^2[0, {R}]$.} 
For $f \in L^2[0, {R}]$, define its $G \geq 1$-order Fourier representation, $\cF(f, G) \in L^2[0, R]$ as 
\begin{align}\label{eq:fourier_approximation}
 \cF(f, G)(t) = a_0 + \sum^G_{g = 1} (a_g \cos(2\pi g t / {R}) +  b_g \cos(2\pi g t / {R})), ~~t \in [0,R], 
\end{align}
where $a_0, a_g, b_g$ with $g \in [G]$ are called the Fourier coefficients of $f$, defined as 
\begin{align*}
a_0 &:= \langle f, 1 \rangle = \frac{1}{{R}} \int^{R}_0 f(t) dt, \\
a_g &:= \langle f, \cos(2\pi g t / {R}) \rangle = \frac{1}{{R}} \int^{R}_0 f(t) \cos(2\pi g t / {R}) dt, \\
b_g &:= \langle f, \sin(2\pi g t / {R}) \rangle = \frac{1}{{R}} \int^{R}_0 f(t) \sin(2\pi g t / {R}) dt.
\end{align*}
\noindent We now state a classic result from Fourier analysis.
\begin{theorem}[\cite{Fourier_Classic}]\label{thm:fourier_representation}
Given $k \geq 1, R > 0$, let $f \in C^k(R, \per)$. Then, for any $t \in [0,R]$ (or more generally $t \in \Reals$), 
\begin{align}
	\lim_{G \to \infty } \cF(f,G)(t) \to f(t).
\end{align}
\end{theorem}
\noindent We next argue that if $f \in C^k(R, \per)$, then its Fourier coefficients decay rapidly. 
\begin{lemma}\label{lemma:fourier_representation}
Given $k \geq 1, R > 0$, let $f \in C^k(R, \per)$. Then, for $j \in [k]$, the $G$-order Fourier coefficient
of $f^{(j)}$, the $j$-th derivative of $f$, recursively satisfy the following relationship: for $g \in [G]$, 
\begin{align}\label{eq:recurse}
	a^{(j)}_g & = -\Big( \frac{2 \pi g}{{R}} \Big) b^{(j-1)}_g,
	\qquad b^{(j)}_g = \Big( \frac{2 \pi g}{{R}} \Big) a^{(j-1)}_g.
\end{align}
\end{lemma}
\begin{proof}{}
We establish \eqref{eq:recurse} for $a^{(1)}_g, ~g \in [G]$. 
Notice that an identical argument applies to establish \eqref{eq:recurse} for any $a^{(j)}_g, b^{(j)}_g$ for $j \in [k]$ and $g \in [G]$. 
\begin{align*}
a^{(1)}_g = \langle f^{(1)}, \cos(2\pi g t / {R}) \rangle &= \frac{1}{{R}} \int^{R}_0 f^{(1)}(t) \cos(2\pi g t / {R}) dt \\
&\stackrel{(a)}= \frac{1}{{R}} \Big(\Big[ f(t) \cos(2\pi g t / {R})  \Big]_0^{R} - \frac{2 \pi g}{{R}}\Big[\frac1R\int^{R}_0 f(t) \sin(2\pi g t / {R}) dt  \Big]\Big) \\
&=  - \Big( \frac{2 \pi g}{{R}} \Big) b_g^{(0)}.
\end{align*} 
(a) follows by integration by parts. 
\end{proof}

\subsubsection{Completing Proof of Proposition \ref{prop:C_k_smoothness_TS}}\label{sec:proof_C_k_smoothness_TS_complete}

\begin{proof}{}
For $G \in \Nb$, let $\cF(f, G)$ be defined as in \eqref{eq:fourier_approximation}. Then for $t \in \Rb$
\begin{align*}
| f(t) - \cF(f, G)(t) | &\stackrel{(a)}= \Big| \sum^\infty_{g = G + 1}  (a_g \cos(2\pi g t / {R}) +  b_g \cos(2\pi g t / {R})) \Big| \\
&\le \sum^\infty_{g = G + 1} |  a_g  | + |  b_g  | \\ 
&\stackrel{(b)}\le \sum^\infty_{g = G + 1}  \Big( \frac{{R}}{2 \pi g} \Big)^k  \Big( |  a^{(k)}_g  | + |  b^{(k)}_g  | \Big) \\
&\stackrel{(c)}\le \sqrt{2} \Big(\frac{{R}}{2 \pi} \Big)^{k} \sqrt{\sum^\infty_{g = G + 1}  \Big( \frac{1}{g} \Big)^{2k}} \sqrt{ \sum^\infty_{g = G + 1}    \Big( |  a^{(k)}_g  |^2 + |  b^{(k)}_g  |^2 \Big)} \\  
&\stackrel{(d)}\le \sqrt{2} \Big(\frac{{R}}{2 \pi} \Big)^{k} \frac{1}{G^{k - 0.5}}  \sqrt{ \sum^\infty_{g = G + 1}    \Big( |  a^{(k)}_g  |^2 + |  b^{(k)}_g  |^2 \Big)} \\  
&\stackrel{(e)}\le \sqrt{2} \Big(\frac{{R}}{2 \pi} \Big)^{k}  \frac{\norm{f^{(k)}} }{G^{k - 0.5}} \\
&=  C(k, {R}) \frac{\norm{f^{(k)}} }{G^{k - 0.5}},
\end{align*}
where $C(k, {R})$ is a constant that depends only on $k$ and ${R}$; 
(a) follows from Theorem \ref{thm:fourier_representation}; 
(b) follows from Lemma \ref{lemma:fourier_representation}; 
(c) follows from Cauchy-Schwarz inequality and fact that $(\alpha + \beta)^2 \leq 2 (\alpha^2 + \beta^2)$ for any $\alpha, \beta \in \Rb$; 
(d) $\sum_{g=G+1}^\infty g^{-2k} \leq \int_{G}^\infty x^{-2k} dx$ which can be bounded as $G^{-2k+1}/(2k-1)$ which is at most $G^{-2k+1}$ since
$k \geq 1$;
(e) follows from Bessel's inequality, i.e. $\norm{f^{(k)}}^2 \geq \sum_{g=0}^\infty (|  a^{(k)}_g  |^2 + |  b^{(k)}_g  |^2)$. 

\vspace{2mm}
\noindent
Thus, for any $t \in \Rb$, we have a uniform error bound for $f$ being approximated by $\cF(f, G)$ which is a sum of $2G$ harmonics. 
Noting $2G$ harmonics can be represented by an order-$4G$ LRF (by Proposition \ref{prop:lowrank_LRF_example}),we complete the proof.
\end{proof}

\subsection{Proof of Proposition \ref{prop:LVM_implies_Hankel}}\label{sec:proof_LVM_implies_Hankel}
This analysis is adapted from \cite{JXu_smooth_graphon}.

\vspace{2mm}
\noindent
\begin{proof}{}
\textbf{Step 1: Partitioning the space $[0,1)^{K}$.} 
Consider an equal partition of $[0,1)^K$. 
Precisely, for any $k \in \Nb$, we partition the the set $[0,1)$ into $1 / k$ half-open intervals of length $1/ k$,
i.e, 
$
[0, 1) = \cup_{i=1}^k \left[ (i-1)/k, i/k \right).
$
It follows that $[0,1)^{K}$ can be partitioned into $k^{K}$ cubes of forms $\otimes_{j=1}^{K}  \left[ (i_j-1)/k,   i_j/k \right)$ with $i_j \in [k]$.
Let $\cE_k$ be such a partition with $I_1, I_2, \ldots, I_{k^{K}} $ denoting all such cubes and  $z_1, z_2, \ldots, z_{k^{K}} \in \Rb^{K} $ 
denoting the centers of those cubes. 

\vspace{2mm}
\noindent
\textbf{Step 2: Taylor Expansion of $g(\cdot, \omega)$.} Consider a fixed $\omega$.  To reduce 
notational overload, we suppress dependence of $g$ on $\omega$, and abuse notation by using 
$g(\cdot) = g(\cdot, \omega)$ in what follows. 

\vspace{2mm}
\noindent
For every $I_i$ with $1 \le i \le k^K$, define $P_{I_i, \ell} (x) $ as the degree-$\ell$ Taylor's series expansion of $g(x)$ at point $z_i$:
\begin{align}\label{eq:Taylor_series}
	P_{I_i, \ell} (x) =\sum_{\kappa: |\kappa| \le \ell} \frac{1}{\kappa ! } \left(x-z_i \right)^\kappa \nabla_\kappa g( z_i ), 
\end{align}
where $\kappa=(\kappa_1,\ldots, \kappa_d)$ is a multi-index with $\kappa!=\prod_{i=1}^K \kappa_i!$, and $\nabla_k g(z_i)$ is the partial derivative defined in Section \ref{ssec:examples}.
Note similar to $g$, $P_{I_i, \ell} (x)$ really refers to $P_{I_i, \ell} (x, \omega)$. 

\vspace{2mm}
\noindent
Now we define a degree-$\ell$ piecewise polynomial 
\begin{align}\label{eq:piece_wise_polynomial}
P_{\cE_k, \ell} (x) = \sum_{i=1}^{k^{K}}  P_{I_i, \ell} (x) \mathds{1}(x \in I_i). 
\end{align}
For the remainder of the proof, let $\ell = \lfloor \alpha \rfloor$ (recall $\lfloor \alpha \rfloor$ refers to the largest integer strictly smaller than $\alpha$).
Since $f \in \cH(\alpha,L)$, it follows that 
\vspace{2mm}
\noindent
\begin{align}
&\sup_{ x \in [0,1)^K } \left| g(x) - P_{\cE_k, \ell} (x) \right| 
= \max_{1 \le i \le k^{K}} \sup_{x \in I_i}  \left| g(x) - P_{I_i, \ell} (x) \right| \nonumber \\
&\stackrel{(a)} = \max_{1 \le i \le k^{K}} \sup_{x \in I_i} \left| \sum_{\kappa:|\kappa| \le \ell -1} \frac{\nabla_\kappa g ( z_i )}{\kappa !}(x - z_i)^\kappa 
			  + \sum_{\kappa:|\kappa| = \ell} \frac{\nabla_\kappa g ( \tz_i )}{\kappa !}(x - z_i)^{\ell}
			  - P_{I_i, \ell} (x) \right| \nonumber  \\
&\stackrel{(b)} = \max_{1 \le i \le k^{K}} \sup_{x \in I_i} \left| \sum_{\kappa:|\kappa| = \ell} \frac{\nabla_\kappa g ( \tz_i )}{\kappa !}(x - z_i)^{\ell} 
			-  \sum_{\kappa:|\kappa| = \ell} \frac{\nabla_\kappa g ( z_i )}{\kappa !}(x - z_i)^{\ell}  \right|  \nonumber  \\
&= \max_{1 \le i \le k^{K}} \sup_{x \in I_i} \left| \sum_{\kappa:|\kappa| = \ell} \frac{\nabla_\kappa g ( \tz_i ) - \nabla_\kappa g ( z_i )}{\kappa !}(x - z_i)^{\ell} 
			 \right| \nonumber  \\
&\stackrel{(c)} \le  \max_{1 \le i \le k^{K}}  \sup_{x \in I_i} \|x - z_i \|_{\infty}^\ell \sup_{x \in I_i} 
\sum_{\kappa: |\kappa|=\ell} \frac{1}{\kappa!} \left| \nabla_\kappa g ( \tz_i ) -\nabla_{\kappa} g(z_{i} ) \right| \nonumber  \\
&\stackrel{(d)} \le  
 \cL  k^{-\alpha}. \label{eq:taylor.aprx}
\end{align}

\noindent where (a) follows from multivariate version of Taylor's theorem (and using the Lagrange form for the remainder) and $\tz_i \in [0, 1)^K$ is a vector that can be represented as $z_i + c x$ for $c \in (0, 1)$; (b) follows from \eqref{eq:Taylor_series}; 
(c) follows from Holder's inequality; (d) follows from Definition \ref{def:holder}.

\vspace{2mm}
\noindent
\textbf{Step 3: Construct Low-Rank Approximation of Time Series Hankel Using $P_{\cE_k, \ell}$.}
Recall the Hankel matrix, $\bH \in \Rb^{\lfloor T/2 \rfloor \times \lfloor T/2 \rfloor }$ induced by 
the original time series over $[T]$,  where $\bH_{ts} = g(\theta_t, \omega_s), ~t, s \in  [\lfloor T/2 \rfloor ]$ 
with $g(\cdot, \omega) \in \cH(\alpha, \cL)$ for any $\omega$. 
We now construct a low-rank approximation of it using $P_{\cE_k, \ell} = P_{\cE_k, \ell}(\cdot, \omega)$. 
Define $\tbH \in  \Rb^{\lfloor T/2 \rfloor \times \lfloor T/2 \rfloor }$, where 
$\tbH_{ts} = P_{\cE_k, \ell}  (\theta_t, \omega_s), ~t, s \in  [\lfloor T/2 \rfloor ]$.

\vspace{2mm}
\noindent
By \eqref{eq:taylor.aprx}, we have that for all $t, s \in  [\lfloor T/2 \rfloor ]$, 
\[
	\Big| \bH_{ts} - \tbH_{ts} \Big| \le \cL k^{-\alpha}.
\]

\vspace{2mm}
\noindent
It remains to bound the rank of $\tbH$. 
Note that since $P_{\cE_k, \ell}  (\cdot, \omega)$ is a piecewise polynomial of degree  $\ell = \lfloor \alpha \rfloor$ for any given $\omega$, it has the following decomposition: for $~t, s \in  [\lfloor T/2 \rfloor ]$, 
\begin{align} 
\tbH_{ts} & = P_{\cE_k, \ell}  (\theta_t, \omega_s) 
~= \sum_{i=1}^{k^{K}} \langle \Phi(\theta_t), \beta_{I_i, s} \rangle \mathds{1}(\theta_t \in I_i)
\end{align}
where for any $\theta \in \Reals^K$, 
\[
	\Phi(\theta) = \Big(1, \theta_1, \dots, \theta_K, \dots, \theta_1^\ell, \dots, \theta_K^\ell \Big)^T,
\]
the vector of all monomials of degree less than or equal to $\ell$, and $\beta_{I_i, s} $ is a vector collecting the corresponding coefficients.
The number of such monomials is easily shown to be equal to $C(\alpha, K) := \sum^{\lfloor \alpha \rfloor}_{i=1} \binom{i + K - 1}{i}$. 
That is, $\tbH_{ts} = u_t^T v_s$ where $u_t, v_s$ are of dimension at most $k^K C(\alpha, K) $ for each $t, s \in  [\lfloor T/2 \rfloor ]$.
That is, $\tbH$ has rank at most $k^K C(\alpha, K) $. Setting $k = \Big\lceil \dfrac{1}{\epsilon}\Big\rceil$ completes the proof.
\end{proof}

\section{Helper Lemmas}
We recall known concentration and perturbation inequalities that will be useful throughout. 
\begin{theorem}[{\bf Bernstein's Inequality \cite{bernstein1}}]\label{thm:bernstein}
Suppose that $X_1, \dots, X_n$ are independent random variables with zero mean, and M is a constant such that $\abs{X_i} \le M$ with probability one for each $i$. 
Let $S := \sum_{i=1}^n X_i$ and $v := \text{Var}(S)$. 
Then for any $t \ge 0$,
\begin{align*}
\mathbb{P}(\abs{S} \ge t) &\le 2 \exp(- \dfrac{3 t^2}{6v + 2Mt} ).
\end{align*} 
\end{theorem}

\begin{theorem}[{\bf Norm of matrices with sub-gaussian entries \cite{vershynin2010introduction}}]\label{thm:subgaussian_matrix}
Let $\bA$ be an $m \times n$ random matrix whose entries $A_{ij}$ are independent, mean zero, sub-gaussian random variables. 
Then, for any $t > 0$, we have
\begin{align*}
\norm{\bA} &\le C K (\sqrt{m} + \sqrt{n} + t)
\end{align*}
with probability at least $1 - 2\exp(-t^2)$. Here, $K = \max_{i,j} \norm{A_{ij}}_{\psi_2}$. 
\end{theorem}

\begin{lemma}[{\bf Maximum of sequence of random variables \cite{vershynin2010introduction}}]\label{lemma:max_subg}
Let $X_1$, $X_2$, $\dots, X_n$ be a sequence of random variables, which are not necessarily independent, and satisfy $\Ex[X_i^{2p}]^{\frac{1}{2p}} \le K p^{\frac{\beta}{2}}$ for some $K, \beta >0$ and all $i$. 
Then, for every $n \ge 2$,
\begin{align}
\Ex \max_{i \le n} \abs{X_i} &\le C K \log^{\frac{\beta}{2}}(n).
\end{align}
\end{lemma}
We note that Lemma \ref{lemma:max_subg} implies that if $X_1, \ldots, X_n$ are $\psi_{\alpha}$ random variables with $\| X_i \|_{\psi_{\alpha}} \leq K_{\alpha}$ for all $i \in [n]$, then 
\begin{align*}
	\Ex \max_{i \le n} \abs{X_i} &\le C K_{\alpha} \log^{\frac{1}{\alpha}}(n).
\end{align*}

\begin{lemma}[{\bf Modified Hoeffding Inequality \cite{agarwal2020principal} }] \label{lemma:hoeffding_random} 
Let $X \in \Rb^n$ be random vector with independent mean-zero sub-Gaussian random coordinates with $\| X_i \|_{\psi_2} \le K$.
Let $a \in \Rb^n$ be another random vector that satisfies $\|a\|_2 \le b$ almost surely for some constant $b \ge 0$.
Then for all $t \ge 0$, 
\begin{align*}
	\Pb \Big( \Big| \sum_{i=1}^n a_i X_i\Big| \ge t \Big) \le 2 \exp\Big(-\frac{ct^2}{K^2 b^2} \Big),
\end{align*}
where $c > 0$ is a universal constant. 
\end{lemma}

\begin{lemma}[{\bf Modified Hanson-Wright Inequality \cite{agarwal2020principal} }] \label{lemma:hansonwright_random} 
Let $X \in \Rb^n$ be a random vector with independent mean-zero sub-Gaussian coordinates with $\|X_i \|_{\psi_2} \le K$. 
Let $\bA \in \Rb^{n \times n}$ be a random matrix satisfying $\|\bA\|_2  \le a$ and $\|\bA\|_F^2 \, \le b$ almost surely for some $a, b \ge 0$.
Then for any $t \ge 0$,
\begin{align*}
	\Pb \left( \abs{ X^T \bA X - \Ex[X^T \bA X] } \ge t \right) &\le 2 \cdot \exp \Big( -c \min\Big(\frac{t^2}{K^4 b}, \frac{t}{K^2 a} \Big) \Big). 
\end{align*}
\end{lemma} 

\vspace{5pt}
\begin{lemma}  [{\bf Weyl's inequality}] \label{lemma:weyl}
Given $\bA, \bB \in \Rb^{m \times n}$, let $\sigma_i$ and $\widehat{\sigma}_i$ be the $i$-th singular values of $\bA$ and $\bB$, respectively, in decreasing order and repeated by multiplicities. 
Then for all $i \in [m \wedge n]$,
\begin{align*}
\abs{ \sigma_i - \widehat{\sigma}_i} &\le \norm{\bA - \bB}_2.
\end{align*}
\end{lemma}

\section{Matrix Estimation via HSVT}\label{sec:appendix_HSVT_error}
This section describes and analyzes a well-known matrix estimation method, Hard Singular Value Thresholding (HSVT). 
While the analysis utilizes known arguments from the literature, we need to adapt it for the setting where the underlying `signal' is only approximately low-rank. 

\subsection{Setup, Notations}\label{ssec:hsvt.setup}
\medskip
\noindent{\bf Setup.} 
Given a deterministic matrix $\bM \in \Rb^{q \times p}$ with $p, q \in \Nb$ and $q \le p$, a random matrix $\bY \in \Rb^{q \times p}$ is such that all of its entries, $Y_{ij}, ~i \in [q], ~j \in [p]$ are mutually independent and for any given $i \in [q], ~j \in [p]$,
\begin{align}\label{eq:random.matrix}
Y_{ij} = \begin{cases}
		M_{ij} + \veps_{ij} & \text{w.p. } \rho,  ~(\text{i.e. ~observed})  \\
		0 & \text{w.p. } 1-\rho,  ~(\text{i.e. not ~observed})
		\end{cases}
\end{align}
for some $\rho \in (0, 1]$ with $\veps_{ij}$ are independent random variables with $\Ex[\veps_{ij}] = 0$ and $\norm{\veps_{ij}}_{\psi_2} \le \sigma$. 
Given this, we have $\Ex[\bY] = \rho \bM$.  
Defineff
\begin{align}\label{eq:hrho}
\hrho  & = \max \Big(1/(q~p) , \big(\sum_{i=1}^q \sum_{j=1}^p \bOne(Y_{ij} \mbox{~is~obs.})\big)/(q~p)\Big).
\end{align}

\medskip
\noindent{\bf Goal of Matrix Estimation.} 
The goal of matrix estimation is to produce an estimate $\bhM$ from observation $\bY$ so that $\bhM$ is close to $\bM$. 
In particular, we will be interested in bounding the error between $\bhM$ and $\bM$ using the following metric: $\| \bhM - \bM\|_{2,\infty}$. 

\subsection{Matrix Estimation using HSVT}
\noindent \textbf{Hard Singular Value Thresholding (HSVT) Map.} 
We define the HSVT map. 
For any $q, p \in \Nb$, consider a matrix $\bB \in \Rb^{q \times p}$ such that $\bB = \sum_{i=1}^{q \wedge p} \sigma_i(\bB) x_i y_i^T$. 
Here for $i \in [q \wedge p]$, $\sigma_i(\bB)$ is the $i$th largest singular value of $\bB$ and $x_i, y_i$ are the corresponding left and right singular vectors respectively. 
Then, for given any $\lambda > 0$, we define the map $\text{HSVT}_{\lambda}: \mathbb{R}^{q \times p} \to \mathbb{R}^{q \times p}$, which simply shaves off the singular values of the input matrix that are below the threshold $\lambda$. 
Precisely, 
\begin{align} \label{eq:prox_matrix}
\hsvt_{\lambda}(\bB) &= \sum_{i = 1}^{q \wedge p} \sigma_i(\bB) \mathds{1}(\sigma_i(\bB) \ge \lambda) x_i y_i^T.
\end{align}

\medskip
\noindent \textbf{Matrix Estimating using HSVT map.} 
We define a matrix estimation method using the HSVT map that is utilized by mSSA for imputation. 
Precisely, we estimate $\bM$ from $\bY$ as follows: given parameter $k \geq 1$, 
\begin{align}\label{eq:hsvt.est}
\bhM & = \frac{1}{\hrho} \hsvt_{\lambda_k}(\bY).
\end{align}
where $\lambda_k = \sigma_k(\bY)$, i,e. the $k$th largest singular value of $\bY$. 

\subsection{A Useful Linear Operator}
We define a linear map associated to HSVT. 
For a specific choice of $\lambda \geq 0$, define $\varphi^{\bB}_{\lambda}: \mathbb{R}^{p} \to \mathbb{R}^{p}$ as follows: for any vector $w \in \mathbb{R}^p$  (i.e. $w \in \mathbb{R}^{ p \times 1}$),
\begin{align} \label{eq:prox_vector}
\varphi^{\bB}_{\lambda}(w) &= \sum_{i = 1}^{q \wedge p} \mathds{1}(\sigma_i (\bB) \ge \lambda) y_i y_i^T w.
\end{align}
Note that $\varphi^{\bB}_{\lambda}$ is a linear operator and it depends on the tuple $(\bB, \lambda)$; more precisely, the singular values and the right singular vectors of $\bB$, as well as the threshold $\lambda$. 
If $\lambda = 0$, then we will adopt the shorthand notation: $\varphi^{\bB} = \varphi_{0}^{\bB}$. 
The following is a simple, but curious relationship between $\varphi^{\bB}_{\lambda}$
and $\hsvt_\lambda$ that will be useful subsequently. 
\begin{lemma}[\textbf{Lemma 35 of \cite{PCR_NeurIPS, pcr_jasa}}] \label{lemma:column_representation}
Let $\bB \in \mathbb{R}^{q \times p}$ and $\lambda \geq 0$ be given. Then for any $j \in [q]$,
\begin{align}
\varphi^{\bB}_{\lambda} \big( \bB_{j \cdot}^T \big) = \emph{HSVT}_{\lambda}\big(\bB \big)_{j \cdot}^T,
\end{align}
where $\bB_{j \cdot} \in \mathbb{R}^{1 \times p}$ represents the $j$th row of $\bB$, and 
$\emph{HSVT}_{\lambda}\big(\bB \big)_{j \cdot} \in \mathbb{R}^{1 \times p}$ represents the $j$th row of the matrix obtained after applying HSVT over $\bB$ with threshold $\lambda$. 
\end{lemma}
\begin{proof}{}  
By \eqref{eq:prox_vector}, the orthonormality of the right singular vectors and noting $\bB_{j \cdot}^T = \bB^T e_j$ with $e_j \in \mathbb{R}^p$ with $j$th entry $1$ and everything else $0$, we have
\begin{align*}
	\varphi^{\bB}_{\lambda} \big( \bB_{j \cdot}^T \big) 
		&= \sum_{i=1}^{q \wedge p} \mathds{1}(\sigma_i (\bB) \ge \lambda) y_i y_i^T \bB^T_{j \cdot}  
		= \sum_{i=1}^{q \wedge p} \mathds{1}(\sigma_i (\bB) \ge \lambda) y_i y_i^T \bB^T e_j \\
		& = \sum_{i=1}^{q \wedge p} \mathds{1}(\sigma_i (\bB) \ge \lambda) y_i y_i^T \big(\sum_{i' = 1}^{q \wedge p} \sigma_{i'} (\bB) x_{i'} y_{i'}^T\big)^T e_j 
		= \sum_{i, i'=1}^{q \wedge p} \sigma_{i'} (\bB) \mathds{1}(\sigma_i (\bB) \ge \lambda) y_i y_i^T y_{i'} x_{i'}^T e_j \\
		& = \sum_{i, i'=1}^{q \wedge p} \sigma_{i'} (\bB) \mathds{1}(\sigma_i (\bB) \ge \lambda) y_i \delta_{ii'} x_{i'}^T e_j 
		= \sum_{i=1}^{q \wedge p} \sigma_{i} (\bB) \mathds{1}(\sigma_i (\bB) \ge \lambda) y_i x_{i}^T e_j \\
		& = \emph{HSVT}_{\lambda}\big(\bB \big)^T e_j 
		= \emph{HSVT}_{\lambda}\big(\bB \big)_{j \cdot}^T.
\end{align*}
\end{proof}

\subsection{HSVT based Matrix Estimation: A Deterministic Bound}
We state the following result about property of the estimator. 
\begin{lemma} \label{lemma:column_error}
For $k \geq 1$, let $\bM = \bM_k + \bE_k$ with $\text{rank}(\bM_k) = k$. Let $\varepsilon = \max(\hrho/\rho, \rho/\hrho) \geq 1$.
Then, the HSVT estimate $\bhM$ with parameter $k$ is such that for all $j \in [q]$, 
\begin{align}\label{eq:main_MCSE_inequality_res}
\| \bhM_{j \cdot}^T - \bM_{j\cdot}^T\|_2^2 
& \leq \frac{ 2 \| \bY - \rho \bM \|_2^2 + 2 \rho^2 \| \bE_k \|_2^2 }{\big(\sigma_k(\rho \bM_k)\big)^2}
\Big( 2 \Big \| [\bM_k]_{j \cdot}^T \Big \|_2^2 +  \frac{4\varepsilon^2 \big(\|\bY_{j \cdot}^T - \rho \bM_{j \cdot}^T\|_2\big)^2}{\rho^2}\Big) \nonumber \\
& \qquad + \frac{4\varepsilon^2}{\rho^2}  \Big \| \varphi^{ \bM_k }(\bY_{j \cdot}^T - \rho \bM_{j \cdot}^T)  \Big \|_2^2 + 2 (\varepsilon-1)^2 \| \bM_{j \cdot}^T \|_2^2 + 2 \Big\| [\bE_k]_{j \cdot}^T \Big \|_2^2.
\end{align}
\end{lemma}
\begin{proof}{}
We prove our lemma in four steps. 
\paragraph{Step 1. Decomposing $\bhM_{j \cdot}^T - \bM_{j \cdot}^T$ in two terms.}
Fix a row index $j \in [q]$. Let $\lambda_k$ be the $k$th largest singular value of $\bY$, as used by HSVT algorithm with parameter $k \geq 1$.
\begin{align}
	\bhM_{j \cdot}^T - \bM_{j \cdot}^T 
	& = \Big( \bhM_{j \cdot}^T - \varphi_{\lambda_k}^{\bY} \big( \bM_{j \cdot}^T \big) \Big) + \Big( \varphi_{\lambda_k}^{\bY} \big( \bM_{j \cdot}^T \big) - \bM_{j \cdot}^T \Big).
\end{align}
By definition per \eqref{eq:prox_vector}, $\varphi_{\lambda_k}^{\bY}: \Rb^{p} \to \Rb^{p}$ is the projection operator onto $\text{span}\big\{ u_1, \ldots, u_{k} \big\}$, the span of top $k$ right singular vectors of $\bY$,  denoted as $u_1,\dots, u_k$. 
Therefore, 
\begin{align}
\varphi_{\lambda_k}^{\bY} (\bM_{j \cdot}^T) - \bM_{j \cdot} ^T & \in \text{span}\{u_1, \ldots, u_{k} \}^{\perp}.
\end{align}
By design, $\text{rank}(\bhM) = k$. 
Therefore, by Lemma \ref{lemma:column_representation}
\begin{align}
\bhM_{j \cdot} -  \varphi_{\lambda_k}^{\bY} (\bM_{j \cdot}^T) 
& = \frac{1}{\hrho}\varphi_{\lambda_k}^{\bY} (\bY_{j \cdot}^T) - \varphi_{\lambda_k}^{\bY} (\bM_{j \cdot}^T) \in \text{span}\{u_1, \ldots, u_{k} \}.
\end{align}
Therefore, $\langle \bhM_{j \cdot}^T -  \varphi_{\lambda_k}^{\bY} (\bM_{j \cdot}^T), \varphi_{\lambda_k}^{\bY} (\bM_{j \cdot}^T) - \bM_{j \cdot}^T \rangle = 0$, and hence
\begin{align}\label{eq:column_error}
\Big\| \bhM_{j \cdot}^T - \bM_{j \cdot}^T \Big\|_2^2
& = \Big\| \bhM_{j \cdot}^T - \varphi_{\lambda_k}^{\bY} \big( \bM_{j \cdot}^T\big) \Big\|_2^2		
+ \Big\| \varphi_{\lambda_k}^{\bY} \big( \bM_{j \cdot}^T \big) - \bM_{j \cdot}^T \Big\|_2^2	
\end{align}
by the Pythagorean theorem. 

\paragraph{Step 2. Bounding Term 1, $\Big\| \bhM_{j \cdot}^T - \varphi_{\lambda_k}^{\bY} \big( \bM_{j \cdot}^T \big) \Big\|_2$.}
We begin by bounding the first term on the right hand side of \eqref{eq:column_error}. 
By Lemma \ref{lemma:column_representation}, 
\begin{align*}
\bhM_{j \cdot} - \varphi_{\lambda_k}^{\bY}(\bM_{j \cdot}^T)
&= \frac{1}{\hrho}\varphi_{\lambda_k}^{\bY}(\bY_{j \cdot}^T) - \varphi_{\lambda_k}^{\bY}(\bM_{j \cdot}^T)
= \varphi_{\lambda_k}^{\bY} \Big(\frac{1}{\hrho} \bY_{j \cdot}^T - \bM_{j \cdot}^T \Big)\\
&=  \frac{1}{\hrho}  \varphi_{\lambda_k}^{\bY} (\bY_{j \cdot}^T - \rho \bM_{j \cdot}^T ) + \frac{\rho - \hrho}{\hrho} \varphi_{\lambda_k}^{\bY}( \bM_{j \cdot}^T ).
\end{align*}

\noindent
Using the Parallelogram Law (or, equivalently, combining Cauchy-Schwartz and AM-GM inequalities), we obtain 
\begin{align}
\norm{\bhM_{j \cdot}^T - \varphi_{\lambda_k}^{\bY}(\bM_{j \cdot})^T}_2^2 
&= \norm{\frac{1}{\hrho}  \varphi_{\lambda_k}^{\bY} (\bM_{j \cdot}^T - \rho \bM_{j \cdot}^T ) 
+ \frac{\rho - \hrho}{\hrho} \varphi_{\lambda_k}^{\bY}( \bM_{j \cdot}^T) }_2^2		\nonumber\\
&\leq 2 \, \norm{\frac{1}{\hrho}  \varphi_{\lambda_k}^{\bY} (\bY_{j \cdot}^T - \rho \bM_{j \cdot}^T ) }_2^2 
+ 2 \, \norm{ \frac{\rho - \hrho}{\hrho} \varphi_{\lambda_k}^{\bY}( \bM_{j \cdot}^T )}_2^2		\nonumber\\
&\leq \frac{2}{\hrho^2} \norm{\varphi_{\lambda_k}^{\bY}(\bY_{j \cdot}^T - \rho \bM_{j \cdot}^T)}_2^2
+ 2 \Big( \frac{\rho - \hrho}{\hrho}\Big)^2 \| \bM_{j \cdot}^T \|_2^2 \nonumber\\
&\leq \frac{2\varepsilon^2}{\rho^2} \norm{\varphi_{\lambda_k}^{\bY}(\bY_{j \cdot}^T - \rho \bM_{j \cdot}^T)}_2^2
+ 2 (\varepsilon-1)^2 \| \bM_{j \cdot}^T \|_2^2.		\label{eqn:term.1a}
\end{align}
From definition of $\varepsilon$, $\frac{1}{\widehat{\rho}} \leq \frac{\varepsilon}{\rho}$ and 
$\left( \frac{\rho - \widehat{\rho}}{\widehat{\rho}} \right)^2 \leq (\varepsilon-1)^2$.
The first term of \eqref{eqn:term.1a} can be decomposed as,
\begin{align}
&\norm{\varphi_{\lambda_k}^{\bY}(\bY_{j \cdot}^T - \rho \bM_{j \cdot}^T)}_2^2 
\\&\le 2  \,  \Big \| \varphi_{\lambda_k}^{\bY}(\bY_{j \cdot}^T - \rho \bM_{j \cdot}^T) 
- \varphi^{\bM_k }(\bY_{j \cdot}^T - \rho \bM_{j \cdot}^T)  \Big \|_2^2 
+2 \,  \Big \| \varphi^{ \bM_k }(\bY_{j \cdot}^T - \rho \bM_{j \cdot}^T)  \Big \|_2^2. \label{eq:tricky}
\end{align}
In above, we have used notation $\varphi^{ \bM_k } = \varphi^{ \bM_k }_{0}$. 
Given that $\bM_k$ is rank $k$ matrix, $\varphi^{ \bM_k }: \Rb^p \to \Rb^p$ is the projection 
operator mapping any element in $\Rb^p$ to the projection onto the subspace spanned 
by $\{\mu_1,\dots, \mu_k\}$, where $\mu_1, \ldots, \mu_{k} \in \Rb^p$ are the $k$ non-trivial right singular vectors of $\bM_k$. 
Similarly, by definition $\varphi_{\lambda_k}^{\bY}$ is a map $\Rb^p \to \Rb^p$ mapping any element in $\Rb^p$ to its projection onto the subspace spanned by $\{u_1,\dots, u_k\}$, the top $k$ right singular vectors of $\bY$--this can be seen by noting $\lambda_k = \sigma_k(\bY)$ is the $k$-th top singular value of $\bY$.
Recall $\sigma_j(\bY), ~j\in [q \wedge p]$ is the $j$th largest singular value of $\bY$.

\medskip \noindent 
Next, we bound the first term on the right hand side of \eqref{eq:tricky}.  
To that end, by Wedin $\sin \Theta$ Theorem {(see \cite{davis1970rotation, wedin1972perturbation})} and recalling $\text{rank}(\bM_k) = k$, 
\begin{align}\label{eq:davis_kahan}
\big\| \varphi_{\lambda_k}^{\bY} - \varphi^{ \bM_k } \big\|_2 
&\leq  \frac{\| \bY - \rho \bM_k \|_2}{\sigma_k(\rho \bM_k)} \nonumber \\
&\leq  \frac{\| \bY - \rho \bM \|_2}{\sigma_k(\rho \bM_k)} + \frac{\rho \| \bM -  \bM_k \|_2}{\sigma_k(\rho \bM_k)} \nonumber \\
&\leq  \frac{\| \bY - \rho \bM \|_2}{\sigma_k(\rho \bM_k)} + \frac{\rho \| \bE_k \|_2}{\sigma_k(\rho \bM_k)}.
\end{align}

\noindent
Then it follows that
\begin{align}\label{eq:hsvt.int.2} 
 \Big \| \varphi_{\lambda_k}^{\bY}(\bY_{j \cdot}^T - \rho \bM_{j \cdot}^T) - \varphi^{\bM_k }(\bY_{j \cdot}^T - \rho \bM_{j \cdot}^T)  \Big \|_2 
& \leq \| \varphi_{\lambda_k}^{\bY} - \varphi^{\bM_k }\|_2 \|\bY_{j \cdot}^T - \rho \bM_{j \cdot}^T\|_2 \nonumber \\
& \leq \frac{ \big(\| \bY - \rho \bM \|_2 + \rho \| \bE_k \|_2\big) \big(\|\bY_{j \cdot}^T - \rho \bM_{j \cdot}^T\|_2\big) }{\sigma_k(\rho \bM_k)}.
\end{align}
Using \eqref{eq:tricky} and \eqref{eq:hsvt.int.2} in \eqref{eqn:term.1a}, 
\begin{align}
\norm{\bhM_{j \cdot} - \varphi_{\lambda_k}^{\bY}(\bM_{j \cdot}^T)}_2^2 
&  \leq \frac{4\varepsilon^2}{\rho^2}  \frac{ \big(\| \bY - \rho \bM \|_2 
+ \rho \| \bE_k \|_2\big)^2 \big(\|\bY_{j \cdot}^T - \rho \bM_{j \cdot}^T\|_2\big)^2 }{\big(\sigma_k(\rho \bM_k)\big)^2} \nonumber \\
& \qquad + \frac{4\varepsilon^2}{\rho^2}  \Big \| \varphi^{ \bM_k }(\bY_{j \cdot}^T - \rho \bM_{j \cdot}^T)  \Big \|_2^2 
+ 2 (\varepsilon-1)^2 \| \bM_{j \cdot}^T\|_2^2. \label{eq:term_step2}
\end{align}

\paragraph{Step 3. Bounding Term 2, $\Big\| \varphi_{\lambda_k}^{\bY} \big( \bM_{j \cdot}^T \big) - \bM_{j \cdot}^T \Big\|_2^2$.}
Recall $\bM = \bM_k + \bE_k$ and using \eqref{eq:davis_kahan},
\begin{align}
\Big\| \varphi_{\lambda_k}^{\bY} \big( \bM_{j \cdot}^T \big) - \bM_{j \cdot}^T \Big\|_2^2
& = \Big\| \varphi_{\lambda_k}^{\bY} \big( [\bM_k]_{j \cdot}^T + [\bE_k]_{j \cdot}^T \big) - [\bM_k]_{j \cdot}^T - [\bE_k]_{j \cdot}^T \Big\|_2^2 \nonumber \\
& \leq 2 \Big\| \varphi_{\lambda_k}^{\bY} \big( [\bM_k]_{j \cdot}^T\big) -  [\bM_k]_{j \cdot}^T \Big \|_2^2 
+ 2 \Big\| \varphi_{\lambda_k}^{\bY} \big( [\bE_k]_{j \cdot}^T\big) -  [\bE_k]_{j \cdot}^T \Big \|_2^2 \nonumber \\
& = 2 \Big\| \varphi_{\lambda_k}^{\bY} \big( [\bM_k]_{j \cdot}^T\big) -  \varphi_{\lambda_k}^{\bM_k}\big([\bM_k]_{j \cdot}^T\big) \Big \|_2^2 
+ 2 \Big\| \varphi_{\lambda_k}^{\bY} \big( [\bE_k]_{j \cdot}^T\big) -  [\bE_k]_{j \cdot}^T \Big \|_2^2 \nonumber \\
& \leq 2 \Big\| \varphi_{\lambda_k}^{\bY} -  \varphi_{\lambda_k}^{\bM_k} \Big\|_2^2 \Big \| [\bM_k]_{j \cdot}^T \Big \|_2^2 
+ 2 \Big\| [\bE_k]_{j \cdot}^T \Big \|_2^2 \nonumber \\
& \leq 2 \frac{ \big(\| \bY - \rho \bM \|_2 + \rho \| \bE_k \|\big)^2 }{\big(\sigma_k(\rho \bM_k)\big)^2} \Big \| [\bM_k]_{j \cdot}^T \Big \|_2^2	
+ 2 \Big\| [\bE_k]_{j \cdot}^T \Big \|_2^2. \label{eq:term_step3}
\end{align}

\paragraph{Step 4. Putting everything together.} 
Inserting \eqref{eq:term_step2} and \eqref{eq:term_step3} back to \eqref{eq:column_error}, 
we have that for each $j \in [q]$, 
\begin{align}
\Big\| \bhM_{j \cdot}^T - \bM_{j \cdot}^T \Big\|_2^2 
& \leq 2 \frac{ \big(\| \bY - \rho \bM \|_2 + \rho \| \bE_k \|_2\big)^2 }{\big(\sigma_k(\rho \bM_k)\big)^2} \Big \| [\bM_k]_{j \cdot}^T \Big \|_2^2		
+ 2 \Big\| [\bE_k]_{j \cdot}^T \Big \|_2^2 \nonumber \\
& \qquad + \frac{4\varepsilon^2}{\rho^2}  \frac{ \big(\| \bY - \rho \bM \|_2 + \rho \| \bE_k \|_2\big)^2 \big(\|\bY_{j \cdot}^T - \rho \bM_{j \cdot}^T\|_2\big)^2 }{\big(\sigma_k(\rho \bM_k)\big)^2} \nonumber \\
& \qquad + \frac{4\varepsilon^2}{\rho^2}  \Big \| \varphi^{ \bM_k }(\bY_{j \cdot}^T - \rho \bM_{j \cdot}^T)  \Big \|_2^2 
+ 2 (\varepsilon-1)^2 \| \bM_{j \cdot}^T \|_2^2 \nonumber \\
& \leq \frac{ 2 \| \bY - \rho \bM \|_2^2 + 2 \rho^2 \| \bE_k \|_2^2 }{\big(\sigma_k(\rho \bM_k)\big)^2}
\Big( 2 \Big \| [\bM_k]_{j \cdot}^T \Big \|_2^2 
+  \frac{4\varepsilon^2 \big(\|\bY_{j \cdot}^T - \rho \bM_{j \cdot}^T\|_2\big)^2}{\rho^2}\Big) \nonumber \\
& \qquad + \frac{4\varepsilon^2}{\rho^2}  \Big \| \varphi^{ \bM_k }(\bY_{j \cdot}^T - \rho \bM_{j \cdot}^T)  \Big \|_2^2 
+ 2 (\varepsilon-1)^2 \| \bM_{j \cdot}^T \|_2^2 + 2 \Big\| [\bE_k]_{j \cdot}^T \Big \|_2^2,
\end{align}
where we used $(a+b)^2 \leq 2a^2 + 2 b^2$. 
This completes the proof. 
\end{proof}

\subsection{HSVT based Matrix Estimation: Deterministic To High-Probability} \label{appendix:hp}
Next, we convert the bound obtained in Lemma \ref{lemma:column_error} to a bound in expectation (as well as one in high-probability) for our metric of interest: $\| \bhM - \bM\|_{2,\infty}$.
%
In particular, we establish 
\begin{theorem} \label{thm:hsvt.l2inf}
For $k \geq 1$, let $\bM = \bM_k + \bE_k$ with $\text{rank}(\bM_k) = k$. 
Let $\epsilon = \| \bE_k\|_\infty$ and $\Gamma =  \|\bM_k\|_\infty$. 
Let $\rho \geq C \log (qp)/ q$ for $C$ large enough and $q \leq p$. 
Then, the HSVT estimate $\bhM$ with parameter $k$ is such that
\begin{align}\label{eq:thm.hsvt.l2inf}
\Ex\big[\max_{j \in [q]} \frac1p \| \bhM_{j \cdot}^T - \bM_{j\cdot}^T\|_2^2\big] 
& \leq  \frac{ p (C \sigma^2 + \rho^2 \epsilon q)  }{\rho^2 \sigma_k( \bM_k)^2} \Big(\Gamma^2 +  \frac{\sigma^2}{\rho^2}\Big) + \frac{C \sigma^2 k \log p}{p \rho^2}   
+ \frac{C (\Gamma + \epsilon)^2}{p}  + 2 \epsilon^2 + \frac{C}{(pq)^2}.
\end{align}
\end{theorem}
\begin{proof}{} 
We start by identifying certain high probability events. 
Subsequently, using these events and Lemma \ref{lemma:column_error}, we shall conclude the proof. \\

\noindent \textbf{High Probability Events.}  
For some positive absolute constant $C > 0$, define 
\begin{align}
	E_1 &:= \Big\{ \abs{\hrho - \rho} \le \rho / 20\Big \}, \label{eq:E1_def}
     \\ E_2&:= \Big\{ \norm{\bY - \rho \bM}_2 \le C \sigma \sqrt{p}\Big \}, \label{eq:E2_def}
     \\ E_3&:= \Big\{ \norm{\bY - \rho \bM}_{\infty, 2}, \norm{\bY - \rho \bM}_{2, \infty} \le C \sigma \sqrt{p} \Big \}, \label{eq:E3_def}
     \\ E_4&:= \Big\{\max_{j \in [q]} \norm{\varphi^{\bB}_{\sigma_k(\bB)} \Big( \bY_{j \cdot}^T - \rho \bM_{j \cdot}^T \Big)}_2^2 \le C \sigma^2 k \log(p) \Big \}, \label{eq:E4_def}
     \\ E_5&:=\Bigg\{ \bigg(1 - \sqrt{\frac{20 \log (q p)}{ \rho q p}}\bigg) \rho \le \hrho \le \frac{1}{1 - \sqrt{\frac{20 \log (q p)}{ \rho q p}}} \rho \Bigg\}. \label{eq:E5_def} 
\end{align}
In \eqref{eq:E4_def} above,  $\bB \in \Rb^{q \times p}$ is a deterministic matrix. Let the singular value decomposition of $\bB$ be given as $\bB = \sum_{i=1}^{q} \sigma_i(\bB) x_i y_i^T$, where $\sigma_i(\bB)$ are the singular vectors of $\bB$ in decreasing order and $x_i, y_i$ are the left and right singular vectors respectively. 
Recall the definition of $\varphi^{\bB}_{\lambda}$ in \eqref{eq:prox_vector}. 
In particular, we choose $\lambda = \sigma_k(\bB)$, the $k$th singular value of $\bB$ in \eqref{eq:E4_def}. 
As a result, in effect, we are bounding norm of projection of random vector $\bY_{j \cdot} - \rho \bM_{j \cdot}$ for any given deterministic subspace of $\Rb^p$ of dimension $k$. 

\begin{lemma} \label{lemma:prelims}
For some positive constant $c_1 > 0$ and $C > 0$ large enough in definitions of $E_1,\dots, E_5$,
\begin{align}
	\Pb(E_1) &\ge 1 - 2e^{-c_1 p q \rho} - (1-\rho)^{p q}, \label{eq:E1}
	\\ \Pb(E_2) &\ge 1 - 2 e^{-p},  \label{eq:E2}
	\\ \Pb(E_3) &\ge 1 - 2 e^{-p},  \label{eq:E3}
	\\ \Pb(E_4) &\ge 1 - \frac{2}{(qp)^{10}}.  \label{eq:E4}
	\\ \Pb(E_5) &\ge 1 - \frac{2}{(qp)^{10}}.  \label{eq:E5}
\end{align} 
\end{lemma}
\begin{proof}{} 
We bound the probability of events $E_1,\dots, E_5$ in that order. 

\medskip
\noindent {\bf Bounding $\bE_1$.}
Let 
\begin{align}
\hrho_0  & = \big(\sum_{i=1}^q \sum_{j=1}^p \bOne(Y_{ij} \mbox{~is~obs.})\big)/(q~p).
\end{align}
That is, $\hrho = \max(\hrho_0, 1/(pq))$ and $\Ex[\hrho_0] = \rho$. 
We define the event $E_6 := \{ \hrho_0 = \hrho \}$.
Thus, we have that
\begin{align*}
	\Pb(E_1^c) &= \Pb(E_1^c \cap E_6) + \Pb(E_1^c \cap E_6^c)
	\\ &= \Pb( \abs{ \hrho_0 - \rho } \ge \rho / 20 ) + \Pb(E_1^c \cap E_6^c)
	\\ &\le \Pb( \abs{ \hrho_0 - \rho } \ge \rho / 20  )  + \Pb(E_6^c)
	\\ &= \Pb( \abs{ \hrho_0 - \rho } \ge \rho / 20  )  + (1-\rho)^{q p},
\end{align*}	
where the final equality follows by the independence of observations assumption and the fact that $\hrho_0 \neq \hrho$ only if we do not have any observations. 
By Bernstein's Inequality, we have that
\begin{align*}
	\Pb(\abs{ \hrho_0 - \rho } \ge \rho / 20) &\ge 1 - 2e^{-c_1 \rho q p}. \\
\end{align*} 

\medskip
\noindent
{\bf Bounding $\bE_2$.} 
To start with, $\Ex[\bY] = \rho \bM$. 
For any $i \in [q], j \in [p]$, the $Y_{ij}$ are independent, $0$ with probability $1-\rho$ and with probability $\rho$ equal to $M_{ij} + \veps_{ij}$ with $\norm{\veps_{ij}}_{\psi_2} \le \sigma$. 
Therefore, it follows that $\| Y_{ij} - \rho M_{ij}\|_{\psi_2} \leq C' \sigma$ for a constant $C' > 0$. 
Since $q \leq p$, using Theorem \ref{thm:subgaussian_matrix} it follows that for an appropriately large constant $C > 0$, 
\begin{align*}
	\mathbb{P}(E_2) &\ge 1 - 2 e^{-p}. \\
\end{align*}

\vspace{2mm}
\noindent
{\bf Bounding $\bE_3$.} 
Recall that we assume $q \le p$.
Observe that for any matrix $A \in \Rb^{q \times p}$, $\| A\|_{\infty, 2}, \ \| A\|_{2, \infty} \leq \|A\|_2$.
Thus using the argument to bound $\bE_2$, we have \eqref{eq:E3}. \\

\vspace{2mm}
\noindent
{\bf Bounding $\bE_4$.} Consider for $j \in [q]$,
\begin{align}
	\norm{\varphi^{\bB}_{\sigma_k(\bB)} \Big( \bY_{j \cdot}^T - \rho \bM_{j \cdot}^T \Big)}^2_2 
	& = \sum^{k}_{i=1} \norm{ y_i  y^T_i  (\bY_{j \cdot}^T - \rho \bM_{j \cdot}^T) }^2_2 
	& \le \sum^{k}_{i=1} \Big( y^T_i  (\bY_{j \cdot}^T - \rho \bM_{j \cdot}^T) \Big)^2_2 ~=~\sum_{i=1}^k Z_i^2,
\end{align}
where $Z_i = y^T_i  (\bY_{j \cdot}^T - \rho \bM_{j \cdot}^T) $. 
By definition of the $\psi_2$ norm of a random variable and since $y_i$ is unit norm vector that is deterministic (and hence independent the of random vector $\bY_{j \cdot}^T - p \bM_{j \cdot}^T$), it follows that 
\[
\norm{Z_i}_{\psi_2} = \norm{y^T_i  (\bY_{j \cdot} - p \bM_{j \cdot})}_{\psi_2} \leq \norm{ (\bY_{j  \cdot} - p \bM_{j \cdot})}_{\psi_2}.
\]
Since the coordinates of $\bY_{j \cdot}^T - \rho \bM_{j  \cdot}^T$ are mean-zero and independent, with $\psi_2$ norm bounded by $\sqrt{C} \sigma$ for some absolute constant $C > 0$, using arguments from 
\cite{PCR_NeurIPS, pcr_jasa}, it follows that 
\begin{align}
\Pb\Big( \sum_{i=1}^k Z_i^2 > t \Big) & \leq 2 k \exp\Big( - \frac{t}{k C \sigma^2} \Big).
\end{align}
Therefore, for choice of $t = C \sigma^2 k \log p$ with large enough constant $C > 0$,  $q \leq p$,
and taking a union bound over all $j \in [p]$, we have that 
\begin{align}
\Pb\Big(E_4^c\Big) & \leq \frac{2}{(qp)^{10}}.
\end{align}

\vspace{2mm}
\noindent
{\bf Bounding $\bE_5$.} 
Recall the definition of $\hrho$. 
By the binomial Chernoff bound, for $\varepsilon > 1$,
\begin{align*}
\Pb \Big( \hrho > \varepsilon \rho \Big)
&\leq \exp\left( - \frac{(\varepsilon - 1 )^2}{\varepsilon + 1} q p \rho \right),	\quad\text{and}\\
\Pb\Big( \hrho < \frac{1}{\varepsilon} \rho  \Big)
&\leq \exp \left( - \frac{(\varepsilon - 1)^2}{2 \varepsilon^2} q p \rho \right).
\end{align*}
By the union bound,
\[
\Pb\Big(  \frac{1}{\varepsilon} \rho \le \hrho \le \rho \varepsilon \Big)
\geq 1 - \Pb\Big(\hrho > \varepsilon \rho \Big) -  \Pb\Big( \hrho < \frac{1}{\varepsilon} \rho  \Big).
\]
Noticing $\varepsilon + 1 < 2 \varepsilon < 2 \varepsilon^2$ for all $\varepsilon > 1$, and substituting $\varepsilon = \left(1 - \sqrt{\frac{20 \log (q p)}{ q p \rho}} \right)^{-1} $ completes the proof.
\end{proof}
The following are immediate corollaries of the above stated bounds. 
\begin{corollary} \label{cor:prelims}
Let $E := E_1 \cap E_2$. 
Then, for $\rho \geq C \log (qp)/ q$, 
\begin{align}
\Pb(E^c) &\le C_1 e^{-c_2 p},
\end{align}
where $C_1$ and $c_2$ are positive constants. 
\end{corollary}

\begin{corollary} \label{cor:prelims_2}
Let $E := E_2 \cap E_3 \cap E_4 \cap E_5$. 
Then,
\begin{align}
	\Pb(E^c) &\le \frac{C_1}{(qp)^{10}},
\end{align}
where $C_1$ is an absolute positive constant. 
\end{corollary}

\medskip
\noindent \textbf{Probabilistic Bound for HSVT based Matrix Estimation.} 
Recall $\epsilon = \| \bE_k\|_\infty$. 
Then $\| \bE_k \|_F^2 \leq \epsilon q p$. And $\| \bE_k \|_2^2 \leq \| \bE_k \|_F^2 \leq \epsilon q p$. 
Let $\rho \geq C \log (qp)/ q$ for $C$ large enough and recall $q \leq p$.
Further, recall $\Gamma =  \|\bM_k\|_\infty$; thus, $\|\bM\|_\infty \leq \Gamma + \epsilon$. 
Then $ \| [\bM_k]_{j \cdot}^T  \|_2 \leq \Gamma \sqrt{p}$ and $\| [\bM]_{j \cdot}^T  \|_2 \leq (\Gamma + \epsilon) \sqrt{p}$. \\

\noindent
Define $E = E_1 \cap E_2 \cap E_3 \cap E_4 \cap E_5$. 
Then, from Corollaries \ref{cor:prelims} and \ref{cor:prelims_2}, we have that $\Pb(E^c) \leq \frac{C_1}{(qp)^{10}}$ for large enough constant $C_1 > 0$.

\noindent
Under $E_5$, we have  $\varepsilon = \max(\hrho/\rho, \rho/\hrho) \leq \left(1 - \sqrt{\frac{20 \log (q p)}{ q p \rho}} \right)^{-1} $.  
Under this choice of $\varepsilon$ and using $\rho \geq C \log (qp)/ q$, we have that for $C$ large enough, $\varepsilon \leq C$ and $(\varepsilon - 1)^2 \leq C/p$.
\vspace{2mm}

\noindent
Given this setup, under event $E$, Lemma \ref{lemma:column_error} leads to the following: for all $j \in [q]$ and with appropriately (re-defined) large enough constant $C > 0$, 
\begin{align}\label{eq:_MCSE_inequality}
\| \bhM_{j \cdot}^T - \bM_{j\cdot}^T\|_2^2 
& \leq C \frac{  \sigma^2 p + \rho^2 \epsilon q p }{\rho^2 \sigma_k( \bM_k)^2}
\Big(p \Gamma^2 +  \frac{\sigma^2 p}{\rho^2}\Big) \nonumber \\
& \qquad + \frac{C \sigma^2 k \log p}{\rho^2}  + C (\Gamma + \epsilon)^2  + 2 p \epsilon^2. 
\end{align}
That is, under event $E$, 
\begin{align}\label{eq:main_MCSE_inequality}
\max_{j \in [q]} \frac1p \| \bhM_{j \cdot}^T - \bM_{j\cdot}^T\|_2^2 
& \leq C\frac{ p ( \sigma^2 + \rho^2 \epsilon q)  }{\rho^2 \sigma_k( \bM_k)^2}
\Big(\Gamma^2 +  \frac{\sigma^2}{\rho^2}\Big) + \frac{C \sigma^2 k \log p}{p \rho^2}  
\nonumber \\ &+ \frac{C (\Gamma + \epsilon)^2}{p}  + 2 \epsilon^2. 
\end{align}
For any random variable $X$ and event $A$, such that under event $A$, $X \leq B$ and $\Pb(A^c) \leq \delta$, we have 
\begin{align}\label{eq:hsvt.end.0}
\Ex[X] & = \Ex[X \mathds{1}(A)] + \Ex[X \mathds{1}(A^c)] \nonumber \\
& \leq \Ex[X \mathds{1}(A)] + \Ex[X^2]^{\frac12} \Pb(A^c)^{\frac12} \nonumber \\
& \leq B + \Ex[X^2]^{\frac12} \delta^{\frac12}.
\end{align}
We shall use this reasoning above to bound $\Ex\big[ \max_{j \in [q]} \frac1p \| \bhM_{j \cdot}^T - \bM_{j\cdot}^T\|_2^2 \big]$: let $X = \max_{j \in [q]} \frac1p \| \bhM_{j \cdot}^T - \bM_{j\cdot}^T\|_2^2$ and $A = E$; $B$ is given by right hand side of \eqref{eq:main_MCSE_inequality}, $\delta = \frac{C_1}{(qp)^{10}}$; the only missing quantity that remains to be bounded is $\Ex[X^2]$. 
We do that next. 

\noindent To begin with, for any $j \in [q]$,
\begin{align}\label{eq:hsvt.end.1}
\norm{\bhM_{j \cdot}^T - \bM_{j \cdot}^T}_2
& \leq\norm{\bhM_{j \cdot}^T}_2 + \norm{\bM_{j \cdot}^T}_2
\end{align}
by triangle inequality. 
As stated earlier, $\| [\bM]_{j \cdot}^T  \|_2 \leq (\Gamma + \epsilon) \sqrt{p}$. 
Next, we bound $\norm{\bhM_{j \cdot}}_2^T $. 
From \eqref{eq:hsvt.est}, the fact that $\hrho \geq 1/(qp)$, and Lemma \ref{lemma:column_representation}, we have 
\begin{align}\label{eq:hsvt.end.2}
\| \bhM_{j \cdot}^T \|_2 & =  \frac{1}{\hrho} \| \hsvt_{\lambda_k}\big(\bY \big)_{j \cdot}^T\|_2 \nonumber \\
& \leq q~p \| \phi^{\bY}_{\lambda_k}\big(\bY_{j \cdot}^T \big) \|_2  \nonumber \\
& \leq q~p \| \phi^{\bY}_{\lambda_k} \|_2 \|\bY_{j \cdot}^T\|_2 \nonumber \\
& \leq q~p \|\bY_{j \cdot}^T\|_2,
\end{align}
where we used the fact that $\phi^{\bY}_{\lambda_k}$ is a projection operator and hence $\|\phi^{\bY}_{\lambda_k}\|_2 = 1$. 
Note that $Y_{ij} = B_{ij} \mult (M_{ij} + \veps_{ij})$, where $B_{ij}$ is an independent Bernoulli variable with $\Pb(B_{ij} = 1) = \rho$ representing whether $(M_{ij} + \veps_{ij}$ is observed or not. 
Therefore, $|Y_{ij}| = |B_{ij}| \mult |M_{ij} + \veps_{ij}| \leq (\Gamma + \epsilon) + |\veps_{ij}|$. 
Therefore, from \eqref{eq:hsvt.end.1} and \eqref{eq:hsvt.end.2},
\begin{align}\label{eq:hsvt.end.3}
\max_{j \in [q]} \norm{\bhM_{j \cdot}^T - \bM_{j \cdot}^T}_2 
& \leq  (\Gamma + \epsilon) \sqrt{p} + q p \big(\max_{j \in [q]} \|\bY_{j \cdot}^T\|_2\big)\nonumber \\
& \leq  (\Gamma + \epsilon) \sqrt{p} + q p \mult \sqrt{p} \big(\max_{i \in [p], j \in [q]} |Y_{ij}| \big) \nonumber \\
& \leq 2 q p^{\frac32}  \big( \Gamma + \epsilon + \max_{i \in [p], j \in [q]} |\veps_{ij}| \big).
\end{align}
Using $(a+b)^2 \leq 2 a^2 + 2 b^2$ twice, we have $(a+b)^4 \leq 8 (a^4 + b^4)$. 
Therefore, from \eqref{eq:hsvt.end.3}
\begin{align}\label{eq:hsvt.end.3.5}
\max_{j \in [q]} \norm{\bhM_{j \cdot}^T - \bM_{j \cdot}^T}_2^4
& \leq 16 q^4 p^{6}  \big( (\Gamma + \epsilon)^4 + \max_{i \in [p], j \in [q]} |\veps_{ij}|^4 \big).
\end{align}
Recall $\Ex[\veps_{ij}] = 0$, $\norm{\veps_{ij}}_{\psi_2} \le \sigma$ and $\veps_{ij}$ are independent across $i, j$. 
A property of $\psi_{2}$-random variables is that  $\big| \eta_{ij}\big|^{\theta}$ is a $\psi_{2/\theta}$-random variable for $\theta \geq 1$. 
With choice of $\theta =4$, we have  
\begin{align} 
\Ex \big[\max_{ij} \abs{\veps_{ij}}^4\big] & \le C' \sigma^4 \log^{2} (q p), \label{eq:hsvt.end.4}
\end{align}
for some $C' > 0$ by Lemma \ref{lemma:max_subg}. 
From \eqref{eq:hsvt.end.2}, \eqref{eq:hsvt.end.3.5}, and \eqref{eq:hsvt.end.4}, we have that 
\begin{align}\label{eq:hsvt.end.5}
\Big(\Ex\big[\max_{j \in [q]} \frac{1}{p^2} \norm{\bhM_{j \cdot}^T - \bM_{j \cdot}^T}_2^4\big] \Big)^{\frac12}
& \leq 4 q^2 p^{2}  \big( (\Gamma + \epsilon)^4 + C' \sigma^4 \log^{2} (q p) \big)^{\frac12}.
\end{align}
Finally, using \eqref{eq:main_MCSE_inequality}, \eqref{eq:hsvt.end.0} and \eqref{eq:hsvt.end.5}, we conclude
\begin{align}
\Ex\big[\max_{j \in [q]} \frac1p \| \bhM_{j \cdot}^T - \bM_{j\cdot}^T\|_2^2\big] 
& \leq  \frac{ p (C \sigma^2 + \rho^2 \epsilon q)  }{\rho^2 \sigma_k( \bM_k)^2}
\Big(\Gamma^2 +  \frac{\sigma^2}{\rho^2}\Big) + \frac{C \sigma^2 k \log p}{p \rho^2}   
+ \frac{C (\Gamma + \epsilon)^2}{p}  + 2 \epsilon^2 + \frac{C}{(pq)^2}.
\end{align}
This completes the proof of Theorem \ref{thm:hsvt.l2inf}.
\end{proof}

\section{Proof of Theorem \ref{thm:mean_estimation_imputation_generalized}}\label{sec:proof_mean_estimation_imputation}
The proof of Theorem \ref{thm:mean_estimation_imputation_generalized} will utilize Theorem \ref{thm:hsvt.l2inf}. 
To begin with, given $N$ time series with observations over $[T]$, the mSSA algorithm as described in Section \ref{sec:mSSA} constructs the $L \times (NT/L)$ stacked page matrix $\StackedPage((X_1,\dots, X_N), T, L)$ with $L = \sqrt{\min(N, T) T}$, i.e. $L \leq T$.  \\

\noindent
As per the model described by \eqref{eq:model} and Section \ref{sec:ts_model}, it follows that each
entry of $\StackedPage((X_1,\dots, X_N), T, L)$ is an independent random variable; it is observed with probability $\rho \in (0,1]$ independently and when it is observed, its equal to value of the latent time series plus zero-mean sub-Gaussian noise. 
In particular, 
\begin{align}
\Ex\big[\StackedPage((X_1,\dots, X_N), T, L)\big] & = \rho \StackedPage((f_1,\dots, f_N), T, L), 
\end{align}
where $\StackedPage((f_1,\dots, f_N), T, L) \in \Rb^{L \times (NT/L)}$ with entry in row $\ell \in [L]$ and column $ (n-1) \mult T/L + j$ equal to $f_n(\ell + (j-1)\mult L)$. 
Further, when entry in row $\ell \in [L]$ and column $ (n-1) \mult T/L + j$ in $\StackedPage((X_1,\dots, X_N), T, L)$ is observed, i.e. $X_n(\ell + (j-1)\mult L) \neq \star$,  it is equal to $f_n(\ell + (j-1)\mult L) + \eta_n(\ell + (j-1)\mult L)$ where $\eta_n(\cdot)$ are independent, zero-mean sub-Gaussian variables with $\|\eta_n(\cdot)\|_{\psi_2} \leq \gamma$ as per the Property \ref{prop:bounded_noise}. \\

\noindent
Under Properties \ref{prop:low_rank_mean_} and  \ref{prop:approx_low_rank_mean_hankel},  
as a direct implication of Proposition \ref{prop:approx_low_rank_hankel_},  $\StackedPage((f_1,\dots, f_N), T, L)$ has $\epsilon'$-rank at most $R \mult G$ with 
$\epsilon' = R \Gamma_1 \epsilon$. 
That is, there exist rank $k \leq R \mult G$ matrix $\bM_k \in \Rb^{L \times (NT/L)}$ so that 
\begin{align}
\StackedPage((f_1,\dots, f_N), T, L) & = \bM_k + \bE_k, 
\end{align}
where $\|\bE_k\|_\infty \leq \epsilon'$. 
Due to Property \ref{prop:low_rank_mean_}, it follows that $\|\bM_k\|_\infty \leq R \Gamma_1 \Gamma_2 + \epsilon'$. 
Under Property \ref{property:spectra_approx}, we have $\sigma_k(\bM_k) \geq c \sqrt{NT}/\sqrt{k}$ for some constant $c > 0$. \\

\noindent
Define 
\begin{align}\label{eq:def.gamma}
\Gamma & = R \Gamma_1 \Gamma_2 + \epsilon' ~=~R \Gamma_1 (\Gamma_2 + \epsilon).
\end{align}
Recall from Section \ref{sec:mSSA}, the elements of the imputed multivariate time series are simply the entries of the matrix $\hStackedPage((X_1,\dots, X_N), T, L)$ where $\hStackedPage((X_1,\dots, X_N), T, L) = \frac{1}{\hrho} \hsvt_k(\StackedPage((X_1,\dots, X_N), T, L))$. 
That is, imputation in mSSA is carried out by applying HSVT to the stacked page matrix $\StackedPage((X_1,\dots, X_N), T, L)$. \\

\noindent
All in all, the above description precisely meets the setup of Theorem \ref{thm:hsvt.l2inf}. 
To apply Theorem \ref{thm:hsvt.l2inf}, we require $\rho \geq C \log (NT)/\sqrt{NT}$ for $C > 0$ large enough. 
Note that the number of columns in $\hStackedPage((X_1,\dots, X_N), T, L)$ is equal to $NT / L$ for $L = \sqrt{\min(N, T) T}$ -- for this choice of $L$, note that $NT / L \ge L$.
Using $\sigma^{2}_k(\bM_k) \geq c NT / k$, 
for some absolute constant $=c \ge 0$, and using Theorem \ref{thm:hsvt.l2inf}, we obtain
\begin{align}\label{eq:ff.1}
& \Ex\big[\frac{1}{(NT / L)}\| \hStackedPage((X_1,\dots, X_N), T, L) - \StackedPage((f_1,\dots, f_N), T, L)\|_{2,\infty}^2\big] 
\\& \leq \frac{ k (NT / L) (C \gamma^2 + \rho^2 \epsilon' L)  }{\rho^2 c^2 NT}
\Big(\Gamma^2 +  \frac{\gamma^2}{\rho^2}\Big) + \frac{C \gamma^2 k \log NT}{(NT / L) \rho^2} 
+ \frac{C (\Gamma + \epsilon')^2}{(NT / L)}  + 2 (\epsilon')^2 + \frac{C}{(NT)^2}  
\end{align}

\noindent
Recall that $k \leq R \mult G$, $\epsilon' = R \Gamma_1 \epsilon$, and $\Gamma = R \Gamma_1 (\Gamma_2 + \epsilon)$.
Hence, simplifying \eqref{eq:ff.1}, we obtain that 
\begin{align}
& \Ex\big[\frac{1}{(NT / L)}\| \hStackedPage((X_1,\dots, X_N), T, L) - \StackedPage((f_1,\dots, f_N), T, L)\|_{2,\infty}^2\big]
\\& \leq \tilde{C}
\bigg(
\frac{ RG (1 + \rho^2 R \epsilon L)  }{\rho^2  L}
\Big(R^2(1 + \epsilon^{2}) +  \frac{1}{\rho^2}\Big) 
+ \frac{RG \log NT}{(NT / L) \rho^2} 
+ \frac{(R(1 + \epsilon))^2}{(NT / L)}  +  (R\epsilon)^2 
\bigg)
\\& \leq \tilde{C}
\bigg(
\frac{R^{3}G \log NT}{\rho^{4} L} 
+ \frac{ R^{4}G ( \epsilon + \epsilon^{2} +  \epsilon^{3} )  }{\rho^2}
\bigg), \label{eq:ff.3}
\end{align}
where $\tilde{C} = C(c, \Gamma_1, \Gamma_2, \gamma)$ is a positive constant dependent on model parameters including $\Gamma_1, \Gamma_2, \gamma$. \\

\noindent
It can be easily verified that for any matrix, $\bA \in \Rb^{m \times n}$,
\begin{align}\label{eq:ff.2}
\frac{1}{mn} \norm{\bA}_{F}^2 & \leq \frac{1}{n} \norm{\bA}_{\infty, 2}^2.
\end{align}
Further, there is a one-to-one mapping of $\hat{f}_n(\cdot)$ (resp. $f_n(\cdot)$) to the entries of $\hStackedPage((X_1,\dots, X_N), T, L)$ (resp. $\StackedPage((f_1,\dots, f_N), T, L)$). 
Hence,
\begin{align}\label{eq:one_to_one_mapping_imp}
\imp(N, T) = \Ex\big[\frac{1}{NT}\| \hStackedPage((X_1,\dots, X_N), T, L) - \StackedPage((f_1,\dots, f_N), T, L)\|_{F}^2\big] 
\end{align} 
Therefore, from \eqref{eq:ff.3}, \eqref{eq:ff.2}, and \eqref{eq:one_to_one_mapping_imp} it follows that 
\begin{align}
\imp(N, T)
\le
C(c, \Gamma_1, \Gamma_2, \gamma)
\bigg(
\frac{R^{3}G \log NT}{\rho^{4} L} 
+ \frac{ R^{4}G ( \epsilon + \epsilon^{2} +  \epsilon^{3}  )  }{\rho^2}
\bigg)
\end{align}
This completes the proof of Theorem \ref{thm:mean_estimation_imputation_generalized}. 

\section{Proof of Theorem \ref{thm:mean_estimation_forecasting_generalized}}\label{appendix:forecasting}
The forecasting algorithm, as described in Section \ref{sec:mSSA}, computes a linear model between the recent past and immediate future to forecast. 
We shall bound the forecasting error, $\fore(N, T, L)$ as defined in \eqref{eq:fore.error}. 
We start with some setup and notations, followed by a key proposition that establishes the existence of a linear model under the setup of Theorem \ref{thm:mean_estimation_forecasting_generalized}, and then conclude with a detailed analysis of noisy, mis-specified least-squares. 

\medskip
\noindent{\em Setup, Notations.}  
For $L \geq 1, k \geq 1$, for ease of notations, we define
\begin{itemize}
\item[$\circ$] $\SP(X) = \StackedPage((X_1,\dots, X_N), T, L) \in \Rb^{L \times (NT/L)}$, 
\item[$\circ$] $\SP(f) = \StackedPage((f_1,\dots, f_N), T, L) \in \Rb^{L \times (NT/L)} $, 
\item[$\circ$] $\SPp(X) \in \Rb^{(L-1) \times (NT/L)}$ as the top $L-1$ rows of  $\StackedPage((X_1,\dots, X_N), T, L) $, 
\item[$\circ$] $\SPp(f)  \in \Rb^{(L-1) \times (NT/L)}$ as the top $L-1$ rows of   $\StackedPage((f_1,\dots, f_N), T, L)$.
\end{itemize}
It is worth noting that $\Ex[\SP(X)] = \rho \SP(f)$ and hence 
\begin{align}\label{eq:f.thm.0}
\SP_{L\cdot}(X)^T & = \rho \SP_{L\cdot}(f)^T + \eta,
\end{align} 
where $\eta \in \Rb^{(NT)/L}$ is a random vector with each component being independent, zero-mean with its distribution given as: 
it is $0$ with probability $1-\rho$ and with probability $\rho$, due to Property \ref{prop:bounded_noise}, it equals a zero-mean sub-Gaussian random variable with $\|\cdot\|_{\psi_2} \leq \gamma$.
Therefore, using arguments 
in \cite{PCR_NeurIPS, pcr_jasa}, each component of $\eta$ is an independent, zero-mean random variable with $\|\cdot\|_{\psi_2}$ bounded above by $C' (\gamma^2 + R \Gamma_1\Gamma_2)$ for some absolute constant $C' > 0$. 
Let $K = C' (\gamma^2 + R \Gamma_1\Gamma_2)$ and hence each component of $\eta$ has $\|\cdot\|_{\psi_2}$ bounded by $K$. \\

\noindent
Now, recall that for forecasting, we first apply the imputation algorithm (i.e. HSVT) to $\StackedPage((X_1,\dots, X_N), T, L)$ by replacing $\star$s, i.e. missing observations by $0$ as well as setting all the entries in the last row equal to $0$. 
Equivalently, the imputation algorithm is applied to $\SPp(X)$ after setting all missing values to $0$. 
Let $\hSPp \in \Rb^{L-1 \times (NT/L)}$ be the estimate produced from the imputation algorithm applied to $\SPp(X)$. 
Under the setup of Theorem \ref{thm:mean_estimation_imputation_generalized}, by 
following arguments identical to that of Theorems \ref{thm:hsvt.l2inf} and \ref{thm:mean_estimation_imputation_generalized}--in particular, refer to \eqref{eq:ff.3}--it follows that by selecting the right choice of $k \leq R \mult G$, we have 
\begin{align}\label{eq:l2.inf.error}
\Ex\Big[\frac{1}{(NT / L)}\| \hSPp - \SPp(f)\|_{2,\infty}^2\Big] 
& \leq  
\tilde{C}
\bigg(
\frac{R^{3}G \log NT}{\rho^{4} L} 
+ \frac{ R^{4}G ( \epsilon + \epsilon^{2} +  \epsilon^{3} )  }{\rho^2}
\bigg), 
\end{align}
where $\tilde{C} = C(c, \Gamma_1, \Gamma_2, \gamma) > 0$ is a constant dependent on $c, \Gamma_1, \Gamma_2, \gamma$. \\

\noindent
Now, the mSSA forecasting algorithm finds $\sbeta = \sbeta((X_1,\dots, X_N), T L; k)$, by 
solving the following Ordinary Least Squares (OLS):
\begin{align}\label{eq:f.thm.ols}
\sbeta & \in {\sf minimize}\quad  \|\frac{1}{\hrho} \SP(X)_{L \cdot} - {\hSPp}^T \beta \|_2^2 \quad \text{\sf over}\quad \beta \in \Rb^{L-1}. 
\end{align}
And subsequently,  ${\hSPp}^T \sbeta$ is used as the estimate for $\SP(f)_{L\cdot} \in \Rb^{NT/L}$, the $L$th row of the latent $\SP(f)$. 
The goal is to bound the forecasting error $\fore(N, T, L)$, which is given by
\begin{align}\label{eq:f.thm.1}
\fore(N, T, L) & = \Ex\Big[ \frac{1}{(NT / L)} \| \SP(f)_{L \cdot} - {\hSPp}^T \sbeta\|_2^2 \Big].
\end{align}
Therefore, our interest is in bounding $\Ex\big[\| \SP_{L \cdot}(f) -  {\hSPp}^T \sbeta\|_2^2\big]$. \\

\vspace{2mm}
\noindent
Now, we recall from Proposition \ref{prop:approx_low_rank_linear} that there exists $\pbeta \in \Rb^{L-1}$,  such that 
$$
\| \SP(f)_{L \cdot}^T - {\SPp(f)}^T \pbeta \|_\infty  \leq C_2  \epsilon,
$$
where $C_2 \coloneqq R\Gamma_1 (1 + \| \pbeta \|_{1})$.

\medskip
\noindent
{\em Bounding $\Ex\big[\| \SP_{L \cdot}(f) -{\hSPp}^T \sbeta\|_2^2\big]$.}
By \eqref{eq:f.thm.ols} and \eqref{eq:f.thm.0}
\begin{align}\label{eq:f.thm.2}
\| \frac{1}{\hrho} \SP(X)_{L \cdot} - {\hSPp}^T \sbeta \|_2^2 & \leq \| \frac{1}{\hrho} \SP(X)_{L \cdot} - {\hSPp}^T \pbeta \|_2^2 \nonumber \\
& = \| \frac{\rho}{\hrho} \SP(f)_{L \cdot} + \eta - {\hSPp}^T \pbeta \|_2^2 \nonumber \\
& = \|\frac{\rho}{\hrho} \SP(f)_{L \cdot} - {\hSPp}^T \pbeta \|_2^2 + \| \eta\|_2^2 + 2 \eta^T (\frac{\rho}{\hrho} \SP(f)_{L \cdot} - {\hSPp}^T \pbeta).
\end{align}
Also, 
\begin{align}\label{eq:f.thm.3}
\| \frac{1}{\hrho} \SP(X)_{L \cdot} - {\hSPp}^T \sbeta \|_2^2 & =  \| \frac{\rho}{\hrho} \SP(f)_{L \cdot} + \eta - {\hSPp}^T \sbeta \|_2^2 \nonumber \\
& = \|\frac{\rho}{\hrho} \SP(f)_{L \cdot} - {\hSPp}^T \sbeta \|_2^2 + \| \eta\|_2^2 + 2 \eta^T (\frac{\rho}{\hrho}\SP(f)_{L \cdot} - {\hSPp}^T \sbeta).
\end{align}
From \eqref{eq:f.thm.2} and \eqref{eq:f.thm.3}
\begin{align}\label{eq:f.thm.4}
&\Ex\big[ \|\frac{\rho}{\hrho} \SP(f)_{L \cdot} - {\hSPp}^T \sbeta \|_2^2 \big] 
\\& \leq  \Ex\big[\|\frac{\rho}{\hrho} \SP(f)_{L \cdot} - {\hSPp}^T \pbeta \|_2^2 \big] 
+ 2 \Ex\big[\eta^T {\hSPp}^T (\pbeta-\sbeta)\big]
\end{align}
$\eta$ is independent of $\hSPp$, $\pbeta$, and $\hrho$; $\Ex[\eta] = \mathbf{0}$; thus, we have that 
\begin{align}\label{eq:f.thm.2.5}
\Ex\big[\eta^T {\hSPp}^T \pbeta\big] = 0.
\end{align}
By \eqref{eq:f.thm.ols}, we have $\sbeta = {\hSPp}^{T, \dagger} \frac{1}{\hrho} \SP(X)_{L \cdot}$, where ${\hSPp}^{T, \dagger}$ is pseudo-inverse of ${\hSPp}^{T}$. 
That is, 
\begin{align}
\sbeta & = {\hSPp}^{T, \dagger} \frac{\rho}{\hrho} \SP(f)_{L \cdot} +  \frac{1}{\hrho} {\hSPp}^{T, \dagger} \eta. \label{eq:f.thm.6}
\end{align}
Using cyclic and linearity of Trace operator; the independence properties of $\eta$;  and \eqref{eq:f.thm.6}; we have
\begin{align}\label{eq:f.thm.10}
\Ex[\eta^T {\hSPp}^T \sbeta] 
& = \Ex[\eta^T {\hSPp}^T {\hSPp}^{T, \dagger} \frac{\rho}{\hrho} \SP(f)_{L \cdot}] +  \Ex[\frac{1}{\hrho} \eta^T {\hSPp}^T {\hSPp}^{T, \dagger} \eta] \nonumber \\
& =  \Ex[\eta]^T \Ex[{\hSPp}^T {\hSPp}^{T, \dagger} \frac{\rho}{\hrho}] \SP(f)_{L \cdot} + \Ex[\frac{1}{\hrho} \Tr(\eta^T {\hSPp}^T {\hSPp}^{T, \dagger} \eta)] \nonumber \\
& = \Ex[\frac{1}{\hrho} \Tr( {\hSPp}^T {\hSPp}^{T, \dagger} \eta \eta^T)] \nonumber \\
& = \Tr\big(\Ex[ \frac{1}{\hrho} {\hSPp}^T {\hSPp}^{T, \dagger}] \Ex[\eta \eta^T] \big)\nonumber \\
& \leq~C(\gamma) k/\rho,
\end{align}
where $C(\gamma)$ is a function only of $\gamma$.
To see the last inequality, we use various facts. 
First, by the definition of the HSVT algorithm ${\hSPp}^T$ has rank at most $k$. 
Second, let ${\hSPp}^T = \bU \bS \bV^T$ be the singular value decomposition of ${\hSPp}^T$, we have 
\begin{align} \label{eq:linear_svd}
{\hSPp}^T {\hSPp}^{T, \dagger} &= 
\bU \bS \bV^T \bV \bS^{\dagger} \bU^T \nonumber
\\ &= \bU \tilde{\bI} \bU^T,
\end{align}
That is, $\frac{1}{\hrho} {\hSPp}^T {\hSPp}^{T, \dagger}$ is a positive semi-definite matrix and 
$\Tr(\frac{1}{\hrho} {\hSPp}^T {\hSPp}^{T, \dagger}) \leq k / \hrho$. 
The matrix $ \Ex[\eta \eta^T]$ is diagonal with all the non-zero entries on diagonal (variance of components of $\eta)$ bounded above by a constant that depends on $\gamma$. 
For a positive semi-definite matrix $A$ and positive semi-definite diagonal matrix $B$, $\Tr(AB) \leq \|B\|_2 \Tr(A)$. For $\rho \geq C \log (NT)/\sqrt{NT}$ for large enough $C$, one can verfiy that $\Ex[1/\hrho] \leq 2/\rho$. 
This completes the justification of the last step of \eqref{eq:f.thm.10}.

\vspace{2mm}
\noindent
Now consider the term $\|\frac{\rho}{\hrho} \SP(f)_{L \cdot} - {\hSPp}^T \pbeta \|_2^2$. 
Note,
\begin{align}
\|\frac{\rho}{\hrho} \SP(f)_{L \cdot} - {\hSPp}^T \pbeta \|_2^2 
=& \| \big(\SP(f)_{L \cdot} - {\hSPp}^T \pbeta\big) + \big(\frac{\rho-\hrho}{\hrho}\big)\SP(f)_{L \cdot}  \|_2^2 \nonumber \\
& \leq 2 \| \big(\SP(f)_{L \cdot} - {\hSPp}^T \pbeta\big)\|_2^2 + 2 \| \frac{\rho-\hrho}{\hrho}\SP(f)_{L \cdot} \|_2^2. \label{eq:f.thm.11}
\end{align}
We will bound the two terms on the r.h.s of \eqref{eq:f.thm.11} separately.
We now consider the first term.
\begin{align}\label{eq:f.thm.12}
\|\SP(f)_{L \cdot} - {\hSPp}^T \pbeta \|_2^2 & \leq 2 \| \SP(f)_{L \cdot} - {\SPp(f)}^T \pbeta \|_2^2 + 2 \| {\SPp(f)}^T \pbeta  - {\hSPp}^T \pbeta \|_2^2.
\end{align}
By Proposition \ref{prop:approx_low_rank_linear} 
\begin{align}\label{eq:f.thm.13}
\| \SP(f)_{L \cdot} - {\SPp(f)}^T \pbeta \|_2 
& \leq \| \SP(f)_{L \cdot} - {\SPp(f)}^T \pbeta \|_{\infty} \sqrt{NT/L} 
~\leq~ C_2 \epsilon \sqrt{NT/L}, 
\end{align}
where we used the fact that for any $v \in \Rb^p$, $\| v\|_2 \leq \|v\|_{\infty} \sqrt{p}$.  
And, 
\begin{align}\label{eq:f.thm.14}
\| {\SPp(f)}^T \pbeta  - {\hSPp}^T \pbeta \|_2 
& = \| ({\SPp(f)} - {\hSPp})^T \pbeta \|_2 
~\leq~\|\SPp(f) - \hSPp\|_{2, \infty} \|\pbeta\|_1,
\end{align}
where we used the fact that for any $A \in \Rb^{q \times p}, v \in \Rb^p$, $\|Av\|_2 \leq \|A^T\|_{2,\infty} \|v\|_1$. 
Finally, note that 
\begin{align}\label{eq:f.thm.2.1}
\|\SP(f)_{L \cdot} - {\hSPp}^T \sbeta \|_2^2 
& \leq 2 \|\frac{\rho}{\hrho} \SP(f)_{L \cdot} - {\hSPp}^T \sbeta \|_2^2 + 2 \| \frac{\rho-\hrho}{\hrho} \SP(f)_{L \cdot}\|_2^2.
\end{align}
Using \eqref{eq:f.thm.4}, \eqref{eq:f.thm.2.5}, \eqref{eq:f.thm.10}, \eqref{eq:f.thm.11},\eqref{eq:f.thm.12}, \eqref{eq:f.thm.13},  \eqref{eq:f.thm.14}, and the bound in \eqref{eq:f.thm.2.1}, we obtain
\begin{align}\label{eq:f.thm.15}
&\Ex\big[ \| \SP(f)_{L \cdot} - {\hSPp}^T \sbeta \|_2^2 \big] 
\\& \leq 4 C(\gamma) k/\rho + 6 \Ex \big[\| \frac{\rho-\hrho}{\hrho}\SP(f)_{L \cdot} \|_2^2\big] 
+ 2 C_2 \epsilon^2 (NT/L) + 2 \|\pbeta\|_1^2 \|\SPp(f) - \hSPp\|_{2, \infty}^2.
\end{align}

\vspace{2mm}
\noindent
%
%
Note that $\|\SP(f)\|_{\infty} \le R \Gamma_1 \Gamma_2$.
Hence, $\|\SP(f)_{L \cdot}\|_2^2 \leq C(\Gamma_1, \Gamma_2) R^{2} (NT / L)$, for large enough constant $C(\Gamma_1, \Gamma_2)$ that may depend on $\Gamma_1, \Gamma_2$. 
Using the bounds derived in Lemma \ref{lemma:prelims}, one can verify that $\Ex[(\frac{\rho-\hrho}{\hrho})^2] \leq C/(NT / L)$ for large enough positive constant $C$. 
Therefore, we have that 
\begin{align}\label{eq:f.thm.3.2}
6 \Ex \big[\| \frac{\rho-\hrho}{\hrho}\SP(f)_{L \cdot} \|_2^2\big] & \leq C(\Gamma_1, \Gamma_2) R^{2}
\end{align}
Using \eqref{eq:l2.inf.error}, \eqref{eq:f.thm.3.2}, and the bound in \eqref{eq:f.thm.15}; diving by $1/(NT / L)$ on both sides; and noting $k \leq R \mult G$, we obtain
\begin{align}
& \Ex\big[ \frac{1}{(NT / L)} \| \SP(f)_{L \cdot} - {\hSPp}^T \sbeta \|_2^2 \big] 
\\& \leq 
C(c, \gamma, \Gamma_1, \Gamma_2) 
\left(
\frac{ RG }{\rho (NT / L)} 
+ \frac{ R^{2}}{(NT / L)}
+  R (1 + \| \pbeta \|_{1}) \epsilon^2 
+ \|\pbeta\|_1^2 
\bigg(
\frac{R^{3}G \log NT}{\rho^{4} L} 
+ \frac{ R^{4}G ( \epsilon + \epsilon^{2} +  \epsilon^{3} )  }{\rho^2}
\bigg)
\right)
\\& \leq 
C(c, \gamma, \Gamma_1, \Gamma_2) 
\left(
\max(1,  \|\pbeta\|_1, \|\pbeta\|^2 _1)
\bigg(
\frac{R^{3}G \log NT}{\rho^{4} L} 
+ \frac{ R^{4}G ( \epsilon + \epsilon^{2} +  \epsilon^{3} )  }{\rho^2}
\bigg)
\right) \label{eq:f.thm.3.1}
\end{align}
Letting $L=\sqrt{\min(N, T)T}$, using \eqref{eq:f.thm.3.1}, and noting that 
$$
\fore(N, T, L) = \Ex\big[ \frac{1}{(NT / L)} \| \SP(f)_{L \cdot} - {\hSPp}^T \sbeta \|_2^2 \big]
$$
completes the proof of Theorem \ref{thm:mean_estimation_forecasting_generalized}.

\subsection{Proof of Proposition \ref{prop:approx_low_rank_linear}}
\label{sec:prop:flattened_mean__approx_linear_representation}

For this proof, we utilize a modified version of the stacked Hankel matrix defined in Appendix \ref{sec:proof_prop_flattened_page_low_rank}.
Define the modified Hankel matrix for time series $f_n$, for $n \in [N]$, as $\widetilde{\Hankel}(n) \in \Rb^{T \times 2T}$, where for $i \in [T], j \in [2T]$, we have
\begin{align}
\widetilde{\Hankel}(n)_{ij} & = f_{n}(i + j - 1 -T).
\end{align}
Define $\widetilde{\stackedHankel} \in \Rb^{T \times NT}$ as the column wise concatenation of the matrices $\widetilde{\Hankel}(n)$ for $n \in [N]$, i.e., $\widetilde{\stackedHankel} \coloneqq [\widetilde{\Hankel}(1), \dots, \widetilde{\Hankel}(N)]$.
By a straightforward modification of the proof of Proposition \ref{prop:approx_low_rank_hankel_}, we have $\widetilde{\stackedHankel}$ has $\epsilon'$-rank bounded by $R \mult G$ with $\epsilon' = R \Gamma_1 \epsilon$.
That is, there exists a matrix $\sM \in \Rb^{T \times NT}$ such that,
\begin{align}\label{eq:modified_hankel_approx}
\text{rank}(\sM) \le RG, \quad \| \widetilde{\stackedHankel} -  \sM \|_\infty \le \epsilon'
\end{align}
Since $\text{rank}(\sM) \le R G $, it must be the case that within the last $R G $ rows of $\sM$, there exists at least one row, which we denote as $r^*$, that can be written as a linear combination of at most $R G $ rows above it, which we denote as $r_1, \dots, r_{R G }$. 
Specifically there exists a vector $\theta \coloneqq(\theta_1, \dots, \theta_{RG}) \in \Rb^{RG}$ such that
\begin{align}\label{eq:low_rank_hankel_linearity}
\sM_{r^*, \cdot} = \sum^{R G }_{\ell = 1} \theta_{\ell} \sM_{r_{\ell}, \cdot}
\end{align}
Hence for $j \in [2T]$, 
\begin{align}
&\bigg| \widetilde{\stackedHankel}_{r^*, j} -  \sum^{R G }_{\ell = 1} \theta_{\ell} \widetilde{\stackedHankel}_{r_{\ell}, j} \bigg| 
\\&= \bigg| \widetilde{\stackedHankel}_{r^*, j} \pm \sM_{r^*, j} -  \sum^{R G }_{\ell = 1} \theta_{\ell} \widetilde{\stackedHankel}_{r_{\ell}, j} \pm \sum^{R G }_{\ell = 1} \theta_{\ell} \sM_{r_{\ell}, t} \bigg| \\
\\&\le 
\bigg| \widetilde{\stackedHankel}_{r^*, j} - \sM_{r^*, j} \bigg|  
+  \bigg| \sum^{R G }_{\ell = 1} \theta_{\ell} \widetilde{\stackedHankel}_{r_{\ell}, j} -  \sum^{R G }_{\ell = 1} \theta_{\ell} \sM_{r_{\ell}, t}\bigg| 
+  \bigg|  \sM_{r^*, j} - \sum^{R G }_{\ell = 1} \theta_{\ell} \sM_{r_{\ell}, t} \bigg| \\
&= \bigg| \widetilde{\stackedHankel}_{r^*, j} - \sM_{r^*, j} \bigg|  
+  \bigg| \sum^{R G }_{\ell = 1} \theta_{\ell} (\widetilde{\stackedHankel}_{r_{\ell}, j} -   \sM_{r_{\ell}, t} )\bigg| \\
&\le \epsilon'  
+ \|\theta\|_1 \|   \widetilde{\stackedHankel}_{r_{\ell}, j} - \sM_{r_{\ell}, t} \|_{\infty} \\
& \le R\Gamma_1 (1 + \| \theta \|_{1}) \epsilon. \label{eq:new_approx_linearity}
\end{align}

\noindent
Observe that every entry of $\SP(f)_{L \cdot}$ appears within $\widetilde{\stackedHankel}_{r^*, \cdot}$; this can be seen by noting that $\widetilde{\stackedHankel}$ is skew-symmetric and thus every entry in the last row of $\widetilde{\stackedHankel}$ appears along the appropriate diagonal. 
Using this skew-symmetric property of $\widetilde{\stackedHankel}$ and \eqref{eq:new_approx_linearity}, it implies that by appropriately selecting entries in $\widetilde{\stackedHankel}$, there exists $\pbeta \in \Rb^{L-1}$,
$$
\| \SP(f)_{L \cdot}^T - {\SPp(f)}^T \pbeta \|_\infty  \leq R\Gamma_1 (1 + \| \beta \|_{1})\epsilon,
$$
where the non-zero entries in $\pbeta$ correspond to the entries of $\theta$.
Noting that $\theta \in \Rb^{RG}$ implies $\|\pbeta\|_0 \le RG$.
This completes the proof. 
%

\section{Proof of Theorem \ref{thm:mean_estimation_forecasting_simplified_oos}} \label{appendix:oos_forecast}

\vspace{2mm}
\noindent {\bf Notation.} 
For integers $t_1< t_2$ where $t_2-t_1 +1 \ge L $,   let $\StackedPage((X_1, \dots, X_N), t_1:t_2, L)$ represents the stacked page matrix  constructed using the contiguous observations $X_n(t_1), \dots, X_n(t_2), ~\forall n \in[N]$. 
Throughout, we use the following notations:
\begin{itemize}
\item $\SP_{0}(X) = \StackedPage((X_1,\dots, X_N), 1:T, L) \in \Rb^{L \times (NT/L)}$,  with zeros replacing missing values.
\item $\SP_{1}(X) = \StackedPage((X_1,\dots, X_N), T+1:T+T_1, L) \in \Rb^{L \times (NT_1/L)}$,   with zeros replacing missing values.
\item $\SP_{0}(f) = \StackedPage((f_1,\dots, f_N), 1:T, L) \in \Rb^{L \times (NT/L)} $.
\item $\SP_{1}(f) = \StackedPage((f_1,\dots, f_N), T+1:T+T_1, L) \in \Rb^{L \times (NT_1/L)} $.
\item $\SP_{1}(\eta) = \StackedPage((\eta_1,\dots, \eta_N), T+1:T+T_1, L) \in \Rb^{L \times (NT_1/L)} $. 
\item $\SPp_{0}(X) \in \Rb^{(L-1) \times (NT/L)}$ as the top $L-1$ rows of  $\SP_{0}(X)$. Let  $\SPp_{1}(X), \SPp_{0}(f)$,  $\SPp_{1}(f)$ and  $\SPp_{1}(\eta)$ be defined analogously. 
\item $\hrho \coloneqq (\max(1, \sum_{i=1}^{L-1}\sum_{j=1}^{NT/L}  \bOne( \SP_{0}(X)_{ij} \neq \star)))/(NT-NT/L)$
\end{itemize}

\vspace{2mm}
Recall that we are interested in bounding the following out-of-sample prediction error:
\begin{align}\label{eq:fore.error_ood0}
\oosfore(N, T,T_1, L) & = \frac{L}{NT_1} \sum_{n=1}^N \sum_{m'=1}^{T_1/L} \Ex\big[ (f_n(T + L \mult m') - \bar{f}_n(T +  L\mult m'))^2\big].
\end{align}
Where the forecasted estimate $\bar{f}_n(\cdot), ~n \in [N]$ are produced by the  algorithm detailed in Section \ref{sec:mSSA}.

Based on the algorithm,  we can write $\oosfore(N, T,T_1, L) $ as follows:

\begin{align}\label{eq:fore.error_ood_mr1}
\oosfore(N, T,T_1, L) & =   \frac{1}{(NT_1 / L)} \Ex\Big[  \|  \frac{1}{\hrho}  {\SPp_1(X)}^T \sbeta  - \SP_1(f)_{L \cdot}^T \|_2^2\Big] \\
 & =  \frac{1}{(NT_1 / L)} \Ex\Big[  \|  \frac{1}{\hrho}   {\SPp_1(X)}^T \sbeta - {\SPp_{1}(f)}^T \pbeta \|_2^2 \Big].
\end{align}

Before bounding this term, we introduce the following important notation.
For $i \in \{0,1\}$, let ${\bU}_i{\bSigma}_i{\bV}^T_i$ denote the Singular Value Decomposition (SVD) of $\SPp_{i}(f)$. 
Also, let  ${\btdU}_i{\btdSigma}_i{\btdV}^T_i$  denote the top k singular components  of the SVD of  
$\SPp_{i}(X)$, while $ \btdU^\perp_i {\btdSigma}_i^{\perp} ({\btdV}_i^{\perp})^T$  denote the remaining $L-k-1$ components such that  $\SPp_{i}(X) ={\btdU}_i{\btdSigma}_i{\btdV}^T_i+\btdU^\perp_i {\btdSigma}_i^{\perp} ({\btdV}_i^{\perp})^T $. 
Finally, let  ${\bV}_i^{\perp}$ and $\bU^\perp_i$  be matrices  of orthornormal basis vectors that span the null space of $\SPp_{i}(f)$ and $\SPp_{i}(f)^T$, respectively.    
Further, let $\hSPpi$ be the HSVT estimate of $\SPp_{i}(f)$. That is $\hSPpi = \frac{1}{\hrho}{\btdU}_i{\btdSigma}_i{\btdV}^T_i $. Also, let  $\hSPpoi = \frac{1}{\hrho}\btdU^\perp_i {\btdSigma}_i^{\perp} ({\btdV}_i^{\perp})^T$.

We start the proof by providing a deterministic upper bound for out-of-sample error. 

 \vspace{2mm}
\noindent \textbf{Deterministic Bound.} Due to triangle inequality,  we have

\begin{align}\label{eq:fore.error_ood_mr2}
 \|  \frac{1}{\hrho}   {\SPp_1(X)}^T \sbeta - {\SPp_{1}(f)}^T \pbeta \|_2^2  &=     \|   \frac{1}{\hrho}   {\SPp_1(X)}^T \sbeta - {\SPp_{1}(f)}^T \pbeta +   {\hSPpo}^T \sbeta -  {\hSPpo}^T \sbeta  \|_2^2   \\ 
 & \leq    2\|  \frac{1}{\hrho}  {\SPp_1(X)}^T \sbeta -  {\hSPpo}^T \sbeta  \|_2^2 +2 \|   {\hSPpo}^T \sbeta -  {\SPp_{1}(f)}^T \pbeta  \|_2^2.
\end{align}

Next, we proceed to bound each of the two terms on the right hand side. 

\vspace{2mm}
\textit{First term: $  \| \frac{1}{\hrho}  {\SPp_1(X)}^T \sbeta  -   {\hSPpo}^T \sbeta\|_2^2 $}.

\begin{align}\label{eq:ood_f1}
 \| \frac{1}{\hrho}  {\SPp_1(X)}^T \sbeta  - {\hSPpo}^T \sbeta\|^2_2  &=  \|  (\hSPpop)^T \sbeta\|^2_2    \\ 
&=   \|  \frac{1}{\hrho}  \btdV^\perp_1 {\btdSigma}_1^{\perp} ({\btdU}_1^{\perp})^T\sbeta \|^2_2    \\ 
& \leq    \|  \frac{1}{\hrho}  {\btdSigma}_1^{\perp} \|^2_2    \|({\btdU}_1^{\perp})^T\sbeta \|^2_2.
\end{align}

Note that $ \| {\btdSigma}_1^{\perp} \|_2$ equals the $(k+1)$-th singular value of $\SPp_{1}(X)$.
Recall that $\Ex[\SPp_1(X)] = \rho \SPp_1(f)$ and hence 
\begin{align}\label{eq:f.thm.0_ood}
\SPp_1(X) & = {\rho}  \SPp_1(f) + \mathbf{\zeta}_1,
\end{align} 
where $\mathbf{\zeta}_1 \in \Rb^{(L -1)\times (NT_1)/L}$ is a random matrix with zero-mean i.i.d. entries where
each entry is $0$ with probability $1-\rho$ and equals a zero-mean sub-Gaussian random variable with $\|\cdot\|_{\psi_2} \leq \gamma$ with probability $\rho$ (due to Property \ref{prop:bounded_noise}).
Next, we show that each component of $\mathbf{\zeta}_1$ is an independent, zero-mean random variable with $\|\cdot\|_{\psi_2}$ bounded above by $ C' (\gamma + R \Gamma_1\Gamma_2)$ for some absolute constant $C' > 0$. 
Let $\zeta_{ij}$ for $ i \in [L-1]$ and $j \in [NT/L]$ denotes the  $ij$-th entry in $\mathbf{\zeta}_1$.  Further, let $P_{ij} \in \{0,1\}$ denotes the random mask which takes the value $1$ with probability $\rho$ such that  $\SPp_1(X)_{ij} = P_{ij} (\SPp_1(f)_{ij}  + \SPp_1(\eta)_{ij}) $. Then, we have
\begin{align}\label{eq:ood_f2}
\|\zeta_{ij}\|_{\psi_2} &= \|	\SPp_1(X)_{ij}  - \rho \SPp_1(f)_{ij} 	\|_{\psi_2}  \\ 
&= \|	 P_{ij} \SPp_1(f)_{ij}  + P_{ij}  \SPp_1(\eta)_{ij}  - \rho \SPp_1(f)_{ij} 	\|_{\psi_2}  \\ 
&\leq  \|	 P_{ij}  \SPp_1(\eta)_{ij}  \|_{\psi_2}  + \| P_{ij}  \SPp_1(f)_{ij}   - \rho \SPp_1(f)_{ij} \|_{\psi_2}  \\ 
&\leq  C\gamma +  \SPp_1(f)_{ij}  \| P_{ij}   - \rho  \|_{\psi_2}  \\ 
&\leq  C'(\gamma +  R \Gamma_1 \Gamma_2),
\end{align} 
where $C, C'>0$ are absolute constants. The first inequality is due to triangle inequality, and the last follows since $P_{ij}   - \rho$ is a random variable bounded between $[-\rho, 1-\rho]$ and $\SPp_1(f)_{ij}$   is bounded by $R \Gamma_1 \Gamma_2$.
With a similar argument, we can also write 

\begin{align}\label{eq:f.thm.0_ood_2}
\SPp_0(X) & = {\rho}  \SPp_0(f) + \mathbf{\zeta}_0,
\end{align} 
where each component of $\mathbf{\zeta}_0$ is again an independent, zero-mean random variable with $\|\cdot\|_{\psi_2}$ bounded above by $ C' (\gamma + R \Gamma_1\Gamma_2)$. Now, recalling that   $ \SPp_1(X)  =  {\rho}  \SPp_1(f) + \mathbf{\zeta}_1$ and using Weyl's inequality ({see Lemma \ref{lemma:weyl}}), we can bound the $(k+1)$-th singular value of $ \SPp_{1}(X)$ by the largest singular value of $\mathbf{\zeta}_1$. That is, 
\begin{align}\label{eq:ood_f3}
\| {\btdSigma}_1^{\perp} \|^2_2 &\leq  \| \mathbf{\zeta}_1\|^2_2.
\end{align} 

Next, we bound the term $\|({\btdU}_1^{\perp})^T\sbeta \|^2_2 $.

\begin{align}\label{eq:ood_f4}
\|({\btdU}_1^{\perp})^T\sbeta \|^2_2 &=  \| {\btdU}_1^{\perp} ({\btdU}_1^{\perp})^T \sbeta  \ \|^2_2  \\ 
&= \| {\btdU}_1^{\perp}  ({\btdU}_1^{\perp})^T \pbeta  + {\btdU}_1^{\perp}  ( {\btdU}_1^{\perp})^T (\sbeta - \pbeta)    \|^2_2  \\ 
&\leq  2\| {\btdU}_1^{\perp}  ({\btdU}_1^{\perp})^T \pbeta  \|^2_2 + 2\| {\btdU}_1^{\perp} ( {\btdU}_1^{\perp})^T (\sbeta - \pbeta)  \|^2_2    \\ 
&\leq   2\| {\btdU}_1^{\perp}  ({\btdU}_1^{\perp})^T \pbeta  \|^2_2 + 2\| \sbeta - \pbeta  \|^2_2.
\end{align} 

First,  consider 
\begin{align}\label{eq:ood_f5}
 \| {\btdU}_1^{\perp}  ({\btdU}_1^{\perp})^T \pbeta  \|_2   &=  \| {\btdU}_1^{\perp}  ({\btdU}_1^{\perp})^T {\bU}_1 ({\bU}_1)^T \pbeta  \|_2  \\
 &\leq   \left\| {\bU}_1^{\perp}  ({\bU}_1^{\perp})^T {\bU}_1 ({\bU}_1)^T \pbeta  \right\|_2 +  \left\| \left( {\btdU}_1^{\perp}  ({\btdU}_1^{\perp})^T {\bU}_1 ({\bU}_1)^T - {\bU}_1^{\perp}  ({\bU}_1^{\perp})^T {\bU}_1 ({\bU}_1)^T \right) \pbeta  \right\|_2  \\
 &\leq    \left\| \left( {\btdU}_1^{\perp}  ({\btdU}_1^{\perp})^T  - {\bU}_1^{\perp}  ({\bU}_1^{\perp})^T\right) \pbeta  \right\|_2  \\
  &\leq   \left \|  {\btdU}_1^{\perp}  ({\btdU}_1^{\perp})^T  - {\bU}_1^{\perp}  ({\bU}_1^{\perp})^T \right\|_2  \left\| \pbeta  \right\|_2  \\
   &=   \left \|  {\btdU}_1  {\btdU}_1^T  - {\bU}_1  {\bU}_1^T \right\|_2  \left\| \pbeta  \right\|_2.
\end{align} 

Where in the first equality we use the fact that $\pbeta = {\bU}_1({\bU}_1)^T \pbeta$, i.e., $\pbeta$ lives in the column space of $\SPp_1(f)$ (Property \ref{property:subspaceinclusion}).
Next, by Wedin $\sin \Theta$ Theorem {(see \cite{davis1970rotation, wedin1972perturbation})} we bound  $\left \|  {\btdU}_1  {\btdU}_1^T  - {\bU}_1  {\bU}_1^T \right\|_2$  as follows: 
\begin{align}\label{eq:davis_kahan2}
 \left \|  {\btdU}_1  {\btdU}_1^T  - {\bU}_1  {\bU}_1^T \right\|_2   \left\| \pbeta  \right\|_2 
&\leq  \frac{\| \SPp_1(X)- {\rho}     \SPp_1(f) \|_2}{\sigma_k({\rho}  \SPp_1(f))}  \left\| \pbeta  \right\|_2   \nonumber \\
&=  \frac{\| \mathbf{\zeta}_1 \|_2} {\sigma_k({\rho}  \SPp_1(f))}  \left\| \pbeta  \right\|_2. 
\end{align}

For $\| \sbeta - \pbeta \|_2$, we have:

\begin{align} \label{eq:parame_est0}
 \| \sbeta - \pbeta  \|_2^2  &=    \| {\btdU}_0^{\perp}  ( {\btdU}_0^{\perp})^T (\sbeta - \pbeta)  + {\btdU}_0  ( {\btdU}_0)^T (\sbeta - \pbeta) \|_2^2  \nonumber \\
 &=  \| {\btdU}_0^{\perp}  ( {\btdU}_0^{\perp})^T (\sbeta - \pbeta)  \|_2^2    + \| {\btdU}_0  ( {\btdU}_0)^T (\sbeta - \pbeta) \|_2^2  \nonumber \\
 &=  \| {\btdU}_0^{\perp}  ( {\btdU}_0^{\perp})^T (\sbeta - \pbeta)  \|_2^2    + \|  {\btdU}_0^T (\sbeta - \pbeta) \|_2^2  \nonumber \\
 &=  \| {\btdU}_0^{\perp}  ( {\btdU}_0^{\perp})^T (\pbeta)  \|_2^2    + \| {\btdU}_0^T (\sbeta - \pbeta) \|_2^2.
\end{align} 

Note that the last equality follow from the fact that $\sbeta = {\hSPpz}^{T, \dagger} \frac{1}{\hrho} \SP_0(X)_{L \cdot}  =   {\btdU}_0({\btdSigma}_0)^{\dagger}{\btdV}^T  \SP_0(X)_{L \cdot} $,
 where ${\hSPpz}^{T, \dagger} $ is the pseudoinverse of ${\hSPpz}^T$, and thus $  ( {\btdU}_0^{\perp})^T\sbeta = 0 $. 
The first term in \eqref{eq:parame_est0}  can be bounded using the same argument in  \eqref{eq:ood_f5} and \eqref{eq:davis_kahan2}, where we utilize the fact that $\pbeta = {\bU}_0({\bU}_0)^T \pbeta$ and Wedin $\sin \Theta$ Theorem to get

\begin{align}\label{eq:davis_kahan3}
\|{\btdU}_0^{\perp}  ({\btdU}_0^{\perp})^T \pbeta  \|_2 
&\leq   \frac{\| \mathbf{\zeta}_0 \|_2} {\sigma_k({\rho}  \SPp_0(f))}   \left\| \pbeta  \right\|_2.
\end{align}

What is left is bounding  $\|  {\btdU}_0^T (\sbeta - \pbeta) \|_2^2$.  To that end, first consider

\begin{align} \label{eq:parame_est1}
\|  \hSPpz^T (\sbeta - \pbeta) \|_2^2 &\leq  2 \| \hSPpz^T \sbeta - \SPp_{0}(f)^T  \pbeta \|_2^2  + 2\|   \SPp_{0}(f)^T  \pbeta  -  \hSPpz^T  \pbeta  \|_2^2 \nonumber  \\
&\leq  2 \|   \hSPpz^T \sbeta -  \SPp_{0}(f)^T  \pbeta \|_2^2  + 2\|    \SPp_{0}(f)  -   \hSPpz   \|_{2,\infty}^2  \|\pbeta\|_1^2.
\end{align} 
Also, consider 

\begin{align} \label{eq:parame_est2}
\|  \hSPpz^T (\sbeta - \pbeta) \|_2^2 &=    (\sbeta - \pbeta)^T \frac{1}{\hrho^2}  {\btdU}_0{\btdSigma}^2_0{\btdU}^T  (\sbeta - \pbeta)  \nonumber \\
&\ge   \sigma_k(\hSPpz)^2  \| {\btdU}_0^T  (\sbeta - \pbeta)  \|_2^2. 
\end{align} 

From \eqref{eq:parame_est2} and \eqref{eq:parame_est1} we get, 
\begin{align} \label{eq:parame_est3}
\| {\btdU}_0^T  (\sbeta - \pbeta)  \|_2^2   
\leq  
\frac{2}{\sigma_k(\hSPpz)^2 }  ( \|   \hSPpz^T \sbeta -    \SPp_{0}(f)^T  \pbeta \|_2^2  + \|  \SPp_{0}(f)   - \hSPpz   \|_{2,\infty}^2  \|\pbeta\|_1^2) .
\end{align} 

Note that, similar to argument in \eqref{eq:f.thm.0_ood},  $\SP_0(X)_{L \cdot}\ = {\rho}  \SP_0(f)_{L \cdot} + \zeta^L_0$, where $\zeta^L_0$ is a vector of i.i.d. entries with  $\|\cdot\|_{\psi_2} \leq  C' (\gamma + R \Gamma_1\Gamma_2)$. 
Then the term $\|  \hSPpz^T \sbeta -   \SPp_{0}(f)^T  \pbeta \|_2^2 $ can be bounded as follows
\begin{align} \label{eq:parame_est4}
 &\|  \hSPpz^T \sbeta -   \frac{1}{\rho} \SP_0(X)_{L \cdot}\|_2^2      \nonumber \\ =   &\|  \hSPpz^T \sbeta -   \SP_0(f)_{L \cdot} - \frac{1}{\rho}  \zeta^L_0   \|_2^2  \nonumber \\
 =  &\|  \hSPpz^T \sbeta -   \SPp_{0}(f)^T  \pbeta   \|_2^2   + \| \frac{1}{\rho}  \zeta^L_0  \|_2^2   - \frac{2}{\rho} ( \hSPpz^T \sbeta -     \SPp_{0}(f)^T  \pbeta)^T  \zeta^L_0.
\end{align} 

Also, we have 

\begin{align} \label{eq:parame_est5}
 &\|\hSPpz^T \sbeta - \frac{1}{\rho} \SP_0(X)_{L \cdot}\|_2^2     \nonumber  \\ \leq   &\|  \hSPpz^T \pbeta - \frac{1}{\rho} \SP_0(X)_{L \cdot}\|_2^2   \nonumber \\ 
  =   &\|  ( \hSPpz^T  -   \SPp_{0}(f)^T    )\pbeta -  \frac{1}{\rho}  \zeta^L_0\ \|_2^2   \nonumber\\ 
  =   &\|  (\hSPpz^T -     \SPp_{0}(f)^T ) \pbeta   \|_2^2   + \|  \frac{1}{\rho}  \zeta^L_0  \|_2^2   -  \frac{2}{\rho} \left(( \hSPpz^T  -  \SPp_{0}(f)^T)  \pbeta\right)^T   \zeta^L_0. 
\end{align} 
From  \eqref{eq:parame_est4} and \eqref{eq:parame_est5} we have,

\begin{align} \label{eq:parame_est6}
 \|  \hSPpz^T \sbeta - \SPp_{0}(f)^T  \pbeta \|_2^2  \leq  \|  (\hSPpz^T -    \SPp_{0}(f)^T ) \pbeta   \|_2^2   + \frac{2}{\rho}  \left(( \hSPpz^T )  (\sbeta- \pbeta) \right)^T \zeta^L_0  \\
 \leq     \|   \hSPpz -   \SPp_{0}(f)   \|_{2,\infty}^2   \|\pbeta\|_1^2 + \frac{2}{\rho}  \left((\hSPpz^T )  (\sbeta- \pbeta) \right)^T  \zeta^L_0.
\end{align} 

Finally, from \eqref{eq:parame_est3} and \eqref{eq:parame_est6} we get

\begin{align} \label{eq:parame_est7}
\| {\btdU}_0^T  (\sbeta - \pbeta)  \|_2^2  
 &\leq  \frac{4}{\sigma_k(\hSPpz)^2 }  \left(  \|  \SPp_{0}(f)   -   \hSPpz   \|_{2,\infty}^2  \|\pbeta\|_1^2  + \frac{1}{\rho}  \left((  \hSPpz^T )  (\sbeta- \pbeta) \right)^T   \zeta^L_0  \right).
\end{align} 

From \eqref{eq:parame_est0}, \eqref{eq:davis_kahan3}, and \eqref{eq:parame_est7} we have
\begin{align} \label{eq:parame_est8}
 \| \sbeta - \pbeta  \|_2^2  &\leq     \frac{\| \mathbf{\zeta}_0 \|^2_2} {\sigma_k({\rho}  \SPp_0(f))^2}   \left\| \pbeta  \right\|^2_2 \\ &+  \frac{4}{\sigma_k(\hSPpz)^2 }  \left(  \|  \SPp_{0}(f)   -  \hSPpz   \|_{2,\infty}^2  \|\pbeta\|_1^2  + \frac{1}{\rho}  \left((  \hSPpz^T )  (\sbeta- \pbeta) \right)^T   \zeta^L_0  \right). 
\end{align}

For ease of exposition, let 
\begin{align}\label{eq:beta_def}
\Delta_1 &\coloneqq   \|  \SPp_{0}(f)   -  \hSPpz  \|_{2,\infty}^2  \|\pbeta\|_1^2  + \frac{1}{\rho}  \left( \hSPpz^T  (\sbeta- \pbeta) \right)^T   \zeta^L_0   \nonumber \\
\Delta_2 &\coloneqq  \frac{\| \mathbf{\zeta}_0 \|^2_2} {\sigma_k({\rho}  \SPp_0(f))^2}   \left\| \pbeta  \right\|^2_2 +  \frac{4}{\sigma_k(\hSPpz)^2 }  (\Delta_1 ).
\end{align}

Using this definition, \eqref{eq:ood_f1}, \eqref{eq:ood_f3},  \eqref{eq:ood_f4},  \eqref{eq:davis_kahan2}, and \eqref{eq:parame_est8}, we have 

\begin{align}\label{eq:det_bound1}
\| \frac{1}{\hrho}  {\SPp_1(X)}^T \sbeta  -  {\hSPpo}^T \sbeta\|_2^2 
&\leq    \|\frac{1}{\hrho}  \mathbf{\zeta}_1\|^2_2   
\Bigg(\frac{ 2 \| \mathbf{\zeta}_1 \|^2_2  \left\| \pbeta  \right\|_2^2 } {\sigma_k({\rho}  \SPp_1(f))^2}   + 2 \Delta_2 \Bigg) .
\end{align}

\vspace{2mm}
\textit{Second term: $   \|  {\SPp_{1}(f)}^T \pbeta -   {\hSPpo}^T \sbeta \|_2^2 $}. To bound the second term, we follow a similar proof to that shown in \cite{agarwal2020principal}.

\begin{align}\label{eq:ood_f21}
  \|  {\SPp_{1}(f)}^T \pbeta -   {\hSPpo}^T \sbeta \|_2^2    &=   \|  {\SPp_{1}(f)}^T \pbeta +    {\hSPpo}^T \pbeta -    {\hSPpo}^T \pbeta -    {\hSPpo}^T \sbeta \|_2^2 \\
  &\leq  2 \|  ({\SPp_{1}(f)} -    {\hSPpo})^T \pbeta \|_2^2  + 2\|   {\hSPpo}^T (\pbeta - \sbeta) \|_2^2. 
  \end{align}
Next, we bound the two terms on the right hand side. 
First, we bound $\|  ({\SPp_{1}(f)} -   {\hSPpo})^T \pbeta \|_2^2 $ as follows.

\begin{align}\label{eq:ood_f22}
\|  ({\SPp_{1}(f)} -    {\hSPpo})^T \pbeta \|_2^2  \leq   \|  \SPp_{1}(f)   -  \hSPpo   \|_{2,\infty}^2  \|\pbeta\|_1^2.
\end{align}

Next, we bound the second term   $ \|  {\hSPpo}^T (\pbeta - \sbeta) \|_2^2  $.

\begin{align}\label{eq:ood_f23}
\|  {\hSPpo}^T (\pbeta - \sbeta) \|_2^2  &\leq  \frac{1}{\hrho^2} \| ( {\btdV}_1{\btdSigma}_1{\btdU}^T_1   + \rho {\SPp_{1}(f)}^T  - \rho {\SPp_{1}(f)}^T ) (\pbeta - \sbeta) \|_2^2  \\
&\leq  \frac{2}{\hrho^2} \| ( {\btdV}_1{\btdSigma}_1{\btdU}^T_1   - \rho {\SPp_{1}(f)}^T)  (\pbeta - \sbeta) \|_2^2  + \frac{2\rho^2}{\hrho^2}  \|{\SPp_{1}(f)}^T (\pbeta - \sbeta) \|_2^2  \\
&\leq  \frac{2}{\hrho^2} \|  {\btdV}_1{\btdSigma}_1{\btdU}^T_1  - \rho {\SPp_{1}(f)}^T  \|_2^2 \| (\pbeta - \sbeta) \|_2^2  + \frac{2\rho^2}{\hrho^2}  \|{\SPp_{1}(f)}^T (\pbeta - \sbeta) \|_2^2.
\end{align}

Further, note that
\begin{align}\label{eq:ood_f24}
\|  {\btdV}_1{\btdSigma}_1{\btdU}^T_1 - \rho {\SPp_{1}(f)}^T  \|_2^2  &\leq 2\|  {\btdV}_1{\btdSigma}_1{\btdU}^T_1   -  {\SPp_{1}(X)}^T   \|_2^2   + 2 \|   {\SPp_{1}(X)}^T    - \rho {\SPp_{1}(f)}^T  \|_2^2   \\
&\leq  4 \|   {\SPp_{1}(X)}^T    - \rho {\SPp_{1}(f)}^T  \|_2^2  =   4 \| \mathbf{\zeta}_1\|_2^2.
\end{align}
Where the last inequality follows from the fact that $\|  {\btdV}_1{\btdSigma}_1{\btdU}^T_1    -  {\SPp_{1}(X)}^T   \|_2$ is the $k+1$-th singular value of $ {\SPp_{1}(X)}$ and hence is bounded by $\|   {\SPp_{1}(X)}^T    - \rho {\SPp_{1}(f)}^T  \|_2  $ using Weyl's inequality. 
Therefore, 
\begin{align}\label{eq:ood_f25}
\|  {\hSPpo}^T (\pbeta - \sbeta) \|_2^2  &\leq   \frac{8}{\hrho^2}  \| \mathbf{\zeta}_1 \|_2^2   \|\pbeta - \sbeta\|_2^2  + \frac{2\rho^2}{\hrho^2}  \|{\SPp_{1}(f)}^T (\pbeta - \sbeta) \|_2^2.
\end{align}

Next, we bound $ \|{\SPp_{1}(f)}^T (\pbeta - \sbeta) \|_2^2$. Recall that ${\bU}_0$  span the column space  of  $\SPp_{1}(f)$. Thus ${\SPp_{1}(f)}^T = {\SPp_{1}(f)}^T{\bU}_0{\bU}_0^T$, therefore, 
\begin{align}\label{eq:ood_f26}
 \|{\SPp_{1}(f)}^T (\pbeta - \sbeta) \|_2^2  &=   \| {\SPp_{1}(f)}^T{\bU}_0{\bU}_0^T  (\pbeta - \sbeta) \|_2^2\\
 &\leq   \| {\SPp_{1}(f)} \|_2^2 \|{\bU}_0{\bU}_0^T  (\pbeta - \sbeta) \|_2^2.
 \end{align}
Recall that ${\btdU}_0$ denote the top k left singular vectors of $\SPp_{0}(x)$, and consider 

\begin{align}\label{eq:ood_f27}
 \|{\bU}_0{\bU}_0^T  (\pbeta - \sbeta) \|_2^2 &=   \|({\bU}_0{\bU}_0^T + {\btdU}_0{\btdU}_0^T-  {\btdU}_0{\btdU}_0^T )  (\pbeta - \sbeta) \|_2^2 \\
 &\leq   2 \|{\bU}_0{\bU}_0^T - {\btdU}_0{\btdU}_0^T \|_2^2   \| \pbeta - \sbeta \|_2^2 +2 \|{\btdU}_0{\btdU}_0^T   (\pbeta - \sbeta) \|_2^2.
 \end{align}
 Using \eqref{eq:ood_f27}, \eqref{eq:parame_est7} and Wedin $\sin \Theta$ Theorem, we obtain,

\begin{align}\label{eq:ood_f28}
 \|{\bU}_0{\bU}_0^T  (\pbeta - \sbeta) \|_2^2  &\leq   \frac{2\| \mathbf{\zeta}_0 \|^2_2} {\sigma_k({\rho}  \SPp_0(f))^2}   \| \pbeta - \sbeta \|_2^2 \\
 &+ \frac{8}{\sigma_k(\hSPpz)^2 } \left(  \|  \SPp_{0}(f)  - \hSPpz   \|_{2,\infty}^2  \|\pbeta\|_1^2  + \frac{1}{\rho}  \left( \hSPpz^T   (\sbeta- \pbeta) \right)^T   \zeta^L_0  \right) .
 \end{align}
Using \eqref{eq:ood_f26} and \eqref{eq:ood_f28}, we have 

\begin{align}\label{eq:ood_f29}
 \|{\SPp_{1}(f)}^T (\pbeta - \sbeta) \|_2^2  &\leq   \| {\SPp_{1}(f)} \|_2^2  \frac{2\| \mathbf{\zeta}_0 \|^2_2} {\sigma_k({\rho}  \SPp_0(f))^2}   \| \pbeta - \sbeta \|_2^2 \\
 &+ \frac{8 \| {\SPp_{1}(f)} \|_2^2  }{\sigma_k(\hSPpz)^2 }  \left(  \|  \SPp_{0}(f)  -  \hSPpz   \|_{2,\infty}^2  \|\pbeta\|_1^2  +  \frac{1}{\rho}   \left( \hSPpz^T   (\sbeta- \pbeta) \right)^T   \zeta^L_0  \right) . 
 \end{align} 
 Finally, using \eqref{eq:ood_f29} and \eqref{eq:ood_f25}, we have 

\begin{align}\label{eq:ood_f210}
\|   {\hSPpo}^T (\pbeta - \sbeta) \|_2^2  &\leq   
\frac{8}{\hrho^2}  \| \mathbf{\zeta}_1 \|_2^2   \|\pbeta - \sbeta\|_2^2 \\
 &+ \frac{4}{\hrho^2}   \frac{\| \mathbf{\zeta}_0 \|^2_2 \| {\SPp_{1}(f)} \|_2^2 } {\sigma_k(\SPp_0(f))^2}   \| \pbeta - \sbeta \|_2^2 \\
 &+   \frac{16 \rho^2}{\hrho^2}\frac{  \| {\SPp_{1}(f)} \|_2^2  }{\sigma_k(\hSPpz)^2 }  \left(  \|  \SPp_{0}(f)   -  \hSPpz   \|_{2,\infty}^2  \|\pbeta\|_1^2  +  \frac{1}{\rho}\left(  \hSPpz^T (\sbeta- \pbeta) \right)^T   \zeta^L_0\right) . 
\end{align}

Finally, combining \eqref{eq:ood_f210},  \eqref{eq:ood_f22}, \eqref{eq:ood_f21}, and \eqref{eq:beta_def} yields, 

\begin{align}\label{eq:ood_f211}
  \|  {\SPp_{1}(f)}^T \pbeta -   {\hSPpo}^T \sbeta \|_2^2  &\leq    C \|  \SPp_{1}(f)   -  \hSPpo   \|_{2,\infty}^2  \|\pbeta\|_1^2 
+\frac{C}{\hrho^2}  \| \mathbf{\zeta}_1 \|_2^2  \Delta_2  \nonumber \\
 &+ \frac{C}{\hrho^2}   \frac{\| \mathbf{\zeta}_0 \|^2_2 \| {\SPp_{1}(f)} \|_2^2 } {\sigma_k( \SPp_0(f))^2}   \Delta_2  
 +   \frac{C \rho^2}{\hrho^2}\frac{  \| {\SPp_{1}(f)} \|_2^2 \Delta_1 }{\sigma_k(\hSPpz)^2 }.  
\end{align}
  
\vspace{2mm}
\textit{Combining}. Incorporating the two bounds in  \eqref{eq:det_bound1} and \eqref{eq:ood_f211} yields, 

\begin{align}\label{eq:final_det_bound}
 \|  \frac{1}{\hrho}   {\SPp_1(X)}^T \sbeta - {\SPp_{1}(f)}^T \pbeta \|_2^2  &\leq   C \|\frac{1}{\hrho}  \mathbf{\zeta}_1\|^2_2   
\Bigg(\frac{  \| \mathbf{\zeta}_1 \|^2_2  \left\| \pbeta  \right\|^2_2  } {\sigma_k({\rho}  \SPp_1(f))^2}  +  \Delta_2 \Bigg) 
 \nonumber   \\ &
    + C \|  \SPp_{1}(f)   - \hSPpo   \|_{2,\infty}^2  \|\pbeta\|_1^2   \nonumber \\
 &+ \frac{C}{\hrho^2}   \frac{\| \mathbf{\zeta}_0 \|^2_2 \| {\SPp_{1}(f)} \|_2^2 } {\sigma_k( \SPp_0(f))^2}   \Delta_2  
 +   \frac{C \rho^2}{\hrho^2} \frac{\| {\SPp_{1}(f)} \|_2^2 \Delta_1 }{\sigma_k(\hSPpz)^2 }.  
\end{align}
For some absolute constant $C>0$.

\vspace{2mm}
\noindent \textbf{High Probability Bound.} We start by defining the following high probability events.
Let $C(\Gamma_1, \Gamma_2, \gamma)$  be a positive constant dependent on model parameters $\Gamma_1, \Gamma_2, \gamma$, and let  $C > 0$ be some positive absolute constant,  define 
\begin{align}
      \bar{E}_1&:= \Big\{ \norm{\mathbf{\zeta}_0}_2 \le C  (\gamma + R \Gamma_1\Gamma_2)\ \sqrt{NT/L}\Big \}, \label{eq:Eb1_def}
     \\ \bar{E}_2&:= \Big\{ \norm{\mathbf{\zeta}_1}_2 \le C  (\gamma + R \Gamma_1\Gamma_2)\ \sqrt{NT_1/L}\Big \}, \label{eq:Eb2_def}
     \\ \bar{E}_3&:=\Bigg\{ \bigg(1 - \sqrt{\frac{20 \log (N T)}{ \rho NT}}\bigg) \rho \le \hrho \le \frac{1}{1 - \sqrt{\frac{20 \log (NT)}{ \rho NT}}} \rho \Bigg\}, \label{eq:Eb3_def} 
     \\ \bar{E}_4&:=\Bigg\{  \| \SPp_{0}(f)   -\hSPpz   \|_{2,\infty}^2 
   \leq   
   C(\gamma, \Gamma_1, \Gamma_2) 
   	\left(
		\frac{ (NT)^2  R^2 }{\rho^4 \sigma_k(\SPp_{0}(f) )^2 L^2}
 		+ \frac{k R^2 \log NT/L}{\rho^2}  
 	\right)  
\Bigg\}, \label{eq:Eb4_def} 
     \\ \bar{E}_5&:=\Bigg\{  \| \SPp_{1}(f)   -   \hSPpo   \|_{2,\infty}^2 
   \leq  C(\gamma, \Gamma_1, \Gamma_2) \left(\frac{ (NT_1)^2  R^2 }{\rho^4 \sigma_k(\SPp_{1}(f) )^2 L^2}
 + \frac{k R^2 \log NT_1/L}{\rho^2}  + \frac{R^2 T_1}{T} \right)  \Bigg\}. \label{eq:Eb5_def} 
 \end{align}

Using Theorem \ref{thm:subgaussian_matrix}, we have the following,
\begin{align}
\Pb(\bar{E}_1) &\ge 1 - 2\exp\left(\frac{-NT}{L}\right), \label{eq:Eb1_prb}  \\
\Pb(\bar{E}_2) &\ge 1 - 2\exp\left(\frac{-NT_1}{L}\right).  \label{eq:Eb2_prb}
\end{align}

Further by Lemma \ref{lemma:prelims},  $\Pb(\bar{E}_3) \ge 1 - \frac{2}{(NT)^{10}}$. 
Finally, the probabilities  of $\bar{E}_4$ and $\bar{E}_5$ are bounded as we show next.  
\begin{lemma} \label{lemma:prelims_bar}
Let $\bar{E}_4$ and $\bar{E}_5$ be defined as in  \eqref{eq:Eb4_def} and  \eqref{eq:Eb5_def}. Then, for a constant $C>0$,  
\begin{align}
	\Pb(\bar{E}_4) &\ge 1 - \frac{C}{(NT)^{10}},  \label{eq:barE4}
	\\ \Pb(\bar{E}_5) &\ge 1 - \frac{C}{(NT_1)^{10}} - \frac{C}{(NT)^{10}}.  \label{eq:barE5}
\end{align} 
\end{lemma}
\begin{proof} {}
\noindent {\bf Bounding $\bar{E}_4$ and $\bar{E}_5$.}
 $\Pb(\bar{E}_4)$   and $\Pb(\bar{E}_5)$ can be bounded using a direct utilization of Lemma  \ref{lemma:column_error} and the high probability events defined in Appendix \ref{appendix:hp}.
Starting with $\bar{E}_4$, using $\eqref{eq:_MCSE_inequality}$, and recalling that in this theorem setup $\epsilon = 0, \Gamma = R \Gamma_1 \Gamma_2$ (Property \ref{prop:low_rank_mean_} and Property \ref{prop:low_rank_mean_hankel}) and $\sigma = \gamma$ (Property \ref{prop:bounded_noise}), we have that with probability $1 - \frac{C}{(NT)^{10}}$, 

\begin{align}
  \| \SPp_{0}(f)   -   \hSPpz   \|_{2,\infty}^2  &\leq C \frac{  \gamma^2 (NT)^2  }{\rho^2 \sigma_k(\SPp_{0}(f) )^2 L^2}
\Big(( R \Gamma_1 \Gamma_2)^2 +  \frac{\gamma^2 }{\rho^2}\Big)  + \frac{C \gamma^2 k \log NT/L}{\rho^2}  + C (R \Gamma_1 \Gamma_2)^2  \\
&\leq C(\gamma, \Gamma_1, \Gamma_2) \left(\frac{ (NT)^2  R^2 }{\rho^4 \sigma_k(\SPp_{0}(f) )^2 L^2}
 + \frac{k R^2 \log NT/L}{\rho^2}  \right).
\end{align}
A similar argument can be used for $\bar{E}_5$, while noting that the term  $\frac{C}{(NT)^{10}}$ shows up due to utilizing the estimate $\hrho$, which is estimated from the first $T$ observations.  
Precisely, we get the following, 

\begin{align}
  \| \SPp_{1}(f)   -   \hSPpo   \|_{2,\infty}^2  &\leq C \frac{  \gamma^2 (NT_1)^2  }{\rho^2 \sigma_k(\SPp_{1}(f) )^2 L^2}
\Big(( R \Gamma_1 \Gamma_2)^2 +  \frac{\gamma^2 }{\rho^2}\Big)  + \frac{C \gamma^2 k \log (NT_1/L)}{\rho^2}  + C \frac{(R \Gamma_1 \Gamma_2)^2T_1}{T} \\
&\leq C(\gamma, \Gamma_1, \Gamma_2) \left(\frac{ (NT_1)^2  R^2 }{\rho^4 \sigma_k(\SPp_{1}(f) )^2 L^2}
 + \frac{R^2k \log (NT_1/L)}{\rho^2}  + \frac{R^2 T_1}{T}  \right).
\end{align} \end{proof}

\noindent Now, given these events, we will provide the high probability bound. Let  $\bar{E}:= \bar{E}_1 \cap \bar{E}_2 \cap \bar{E}_3 \cap \bar{E}_4 \cap \bar{E}_5$. 
\begin{align} \label{eq:barEProb}
	\Pb(\bar{E}^c) &\le \frac{C_0}{(NT)^{10}} + \frac{C_1}{(NT_1)^{10}},
\end{align}
for some absolute constants $C_0, C_1>0$.
Note that under event  $\bar{E}_3$, we have that $\hrho  \geq \rho  \bigg(1 - \sqrt{\frac{20 \log (N T)}{ \rho NT}}\bigg)$.  
By further using the assumption $\rho \geq C \log (N T)/ \sqrt{N T}$ for a sufficiently large $C$  we have that  $\hrho  \geq C' \rho $ and $\frac{(\hrho - \rho)^2}{\hrho^2} \leq \frac{C}{  \sqrt{NT}} $ .
Now, recall $\Delta_1$ and $\Delta_2$ definition in \eqref{eq:beta_def}. Under  event $\bar{E}$, we can bound $\Delta_1$ as follows,	 
\begin{align} \label{eq:delta_1_bound1}
\Delta_1 
&= 
	\| \hSPpz -  \SPp_0(f)\|^2_{2,\infty}  \| \pbeta  \|^2_1    + \frac{1}{\rho}  \left( \hSPpz^T  (\sbeta- \pbeta) \right)^T   \zeta^L_0  \\ 
	&\le 
	C(\gamma, \Gamma_1, \Gamma_2)  \| \pbeta  \|^2_1  
   	\left(
		\frac{ (NT)^2  R^2 }{\rho^4 \sigma_k(\SPp_{0}(f) )^2 L^2}
 		+ \frac{k R^2 \log (NT/L)}{\rho^2}  
 	\right)  + \frac{1}{\rho}  \left( \hSPpz^T  (\sbeta- \pbeta) \right)^T   \zeta^L_0.
\end{align}
Similarly,  under  event $\bar{E}$, we can bound $\Delta_2$ as follows,
\begin{align} \label{eq:delta_2_bound1}
\Delta_2 
&\le 
	C(\gamma , \Gamma_1, \Gamma_2)  \| \pbeta  \|^2_1 
	\Bigg(
		\frac{NT R^2}{L\sigma_k({\rho}  \SPp_0(f))^2} 
		+ \frac{1}{\sigma_k(\hSPpz)^2 }   
   	\left(
		\frac{ (NT)^2  R^2 }{\rho^4 \sigma_k(\SPp_{0}(f) )^2 L^2}
 		+\frac{k R^2 \log NT/L}  {\rho^2}
 	\right) 
	\Bigg)
	 \\ &+ \frac{C }{ \rho\sigma_k(\hSPpz)^2 }  \left( \left(\hSPpz^T  (\sbeta- \pbeta) \right)^T   \zeta^L_0 \right).
 \end{align}
Further, using Weyl's inequality ({see Lemma \ref{lemma:weyl}}), we can bound $|\sigma_k(\hSPpz) - \sigma_k(\SPp_{0}(f))|$ as follows, 
\begin{align} \label{eq:expectation_term2_2}
 {|\sigma_k(\hSPpz) - \sigma_k(\rho\SPp_{0}(f))|} &=   \frac{1}{\hrho} {|\sigma_k(\btdSigma_0) - \hrho \sigma_k(\SPp_{0}(f))|} \\
 & \leq \frac{1}{\hrho} {|\sigma_k(\btdSigma_0) - \rho \sigma_k(\SPp_{0}(f))|}  +  \frac{|\hrho - \rho|}{\hrho}  \sigma_k(\SPp_{0}(f))  \\
 & \leq \frac{\|  \mathbf{\zeta}_0     \|_{2}}{\hrho}   +  \frac{|\hrho - \rho|}{\hrho}  \sigma_k(\SPp_{0}(f))  \\
\end{align}
Under  $\bar{E}$, and using property \ref{property:spectra}, we have that with probability of at least $1- \frac{1}{(NT)^{10}}$,

\begin{align} \label{eq:expectation_term2_3}
 \frac{|\sigma_k(\hSPpz) - \sigma_k(\SPp_{0}(f))|}{\sigma_k(\SPp_{0}(f)) } 
 & \le \frac{ C  (\gamma + R \Gamma_1\Gamma_2) \sqrt{NT/L}}{\rho \sigma_k(\SPp_{0}(f)) }   +  \frac{|\hrho - \rho|}{\hrho}  
 \\
 &\le   \frac{ C  (\gamma + R \Gamma_1\Gamma_2) \sqrt{k} }{\rho \sqrt{ L}} + \frac{C}{\sqrt{NT}}.
\end{align}

\noindent Using  $\rho \ge C  (\gamma +R \Gamma_1\Gamma_2)  \sqrt{\frac{k}{L}}$  we get  $ \frac{1}{\sigma_k(\hSPpz)^2}  \le \frac{C}{\sigma_k(\SPp_{0}(f))^2}$.
Using property \ref{property:spectra}, we get the following bounds for $\Delta_1$ and $\Delta_2$, 

\begin{align} \label{eq:delta_1_bound}
\Delta_1 
&\le 
	C(\gamma, \Gamma_1, \Gamma_2, c)  \| \pbeta  \|^2_1  k R^2 
   	\left(
		\frac{  NT  }{L^2 \rho^4}
 		+ \frac{ \log({NT/L})}{\rho^2}  
 	\right)  + \frac{1}{\rho}  \left( \hSPpz^T  (\sbeta- \pbeta) \right)^T   \zeta^L_0.
\end{align}

\begin{align} \label{eq:delta_2_bound}
\Delta_2 
&\le 
C(\gamma , \Gamma_1, \Gamma_2, c)  {\| \pbeta  \|^2_1 }
	\Bigg(
		\frac{k R^2}{L\rho^2  } 
		+ \frac{k^2 R^2}{ NT }   
   	\left(
		\frac{ NT }{ L^2 \rho^4}
 		+ \frac{ \log (NT/L)}{\rho^2}  
 	\right) 
	\Bigg)
	 \\ &+ \frac{Ck}{\rho NT }  \left( \left(\hSPpz^T  (\sbeta- \pbeta) \right)^T   \zeta^L_0 \right) \\
	 &\le 
C(\gamma , \Gamma_1, \Gamma_2, c)  {\| \pbeta  \|^2_1 } k^2R^2
	\Bigg(
		\frac{1}{L\rho^2  } 
 		+ \frac{  \log({NT/L})}{L }  	\Bigg)
	 \\ &+ \frac{Ck}{\rho NT }  \left( \left(\hSPpz^T  (\sbeta- \pbeta) \right)^T   \zeta^L_0 \right),
 \end{align}
 where $\rho \ge C  (\gamma +R \Gamma_1\Gamma_2)  \sqrt{\frac{k}{L}}$ is used to obtain the last inequality. 
 Finally, using properties \ref{property:spectra} and \ref{property:spectra2},  $\hrho  \geq C' \rho $, and \eqref{eq:final_det_bound}, \eqref{eq:delta_1_bound}, and \eqref{eq:delta_2_bound},  we have under event $\bar{E}$,

\begin{align}\label{eq:final_prob_bound}
 &\|  \frac{1}{\hrho}   {\SPp_1(X)}^T \sbeta - {\SPp_{1}(f)}^T \pbeta \|_2^2  
 \\ &\leq  
 C(\gamma, \Gamma_1, \Gamma_2, c) \left(\frac{k^3 N T_1 R^6 }{L^2 \rho^4} + \frac{R T_1 }{T} \right)  \|\pbeta\|_1^2     
 \\& + C(\gamma, \Gamma_1, \Gamma_2, c) \left( \frac{k^3 R^6 \log ({NT/L})}{\rho^2}(\frac{NT_1}{L^2} + \frac{T_1}{T}) +  \frac{k R^2 \log ({NT_1/L})}{\rho^2} \right)  \|\pbeta\|_1^2     
 \\&+  C(\gamma, \Gamma_1, \Gamma_2, c) \frac{R^4 k^2 T_1 }{ T\rho^3}  \left( \hSPpz^T  (\sbeta- \pbeta) \right)^T   \zeta^L_0. 
 \end{align}

\noindent \textbf{Expectation Bound.} We get the bound in expectation using the high probability bound above, and by assuming that our forecast is bounded such that $ |\bar{f}_n(T +  L\mult m')| \leq R\Gamma_1 \Gamma_2$ for $m' \in [T_1/L]$.
Specifically, we have using \eqref{eq:final_prob_bound} and \eqref{eq:barEProb}, 
\begin{align}\label{eq:final_expect_bound}
& \oosfore(N, T,T_1, L) =   \frac{1}{(NT_1 / L)} \Ex\left[  \left\|  \frac{1}{\hrho}   {\SPp_1(X)}^T \sbeta - {\SPp_{1}(f)}^T \pbeta \right\|_2^2 \right] \\
&  \leq 
\frac{1}{(NT_1 / L)} \Ex\left[  \left\|  \frac{1}{\hrho}   {\SPp_1(X)}^T \sbeta - {\SPp_{1}(f)}^T \pbeta \right\|_2^2 \Bigg| \bar{E} \right] +  \frac{C R^2\Gamma_1^2 \Gamma_2^2}{(N\min(T,T_1))^{10}} \\
 &\leq \frac{L}{NT_1} C(\gamma, \Gamma_1, \Gamma_2, c) 
 \Bigg( 
 	\left(\frac{k^3 N T_1 R^6 }{L^2 \rho^4} + \frac{R T_1 }{T} \right)
   \|\pbeta\|_1^2       \\
   & +\left( \frac{k^3 R^6 \log ({NT/L})}{\rho^2}\left(\frac{NT_1}{L^2} + \frac{T_1}{T}\right) +  \frac{k R^2 \log ({NT_1/L})}{\rho^2} \right)  \|\pbeta\|_1^2     \\
   & + \frac{R^4 k^2 T_1 }{ T\rho^3}  \Ex\Big[   \left( \hSPpz^T  (\sbeta- \pbeta) \right)^T   \zeta^L_0    \Big|  \bar{E} \Big] 
   \Bigg) \\ &+  \frac{C R^2\Gamma_1^2 \Gamma_2^2}{(N\min(T,T_1))^{10}}.
 \end{align}
Noting that the $\Ex[\zeta^L_0 \big| \bar{E} ] =  \mathbf{0}$, and  $ \zeta^L_0$  is independent of $\hSPpz$, $\hrho$,  $\pbeta$ and the event $  \bar{E} $; we have 
\begin{align}\label{eq:of.thm.2.5}
\Ex \Big[   \left(\hSPpz^T  \pbeta \right)^T   \zeta^L_0 \Big] = 0.
\end{align}

By \eqref{eq:ols.mssa}, we have $\sbeta =  {\btdU}_0({\btdSigma}_0)^{\dagger}{\btdV}^T  \SP_0(X)_{L \cdot} $. 
That is, 
\begin{align}
\sbeta & =  {\btdU}_0({\btdSigma}_0)^{\dagger}{\btdV}^T \rho \SP_0(f)_{L \cdot} + {\btdU}_0({\btdSigma}_0)^{\dagger}{\btdV}^T \zeta^L_0. \label{eq:f.thm.6.ood}
\end{align}
Using cyclic and linearity of Trace operator; the independence properties of $  \zeta^L_0$;  and \eqref{eq:f.thm.6.ood}; we have
\begin{align}\label{eq:f.thm.10ood}
& \Ex \Big[   \left(\hSPpz^T  \sbeta \right)^T   \zeta^L_0 \Big] \\
& = \Ex\Big[   \left(\hSPpz^T  {\btdU}_0({\btdSigma}_0)^{\dagger}{\btdV}^T \rho \SP_0(f)_{L \cdot}  \right)^T \zeta^L_0 \  \Big] +  \Ex\Big[  \left({\btdV}_0 {\btdV}^T \zeta^L_0  \right)^T   \zeta^L_0  \Big]\nonumber \\
& =  \Ex[ \Tr( (\zeta^L_0)^T {\btdV}_0 {\btdV}^T   \zeta^L_0  )] \nonumber \\
& = \Ex[\Tr( {\btdV}_0 {\btdV}^T    \zeta^L_0   (\zeta^L_0)^T] \nonumber \\
& = \Tr\big(\Ex[ {\btdV}_0 {\btdV}^T  ] \Ex[  \zeta^L_0   (\zeta^L_0)^T] \big)\nonumber \\
& \leq   C(\gamma + \Gamma_1 \Gamma_2 R)^2  k.
\end{align}
Where to obtain the last inequality  we use the  trace property $\Tr(AB) \leq \|B\|_2 \Tr(A)$ for positive semi-definite matrices $A,B$, and that rank of $\hSPpz$ is k. 
Finally, using \eqref{eq:f.thm.10ood}, and recalling that $T_1 \geq L$ and $L \leq T$ we get,

\begin{align}\label{eq:final_expect_bound1} 
&\oosfore(N, T,T_1, L)
\\ &\leq \frac{L}{NT_1} C(\gamma, \Gamma_1, \Gamma_2, c) 
 \Bigg( 
 	\left(\frac{R^6 k^3 N T_1  }{L^2 \rho^4} + \frac{R T_1 }{T} \right)
   \|\pbeta\|_1^2       \\
   & +\left( \frac{R^6 k^3  \log ({NT/L})}{\rho^2}\left(\frac{NT_1}{L^2} + \frac{T_1}{T}\right) +  \frac{R^2 k  \log ({NT_1/L})}{\rho^2} \right)  \|\pbeta\|_1^2     
   + \frac{R^6 k^3 T_1 }{ T\rho^3} \Bigg)  \\
   &+  \frac{C R^2\Gamma_1^2 \Gamma_2^2}{(NL)^{10}} \\
    &\leq \frac{L}{NT_1} C(\gamma, \Gamma_1, \Gamma_2, c)  \max(1,  \|\pbeta\|_1^2)
 \Bigg( 
 	\frac{R^6  k^3  N T_1  }{L^2 \rho^4} + \frac{ R^6 k^3 T_1 }{T\rho^3}    \\
   & + \frac{R^6  k^3  \log ({NT})}{\rho^2}\left(\frac{NT_1}{L^2} + \frac{T_1}{T}\right) +  \frac{ R^2 k \log ({NT_1})}{\rho^2}   +  \frac{ R^2}{(NL)^{10}} \Bigg) \\
       &\leq \frac{L}{NT_1} C(\gamma, \Gamma_1, \Gamma_2, c)  \max(1,  \|\pbeta\|_1^2)
 \Bigg( 
  \frac{R^6  k^3  \log ({NT})}{\rho^4}\left(\frac{NT_1}{L^2} + \frac{T_1}{T}\right) +  \frac{ R^2 k \log ({NT_1})}{\rho^2}  \Bigg).
 \end{align}
 
 Then, with $L = \sqrt{\min(N, T) T}$, we get, 
 
\begin{align}\label{eq:final_expect_bound2}
 &\oosfore(N, T,T_1, L)  \nonumber \\ &\leq   \frac{\sqrt{\min(N, T) T}}{NT_1} C(\gamma, \Gamma_1, \Gamma_2, c)  \max(1,  \|\pbeta\|_1^2)
 \Bigg( 
  \frac{R^6  k^3  \log ({NT})}{\rho^4}\left(\frac{NT_1}{T \min(N, T)} + \frac{T_1}{T}\right) +  \frac{ R^2 k \log ({NT_1})}{\rho^2}  \Bigg)
 \nonumber \\
 &\leq     \frac{T}{T_1} \frac{\sqrt{\min(N, T) T}}{NT} C(\gamma, \Gamma_1, \Gamma_2, c)  \max(1,  \|\pbeta\|_1^2)
 \Bigg( 
  \frac{R^6  k^3  \log ({NT})}{\rho^4}\left(\frac{NT_1}{T \min(N, T)} + \frac{T_1}{T}\right) +  \frac{ R^2 k \log ({NT_1})}{\rho^2}  \Bigg)
 \nonumber \\
  &\leq     \frac{\sqrt{\min(N, T) T}}{NT} C(\gamma, \Gamma_1, \Gamma_2, c)  \max(1,  \|\pbeta\|_1^2)
 \Bigg( 
  \frac{R^6  k^3  \log ({NT})}{\rho^4}\left(\frac{N}{ \min(N, T)} + 1 \right) +  \frac{ T R^2 k \log ({NT_1})}{T_1 \rho^2}  \Bigg)
 \nonumber \\
 & \leq C(\gamma, \Gamma_1, \Gamma_2, c)  \max(1,  \|\pbeta\|_1^2)
 \Bigg( \frac{R^6  k^3  \log ({N\max(T, T_1)})}{\rho^4 \sqrt{\min(N, T) T} }   \left(\max(1, \frac{N}{T})  + \frac{T}{T_1}\right) \Bigg).
 \end{align}
 Choosing $ k = RG$ completes the proof. 
 
\section{Proof of Theorem \ref{thm:var_estimation_imputation_simplified}}
\label{sec:proof_variance_estimation_imputation}
\noindent{\bf Setup, Notations.}  
For $L \geq 1, k \geq 1$, for ease of notations, we define
\begin{itemize}
\item[$\circ$] $\SP(X) = \StackedPage((X_1,\dots, X_N), T, L) \in \Rb^{L \times (NT/L)}$, 
\item[$\circ$] $\SP(X^{2}) = \StackedPage((X^{2}_1,\dots, X^{2}_N), T, L) \in \Rb^{L \times (NT/L)}$,
\item[$\circ$] $\SP(f) = \StackedPage((f_1,\dots, f_N), T, L) \in \Rb^{L \times (NT/L)} $, 
\item[$\circ$] $\SP(f^{2}) = \StackedPage((f^{2}_1,\dots, f^{2}_N), T, L) \in \Rb^{L \times (NT/L)} $, 
\item[$\circ$] $\SP(\sigma^{2}) = \StackedPage((\sigma^{2}_1,\dots, \sigma^{2}_N), T, L) \in \Rb^{L \times (NT/L)} $, 
\item[$\circ$] $\SP(f^{2} + \sigma^{2}) = \SP(f^{2}) + \SP(\sigma^{2})$.
\end{itemize}
Recalling that $\rho = 1$, we note that 
\begin{align}
\Ex[\SP(X)] =  \SP(f), \quad \Ex[\SP(X^{2})] =  \SP(f^{2} + \sigma^{2}).
\end{align}
Further, from the definition of the variance estimation algorithm, we recall
\begin{align}
\widehat{\SP}(f) &\coloneqq \hStackedPage((X_1,\dots, X_N), T, L) = \frac{1}{\hrho} \hsvt_k(\StackedPage((X_1,\dots, X_N), T, L))
\\ \widehat{\SP}(f^{2} + \sigma^{2}) &\coloneqq \hStackedPage((X^{2}_1,\dots, X^{2}_N), T, L) = \frac{1}{\hrho} \hsvt_k(\StackedPage((X^{2}_1,\dots, X^{2}_N), T, L))
\end{align}
We denote
\begin{itemize}
\item[$\circ$] $\widehat{\SP}(f^{2}) = \widehat{\SP}(f) \circ \widehat{\SP}(f)$
\item[$\circ$] $\widehat{\SP}(\sigma^{2}) = \max\Big(\widehat{\SP}(f^{2} + \sigma^{2}) - \widehat{\SP}(f^{2}), \boldsymbol{0}\Big)$,
\end{itemize}
where $\boldsymbol{0} \in \Rb^{L \times (NT/L)}$ is a matrix of all zeroes, and we apply the $\max(\cdot)$ above entry-wise.
We remind the reader the output of the variance estimation algorithm is $\widehat{\SP}(\sigma^{2})$.
Thus, we have
$$
\frac{1}{NT} \sum_{n =1}^N \sum_{t=1}^T \big(\sigma_n(t)^2 - \hat{\sigma}_n^2(t)\big)^2 
= \frac{1}{NT} \| \SP(\sigma^{2}) - \widehat{\SP}(\sigma^{2})  \|_F^{2}. \\
$$

\noindent{\bf Initial Decomposition.}  
Note that since $\sigma^{2}_n(t) \ge 0$ for $n \in [N]$ and $t \in [T]$, we have that 
\begin{align}
&\frac{1}{NT} \| \SP(\sigma^{2}) - \widehat{\SP}(\sigma^{2})  \|_F^{2}
\\& \le \frac{1}{NT} \| \SP(\sigma^{2}) - (\widehat{\SP}(f^{2} + \sigma^{2}) - \widehat{\SP}(f^{2}))  \|_F^{2} 
\\ &= \frac{1}{NT} \| \SP(f^{2} + \sigma^{2}) - \SP(f^{2}) - (\widehat{\SP}(f^{2} + \sigma^{2}) - \widehat{\SP}(f^{2})  \|_F^{2} 
\\ &\le 
\frac{2}{NT} \| \SP(f^{2} + \sigma^{2})  - \widehat{\SP}(f^{2} + \sigma^{2})  \|_F^{2} 
+  \frac{2}{NT} \|  \SP(f^{2}) - \widehat{\SP}(f^{2})  \|_F^{2} \label{eq:variance_1.0}
\end{align}
We bound the two terms on the r.h.s of \eqref{eq:variance_1.0} separately. \\

\noindent
{\bf Bounding  $\Ex[\|  \SP(f^{2}) - \widehat{\SP}(f^{2})  \|_F^{2}$].}
\\
\begin{align}
\|  \SP(f^{2}) - \widehat{\SP}(f^{2})  \|_F^{2}
&=  \sum^N_{n=1} \sum^T_{t=1} \Big( f^2_n(t) - \hat{f}^2_n(t)  \Big)^2 \\
&= \sum^N_{n=1} \sum^T_{t=1}  \Big( f_n(t) - \hat{f}_n(t)  \Big)^2 \Big( f_n(t) + \hat{f}_n(t)  \Big)^2 \\
&\le \left[ \max_{n \in [N], t \in [T]} \Big( f_n(t) + \hat{f}_n(t)  \Big)^2 \right] \left[ \sum^N_{n=1} \sum^T_{t=1}  \Big( f_n(t) - \hat{f}_n(t)  \Big)^2 \right] \\
&\stackrel{(a)}\le C(\Gamma_1,  \Gamma_2, \Gamma_3) R^{2}  \left[\sum^N_{n=1} \sum^T_{t=1}  \Big( f_n(t) - \hat{f}_n(t)  \Big)^2 \right] \\
&= C(\Gamma_1,  \Gamma_2, \Gamma_3) R^{2} \norm{ {\SP}(f)- \widehat{\SP}(f)}_F^2 \label{eq:variance_final_1}
\end{align}

\noindent
{\bf Bounding  $\| \SP(f^{2} + \sigma^{2})  - \widehat{\SP}(f^{2} + \sigma^{2})  \|_F^{2} $.}
To bound  $\| \SP(f^{2} + \sigma^{2})  - \widehat{\SP}(f^{2} + \sigma^{2})  \|_F^{2} $, we modify the proof of Theorem \ref{thm:mean_estimation_imputation_generalized} in a straightforward manner. 
The need for the modification is that Theorem \ref{thm:mean_estimation_imputation_generalized} was proven for the case where the coordinate wise noise, $\eta_n(t) = X_n(t) - f_n(t)$ are independent sub-gaussian random variables, and $\| \eta \|_{\psi_2} \le \gamma$. 
However, one can verify that $X^2_n(t) - f^{2}_n(t) - \sigma^{2}_n(t)$ is a sub-exponential random variable with $\| \cdot \|_{\psi_1}$ norm bounded as 
\begin{align}
\norm{X^2_n(t) - f^{2}_n(t) - \sigma^{2}_n(t)}_{\psi_1} 
&\le \norm{X^2_n(t)}_{\psi_1}  \\
&= \norm{f^2_n(t) + 2 f_n(t) \eta_n(t) + \eta_n^2(t)}_{\psi_1} \\
&\le 2\norm{f^2_n(t)}_{\psi_1} + 2\norm{\eta_n^2(t)}_{\psi_1} \\
&= 2\norm{f_n(t)}^{2}_{\psi_2} + 2\norm{\eta_n(t)}^2_{\psi_2} \\
&\le C(\Gamma_1, \Gamma_2) R^{2} + 2\gamma^{2} \\
&\le C(\Gamma_1, \Gamma_2, \gamma) R^{2},
\end{align}
where we have use the standard facts that for a random variable $A$, $\| A - \Ex[A] \|_{\psi_1} \le \| A \|_{\psi_1}$ and $\| A^{2} \|_{\psi_1} = \| A \|^{2}_{\psi_2}$. \\

\noindent
Further, note that by using Properties \ref{prop:low_rank_mean_}, \ref{prop:low_rank_mean_hankel}, \ref{prop:low_rank_var_}, and \ref{prop:low_rank_var_hankel}, and a straightforward modification of Proposition \ref{prop:approx_low_rank_hankel_}, we have
\begin{align}
\text{rank}(\SP(f^{2} + \sigma^{2})) 
&\le \text{rank}(\SP(f^{2})) + \text{rank}(\SP(\sigma^{2}))
\\&\le (RG)^{2} + (R' G'),
\end{align}
where we have used that for any two matrices $\bA, \bB$, we have $\text{rank}(\bA \circ \bA) \le \text{rank}(\bA)^{2}$, where $\circ$ denotes Hadamard product, and $\text{rank}(\bA + \bB) \le \text{rank}(\bA) + \text{rank}(\bB)$.
We define $\tilde{k} \coloneqq (RG)^{2} + (R' G')$.

\medskip
\noindent
{\em Modified Theorem \ref{thm:mean_estimation_imputation_generalized}}.
Below, we state the modified version of Theorem \ref{thm:mean_estimation_imputation_generalized} to get our desired result. 
\begin{lemma}[Imputation Error]\label{lem:hsvt.l2inf_modified}
Let the conditions of Theorem \ref{thm:var_estimation_imputation_simplified} hold. Then,
\begin{align}\label{eq:lem.hsvt.l2inf_modified}
&\Ex\big[\max_{j \in [L]} \frac{1}{(NT/L)} \| \SP(f^{2} + \sigma^{2})_{L, \cdot}^{T}  - \widehat{\SP}(f^{2} + \sigma^{2})_{L, \cdot}^{T}\|_2^2\big] 
\\& \leq  
C(\Gamma_1, \Gamma_2, \Gamma'_1, \Gamma'_2, \gamma, R, R') 
\left(
\frac{  (G^{2} + G') \log^{2} NT}{ L }.
\right),
\end{align}
where $C(\Gamma_1, \Gamma_2, \Gamma'_1, \Gamma'_2, \gamma, R, R')$ is a term that depends only polynomially on $\Gamma_1$, $\Gamma_2$, $\Gamma'_1$, $\Gamma'_2$, $\gamma$, $R$, $R'$.
\end{lemma}
\begin{proof}
To reduce redundancy, we provide an overview of the argument needed for this proof, focusing only the parts of the arguments made in Theorem \ref{thm:mean_estimation_imputation_generalized} that need to be modified.
For ease of exposition, we let $\tilde{C} = C(\Gamma_1, \Gamma_2, \Gamma'_1, \Gamma'_2, \gamma, R, R')$.
We being by matching notation with that used in Theorem \ref{thm:mean_estimation_imputation_generalized}; in particular with respect to $\rho, k, \epsilon, \Gamma$.
Under the setup of Theorem \ref{thm:var_estimation_imputation_simplified}, we have $\rho = 1$, $k = \tilde{k}$, $\epsilon = 0$, $\Gamma \le \tilde{C}$
Further, recall the definition of $\bY, \bM, p, q, \sigma$ from Appendix \ref{ssec:hsvt.setup}.
We will now use $\bY = \SP(X^{2})$, and $\bM = \SP(f^{2} + \sigma^{2})), \sigma = \gamma$, $p = (NT/L), q = L$.
One can verify that there is only required change to the proof of Theorem \ref{thm:mean_estimation_imputation_generalized}; 
in particular, in the argument made to prove Theorem \ref{thm:hsvt.l2inf}, we need to re-define events $E_2, E_3, E_4$ in \eqref{eq:E2_def}, \eqref{eq:E3_def}, \eqref{eq:E4_def} for the case where $(\bY - \bM)_{ij}$ is mean-zero sub-exponential.
Using the result from 
\cite{PCR_NeurIPS, pcr_jasa},  which bounds the operator norm of a matrix with sub-exponential mean-zero entries, we have with probability at least $1 - 1/((NT)^{10})$
\begin{align}\label{eq:matrix_concentration_sub_exponential}
 \| \bY -  \bM \|_2 \le \tilde{C} \sqrt{(NT / L)} \log^{2} NT
\end{align}
As a result \eqref{eq:matrix_concentration_sub_exponential}, and standard concentration inequalities for sub-exponential random variables, we have the modified events, $\tilde{E}_2, \tilde{E}_3, \tilde{E}_4$.
\begin{align}
     \tilde{E}_2&:= \Big\{ \norm{\bY - \rho \bM}_2 \le  \tilde{C} \sqrt{(NT / L)} \log^{2} NT \Big \}, \label{eq:E2_def_modified}
     \\ \tilde{E}_3&:= \Big\{ \norm{\bY - \rho \bM}_{\infty, 2}, \norm{\bY - \rho \bM}_{2, \infty} \le \tilde{C} \sqrt{(NT / L)} \log^{2} NT \Big \}, \label{eq:E3_def_modified}
     \\ \tilde{E}_4&:= \Big\{\max_{j \in [q]} \norm{\varphi^{\bB}_{\sigma_k(\bB)} \Big( \bY_{j \cdot}^T - \rho \bM_{j \cdot}^T \Big)}_2^2 \le \tilde{C} \tilde{k} \log^{2}(NT / L) \Big \}, \label{eq:E4_def_modified}
\end{align}
Using these modified events in the proofs of Theorem \ref{thm:hsvt.l2inf} and Theorem \ref{thm:mean_estimation_imputation_generalized}, and appropriately simplifying leads to the desired result.
\end{proof}
By Lemma \ref{lem:hsvt.l2inf_modified} and \eqref{eq:ff.2}, we have that
\begin{align}
\frac{1}{NT} \Ex[\|  \SP(f^{2} + \sigma^{2})  - \widehat{\SP}(f^{2} + \sigma^{2})    \|_F^{2} 
&\le \Ex\big[\max_{j \in [L]} \frac{1}{(NT/L)} \| \SP(f^{2} + \sigma^{2})_{L, \cdot}^{T}  - \widehat{\SP}(f^{2} + \sigma^{2})_{L, \cdot}^{T}\|_2^2\big]
\\&\le C(\Gamma_1, \Gamma_2, \Gamma'_1, \Gamma'_2, \gamma, R, R') 
\left(
\frac{  (G^{2} + G') \log^{2} NT}{ L }.
\right).
\label{eq:variance_final_2}
\end{align}

\noindent
{\bf Completing proof.}
Substituting \eqref{eq:variance_final_1} and \eqref{eq:variance_final_2} into \eqref{eq:variance_1.0} and letting $L = \sqrt{\min(N, T)T}$
\begin{align}
\frac{1}{NT} \| \SP(\sigma^{2}) - \widehat{\SP}(\sigma^{2})  \|_F^{2}
\le C(\Gamma_1, \Gamma_2, \Gamma_3, \Gamma'_1, \Gamma'_2, \gamma, R, R') 
\left(
\frac{  (G^{2} + G') \log^{2} NT}{ \sqrt{\min(N, T)T} }.
\right).
\end{align}
This completes the proof.

\section{tSSA Proofs}\label{sec:tSSA_proofs}

\subsection{Proof of Propositions \ref{prop:tensor} and \ref{prop:tesnor_exp}}
Consider $n \in [N], ~\ell \in [L], ~s \in [T/L]$.  By Property \ref{prop:low_rank_mean_}, 
\begin{align}
    \tensor_{n \ell s} & = f_n((s-1)\times L + \ell) \nonumber \\
                 & = \sum_{r=1}^R U_{nr} W_{r ((s-1)\times L + \ell)}. \label{eq:cp-rank.1}
\end{align}
The Hankel matrix induced by time series $W_{r \cdot}$ has rank at most $G$ as per Property \ref{prop:low_rank_mean_hankel}. 
The Page matrix associated with it is of dimension $L \times T/L$ with entry in its $\ell$-th row and $s$-th column equal to $W_{r ((s-1)\times L + \ell)}$. 
Since this Page matrix can be viewed as a sub-matrix of the Hankel matrix, it has rank at most $G$ as well. 
That is, there exists vectors $w^r_{\ell \cdot}, v^r_{s \cdot} \in \Rb^G$ such that 
\begin{align}
W_{r ((s-1)\times L + \ell)} & = \sum_{g=1}^G w^r_{\ell g} v^r_{s g}.\label{eq:cp-rank.2}
\end{align}
From \eqref{eq:cp-rank.1} and \eqref{eq:cp-rank.2}, it follows that   
\begin{align}
\tensor_{n \ell s} & = \sum_{r=1}^R U_{nr} \Big(\sum_{g=1}^G w^r_{\ell g} v^r_{s g} \Big) \nonumber \\
& = \sum_{r \in [R], g \in [G]} U_{nr} w^r_{\ell g} v^r_{s g} \nonumber \\
& = \sum_{r\in [R], g\in [G]}  a_{n ~(r, g)} b_{ \ell ~(r, g)} c_{s~(r, g)}, \label{eq:cp-rank.3}
\end{align}
where $a_{n ~(r, g)} = U_{nr}$, $b_{ \ell ~(r, g)} = w^r_{\ell g} $ and $c_{s~(r, g)} = v^r_{s g}$.
Thus \eqref{eq:cp-rank.3} implies that $\tensor$ has CP-rank at most $R \mult G$, which completes the proof for Propositions \ref{prop:tensor}.

\smallskip
\noindent
For Proposition  \ref{prop:tesnor_exp} , by the setup and model definition, it follows $\Tensor_{n \ell s} = X_n((s-1)\times L + \ell)$. 
And $X_n((s-1)\times L + \ell) = \star$ with probability $1-\rho$ and $f_n((s-1)\times L + \ell) + \eta_n((s-1)\times L + \ell)$ with probability $\rho$, where $\eta_n((s-1)\times L + \ell)$ are independent and zero-mean. 
Therefore, it follows that the entries of $\Tensor$ are independent and 
\begin{align}
\Ex[\Tensor_{n \ell s}] & = \Ex[X_n((s-1)\times L + \ell)] \nonumber \\
& = \rho f_n((s-1)\times L + \ell) \nonumber \\
& = \rho \tensor_{n \ell s}.
\end{align}
That is, $\Ex[\Tensor] = \rho \tensor$. 
This concludes the proof.

\subsection{Proof of Proposition \ref{prop:imputation_comparisons_new} }
From Property \ref{property:te_error_rates}, and our choice of parameter $L$ for mSSA ($L = \sqrt{\min(N, T)T}$) and tSSA ($L = \sqrt{T}$), we have that 
\begin{align}
	\imp(N,T; \text{tSSA})  &= \tilde{\Theta}\left(\frac{1}{\min\left(N, \sqrt{T}\right)^{2}}\right) =  \tilde{\Theta}\left(\frac{1}{\min\left(N^{2},T\right)}\right), \label{eq:compare_tssa}
	\\ \imp(N,T; \text{mSSA})  &=   \tilde{\Theta}\left(\frac{1}{\sqrt{\min(N, T) T}}\right), \label{eq:compare_mssa}
	 \\ \imp(N,T; \text{ME})  &=  \tilde{\Theta}\left(\frac{1}{\min\left(N, T \right) }\right). \label{eq:compare_me}
\end{align}
We proceed in cases.

\vspace{2mm}
\noindent
{\bf Case 1: $T = o(N)$.}
In this case, from \eqref{eq:compare_tssa}, \eqref{eq:compare_mssa}, and \eqref{eq:compare_me}, we have 
\begin{align}
	\imp(N,T; \text{tSSA}), \ \imp(N,T; \text{mSSA}), \  \imp(N,T; \text{ME}) = \tilde{\Theta}\left(\frac{1}{T}\right)
\end{align}

\vspace{2mm}
\noindent
{\bf Case 2: $N = o(T)$.}
In this case, from \eqref{eq:compare_tssa}, \eqref{eq:compare_mssa}, and \eqref{eq:compare_me}, we have
\begin{align}
	\imp(N,T; \text{tSSA})  &= \tilde{\Theta}\left(\frac{1}{N^{2}}\right), \label{eq:compare_tssa_2}
	\\ \imp(N,T; \text{mSSA})  &=   \tilde{\Theta}\left(\frac{1}{\sqrt{N T}}\right), \label{eq:compare_mssa_2}
	 \\ \imp(N,T; \text{ME})  &=  \tilde{\Theta}\left(\frac{1}{N}\right). \label{eq:compare_me_2}
\end{align}
In this case, we have 
$$
\imp(N,T; \text{tSSA}), \imp(N,T; \text{mSSA}) = \tilde{o}(\imp(N,T; \text{ME})).
$$
It remains to compare the relative performance of tSSA and mSSA for the regime $N = o(T)$. 
Towards this, note from \eqref{eq:compare_tssa_2} and \eqref{eq:compare_mssa_2} that
\begin{align}
\imp(N,T; \text{tSSA}) &=\tilde{o}(\imp(N,T; \text{mSSA}))
\\ \iff \frac{1}{N^{2}} &=\tilde{o}(\frac{1}{\sqrt{N T}})
\\ \iff T^{1/3} & ={o}(N)
\end{align}
This completes the proof.

\subsection{Proof of Proposition \ref{prop:high_tensor}}
\begin{proposition}
Let Properties \ref{prop:low_rank_mean_tensor}, \ref{prop:low_rank_mean_hankel}, and \ref{prop:bounded_noise} hold. 
Then, for any $1 \leq L \leq \sqrt{T}$, $\hightensor$ has CP-rank at most $R \mult G$. 
Further, all entries of $\highTensor$ are independent random variables with each entry observed with probability $\rho \in (0,1]$, and $\Ex[\highTensor] = \rho \hightensor$. 
%
\end{proposition}
Consider $n_1, \dots, n_d \in [N_1] \times \dots \times [N_d], ~\ell \in [L], ~s \in [T/L]$.  
By Property \ref{prop:low_rank_mean_tensor}, 
\begin{align}
    \hightensor_{n_1, \dots, n_d, \ell, s} & = f_{n_1, \dots, n_d}((s-1)\times L + \ell) \nonumber \\
                 & =  \sum^{R}_{r=1} U_{n_1, r} \dots U_{n_d, r} \ W_{r, ((s-1)\times L + \ell)},
\end{align}
The rest of the proof follows in a similar fashion to that of Proposition \ref{prop:tensor}.

\newpage
\section{Additional Figures}\label{sec:figures}

\begin{figure}[h]
	\begin{subfigure}[t]{.48\textwidth}
		\includegraphics[width=\linewidth]{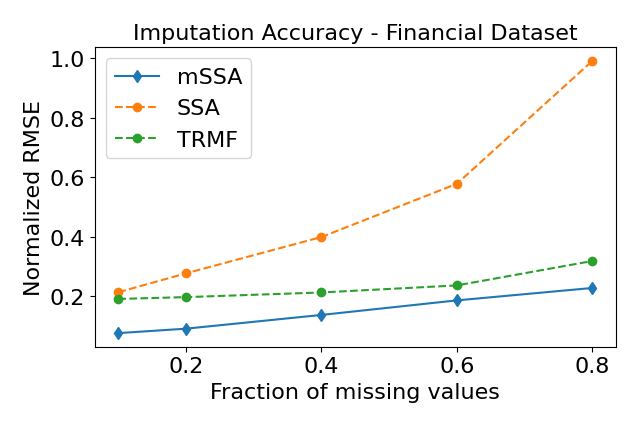}
		\caption{}
		\label{fig:141}
	\end{subfigure}
	\begin{subfigure}[t]{.48\textwidth}
		\includegraphics[width=\linewidth]{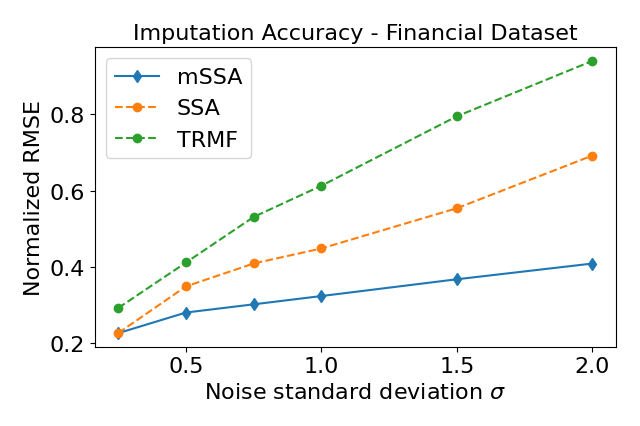}
		\caption{}
		\label{fig:142}
	\end{subfigure}
    	\begin{subfigure}[t]{.48\textwidth}
		\includegraphics[width=\linewidth]{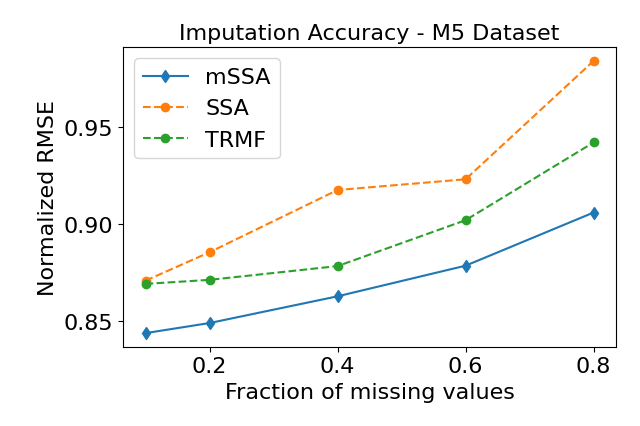}
		\caption{}
		\label{fig:151}
	\end{subfigure}
    	\begin{subfigure}[t]{.48\textwidth}
		\includegraphics[width=\linewidth]{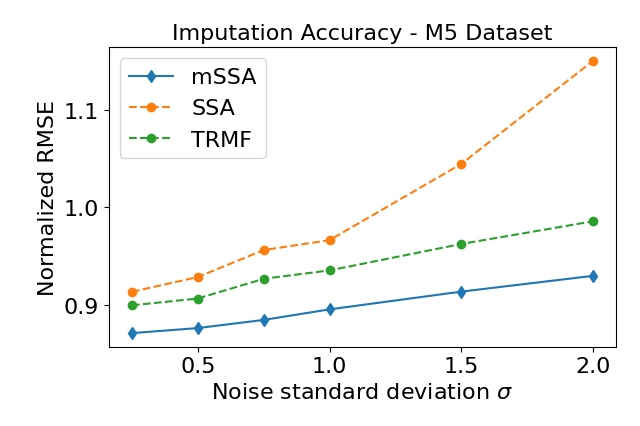}
		\caption{}
		\label{fig:152}
	\end{subfigure}
    \caption{mSSA vs. TRMF vs. SSA - imputation performance on the Financial and M5 datasets. Figures \ref{fig:141}, and \ref{fig:151}  show imputation accuracy of mSSA, TRMF and SSA as we vary the fraction of missing values; Figures  \ref{fig:142}, and \ref{fig:152} show imputation accuracy as we vary the noise level (and with 50\% of values missing).}
    \label{fig:mSSA_imp2}
\end{figure} 

\begin{figure}[h]
	\begin{subfigure}[t]{.45\textwidth}
		\includegraphics[width=\linewidth]{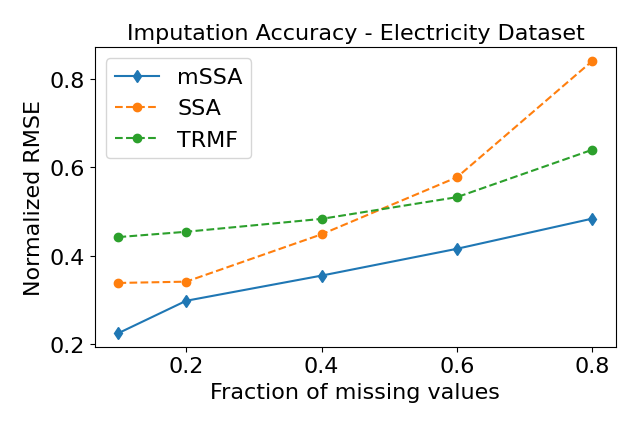}
		\caption{}
		\label{fig:111}
	\end{subfigure}
   \hfill
        \begin{subfigure}[t]{.45\textwidth}
		\includegraphics[width=\linewidth]{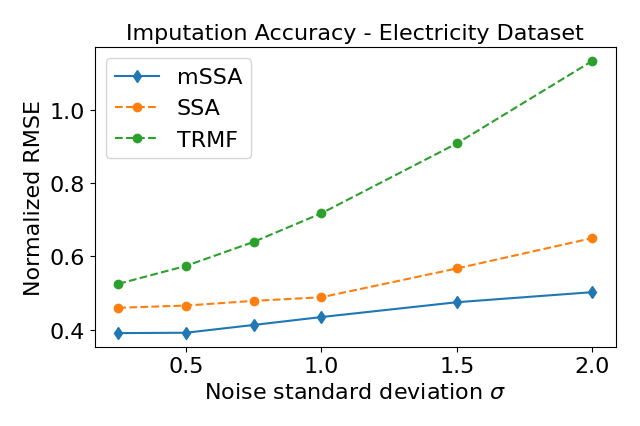}
		\caption{}
		\label{fig:112}
	\end{subfigure}
    	\begin{subfigure}[t]{.45\textwidth}
		\includegraphics[width=\linewidth]{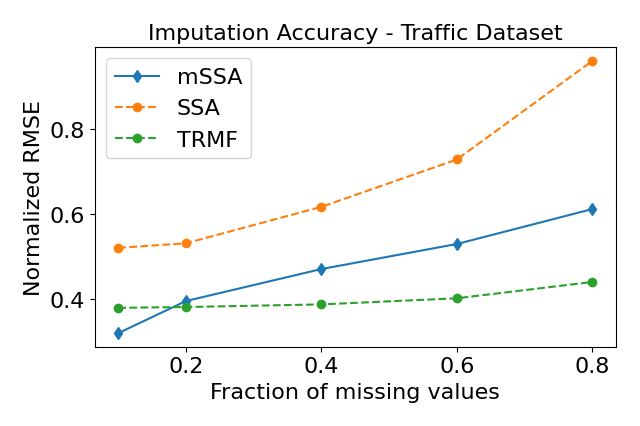}
		\caption{}
		\label{fig:121}
	\end{subfigure}
   \hfill
    	\begin{subfigure}[t]{.45\textwidth}
		\includegraphics[width=\linewidth]{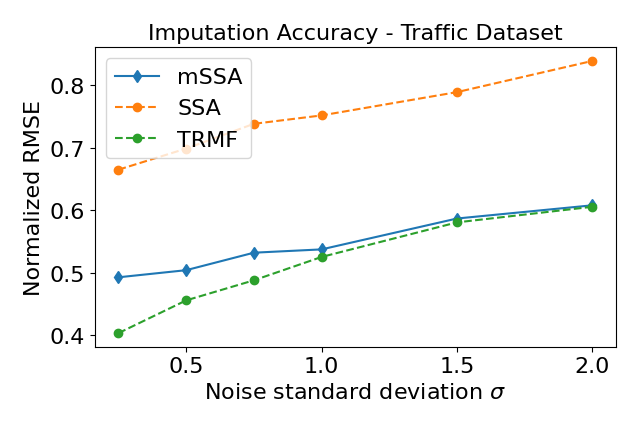}
		\caption{}
		\label{fig:122}
	\end{subfigure}

	\begin{subfigure}[t]{.45\textwidth}
		\includegraphics[width=\linewidth]{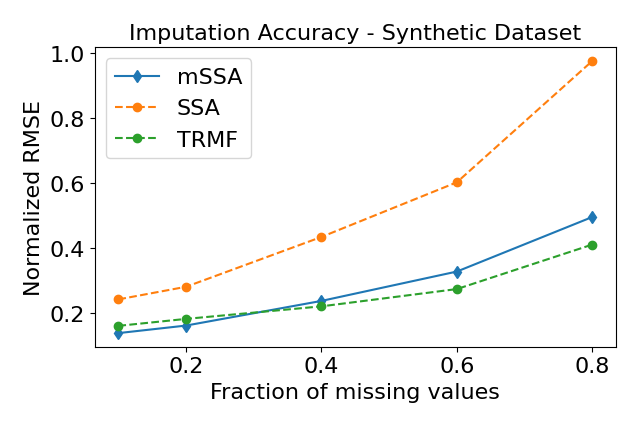}
		\caption{}
		\label{fig:131}
	\end{subfigure}
    \hfill
    	\begin{subfigure}[t]{.45\textwidth}
		\includegraphics[width=\linewidth]{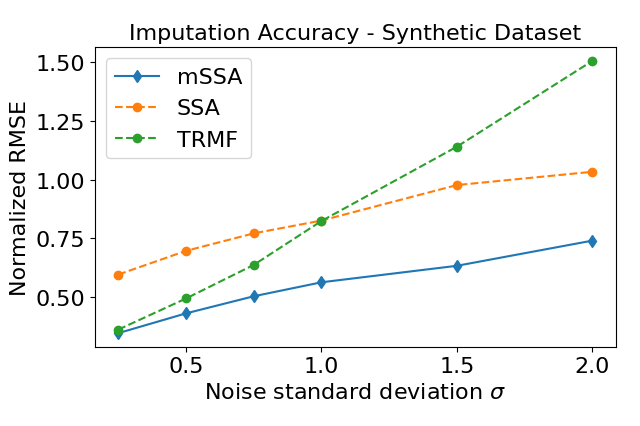}
		\caption{}
		\label{fig:132}
	\end{subfigure}

    \caption{mSSA vs. TRMF vs. SSA - imputation performance on the Electricity, Traffic and Synthetic datasets. Figures  \ref{fig:111}, \ref{fig:121}, and \ref{fig:131}, show imputation accuracy of mSSA, TRMF and SSA as we vary the fraction of missing values; Figures  \ref{fig:112},  \ref{fig:122}, and \ref{fig:132} show imputation accuracy as we vary the noise level (and with 50\% of values missing).}
    \label{fig:mSSA_imp1}
    \vspace{-2mm}
\end{figure}

\begin{figure}[h]
	\begin{subfigure}[t]{.48\textwidth}
		\includegraphics[width=\linewidth]{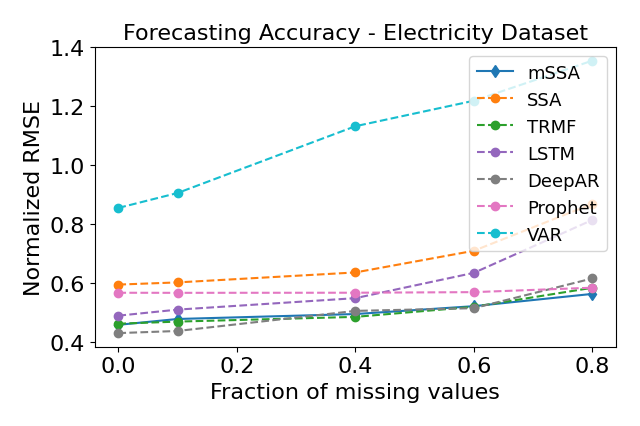}
		\caption{}
		\label{fig:211}
	\end{subfigure}
	\begin{subfigure}[t]{.48\textwidth}
		\includegraphics[width=\linewidth]{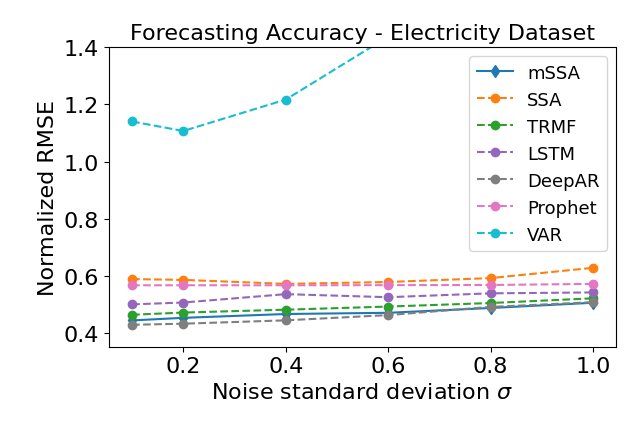}
		\caption{}
		\label{fig:212}
	\end{subfigure}
    	\begin{subfigure}[t]{.48\textwidth}
		\includegraphics[width=\linewidth]{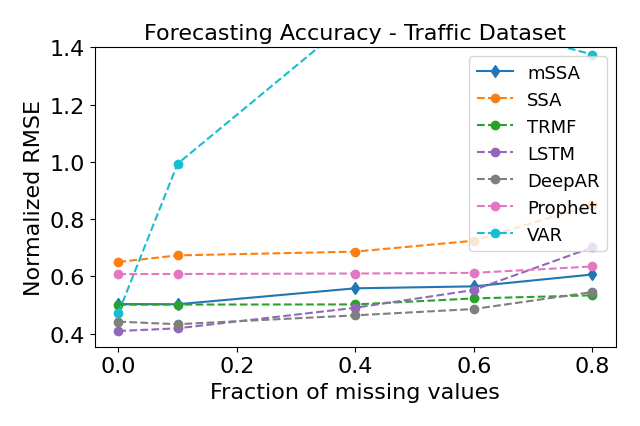}
		\caption{}
		\label{fig:221}
	\end{subfigure}
    	\begin{subfigure}[t]{.48\textwidth}
		\includegraphics[width=\linewidth]{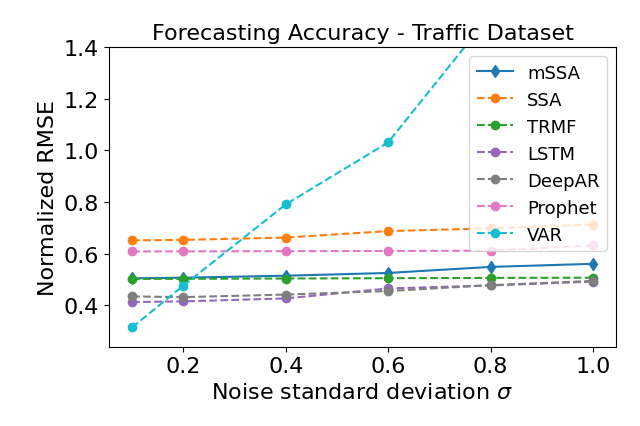}
		\caption{}
		\label{fig:222}
	\end{subfigure}
	    	\begin{subfigure}[t]{.48\textwidth}
		\includegraphics[width=\linewidth]{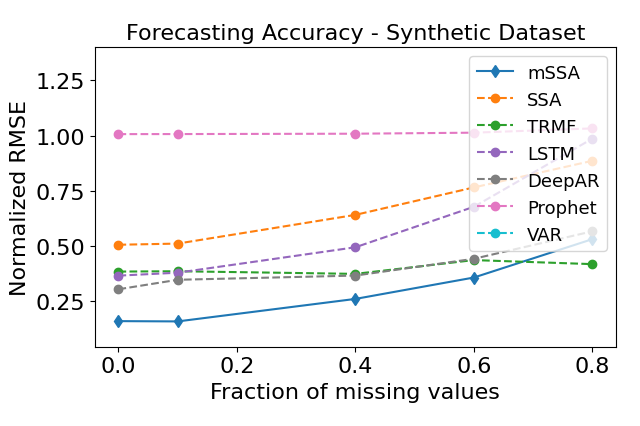}
		\caption{}
		\label{fig:231}
	\end{subfigure}
    	\begin{subfigure}[t]{.48\textwidth}
		\includegraphics[width=\linewidth]{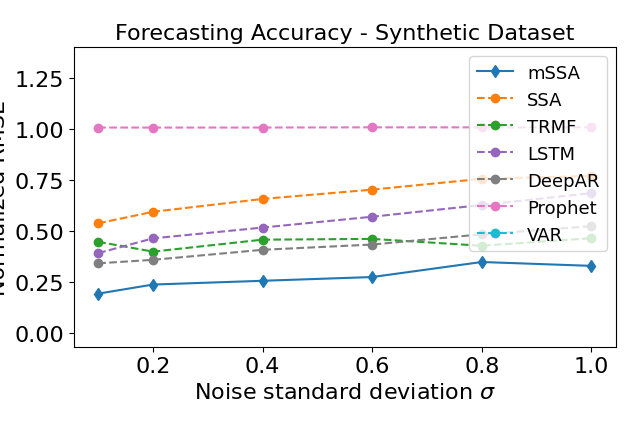}
		\caption{}
		\label{fig:232}
	\end{subfigure}
    \caption{ mSSA forecasting performance  on standard multivariate time series benchmark is competitive with/outperforming industry standard methods as we vary the number of missing data and  noise level.  Figures  \ref{fig:211}, \ref{fig:221}, and \ref{fig:231} show the forecasting accuracy of all methods (some of VAR results are not shown due to its relatively high error) on the Electricity, Traffic and Synthetic datasets with varying fraction of missing values; Figures \ref{fig:212}, \ref{fig:222}, and \ref{fig:232}  shows the forecasting accuracy   on the same  datasets with varying noise level.}
    \label{fig:mSSA_fore1}
\end{figure}

\begin{figure}[h]
    	\begin{subfigure}[t]{.48\textwidth}
		\includegraphics[width=\linewidth]{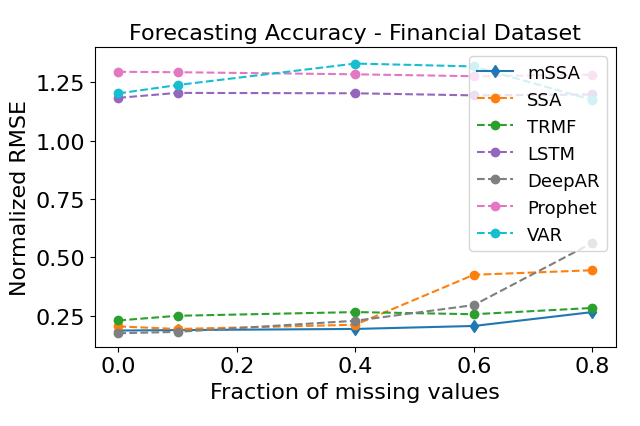}
		\caption{}
		\label{fig:241}
	\end{subfigure}
    	\begin{subfigure}[t]{.48\textwidth}
		\includegraphics[width=\linewidth]{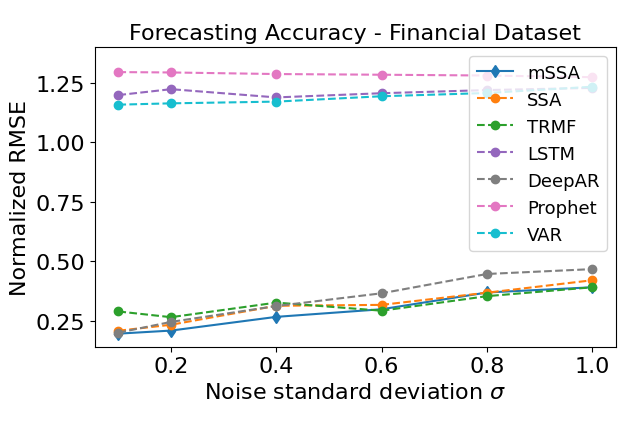}
		\caption{}
		\label{fig:242}
	\end{subfigure}
	    	\begin{subfigure}[t]{.48\textwidth}
		\includegraphics[width=\linewidth]{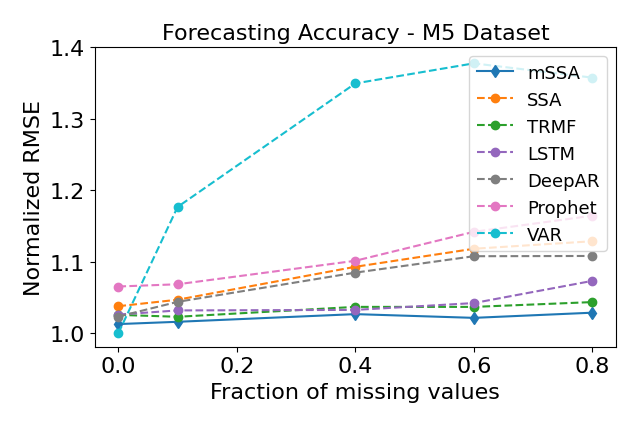}
		\caption{}
		\label{fig:251}
	\end{subfigure}
    	\begin{subfigure}[t]{.48\textwidth}
		\includegraphics[width=\linewidth]{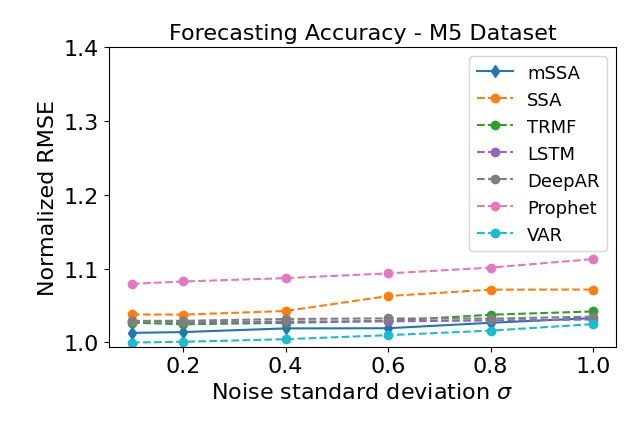}
		\caption{}
		\label{fig:252}
	\end{subfigure}
    \caption{ Figures \ref{fig:241}, and \ref{fig:251}  show the forecasting accuracy of all methods (some of VAR results are not shown due to its relatively high error) on the financial and M5 datasets with varying fraction of missing values; Figures \ref{fig:242}, and \ref{fig:252} show the forecasting accuracy   on the same  datasets with varying noise levels.}
    \label{fig:mSSA_fore2}
\end{figure}

\begin{figure}[h!]
\centering
\begin{subfigure}{0.40\textwidth}
\begin{tikzpicture}
\definecolor{color0}{rgb}{0.12156862745098,0.466666666666667,0.705882352941177}
\definecolor{color1}{rgb}{1,0.498039215686275,0.0549019607843137}
\definecolor{color2}{rgb}{0.172549019607843,0.627450980392157,0.172549019607843}

\begin{axis}[
legend cell align={left},
width = \linewidth,
legend style={
  fill opacity=0.0,
  draw opacity=1,
  text opacity=1,
  at={(0.03,0.97)},
  anchor=north west,
  draw=none
},
log basis x={10},
log basis y={10},
tick align=outside,
tick pos=left,
title={Training Time - Synthetic Dataset},
x grid style={white!69.0196078431373!black},
xlabel={Timesteps (T)},
xmin=340.535969008314, xmax=11746.1894308802,
xmode=log,
xtick style={color=black},
y grid style={white!69.0196078431373!black},
ylabel={Seconds},
ymin=0.00990452154238732, ymax=825.334475947086,
ymode=log,
ytick style={color=black}
]
\addplot [very thick, color0, mark=asterisk, mark size=1.5, mark options={solid}]
table {%
400 0.016577005
600 0.0379211902618408
800 0.0569393634796142
1000 0.0500750541687011
2000 0.0806777477264404
4000 0.223652601242065
6000 0.255280495
8000 0.354083538
10000 0.471513033
};
\addlegendentry{mSSA}
\addplot [very thick, color1, dashed, mark=*, mark size=1.5, mark options={solid}]
table {%
400 0.4852283
600 0.837280035
800 1.445567131
1000 2.04690361
2000 7.227142334
4000 30.48887277
6000 86.89790583
8000 184.2927191
10000 323.5471435
};
\addlegendentry{hSSA}
\addplot [very thick, color2, dashed, mark=triangle*, mark size=1.5, mark options={solid}]
table {%
400 0.387770653
600 0.670904636
800 1.016738176
1000 1.500113249
2000 6.171154261
4000 35.16030359
6000 112.8221099
8000 259.0603824
10000 493.1254528
};
\addlegendentry{vSSA}
\end{axis}
\end{tikzpicture}
\caption{}
\label{fig:syn_train_time}
 \end{subfigure}
 \hspace{4mm}
   \begin{subfigure}{0.40\textwidth}
\begin{tikzpicture}
\definecolor{color0}{rgb}{0.12156862745098,0.466666666666667,0.705882352941177}
\definecolor{color1}{rgb}{1,0.498039215686275,0.0549019607843137}
\definecolor{color2}{rgb}{0.172549019607843,0.627450980392157,0.172549019607843}

\begin{axis}[
legend cell align={left},
width = \linewidth,
legend style={
  fill opacity=0.0,
  draw opacity=1,
  text opacity=1,
  at={(0.03,0.97)},
  anchor=north west,
  draw=none
},
log basis x={10},
log basis y={10},
tick align=outside,
tick pos=left,
title={Training Time - Electricity Dataset},
x grid style={white!69.0196078431373!black},
xlabel={Timesteps (T)},
xmin=328.936063770756, xmax=24320.8358131732,
xmode=log,
xtick style={color=black},
y grid style={white!69.0196078431373!black},
ylabel={Seconds},
ymin=0.0100792516593312, ymax=866.779489175226,
ymode=log,
ytick style={color=black}
]
\addplot [very thick, color0, mark=asterisk, mark size=1.5, mark options={solid}]
table {%
400 0.016893625
600 0.0430974960327148
800 0.043970585
1000 0.045286655
2000 0.0832948684692382
4000 0.185041189193726
6000 0.239439249
8000 0.381483793258667
10000 0.493450403
20000 0.992658138
};
\addlegendentry{mSSA}
\addplot [very thick, color1, dashed, mark=*, mark size=1.5, mark options={solid}]
table {%
400 0.488216639
600 0.969156504
800 1.349081278
1000 2.471289873
2000 8.351601601
4000 36.89091086
6000 104.4114544
8000 220.7085018
10000 354.0430982
};
\addlegendentry{hSSA}
\addplot [very thick, color2, dashed, mark=triangle*, mark size=1.5, mark options={solid}]
table {%
400 0.434861898
600 0.768960238
800 1.172620058
1000 1.724573851
2000 7.031897783
4000 39.27743149
6000 124.9336317
8000 282.1014705
10000 517.1470661
};
\addlegendentry{vSSA}
\end{axis}
\end{tikzpicture}
\caption{}
\label{fig:elec_train_time}
 \end{subfigure}
  \caption{The training time of the original mSSA variants (hSSA in the orange dotted line and vSSA in the green dotted line) are orders of magnitude higher than that of the mSSA variant we propose (blue solid line).}
    \label{fig:mssa_vs_hankelmssa_training_time}
\end{figure}
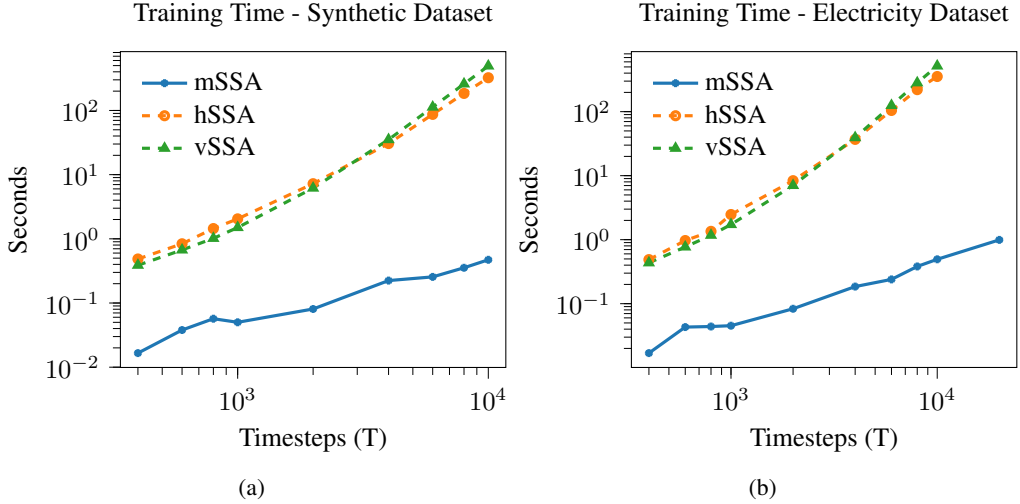

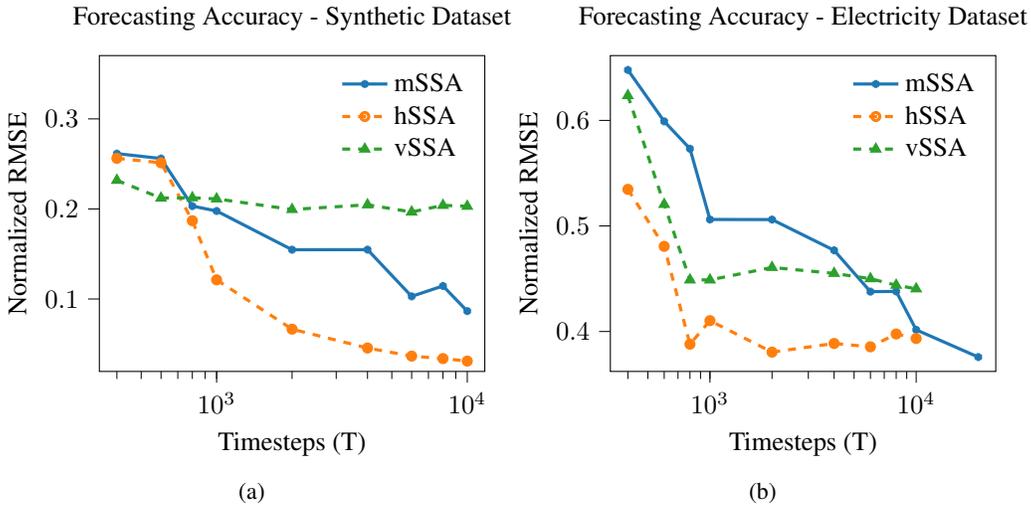
\begin{figure}[h!]
\centering
  \begin{subfigure}{0.40\textwidth}
\begin{tikzpicture}
\definecolor{color0}{rgb}{0.12156862745098,0.466666666666667,0.705882352941177}
\definecolor{color1}{rgb}{1,0.498039215686275,0.0549019607843137}
\definecolor{color2}{rgb}{0.172549019607843,0.627450980392157,0.172549019607843}
\begin{axis}[
legend cell align={left},
  width = \linewidth,
legend style={
  fill opacity=0.0,
  draw opacity=1,
  text opacity=1,
  draw=none
},
log basis x={10},
tick align=outside,
tick pos=left,
title={Forecasting Accuracy - Synthetic Dataset},
x grid style={white!69.0196078431373!black},
xlabel={Timesteps (T)},
xmin=340.535969008314, xmax=11746.1894308802,
xmode=log,
xtick style={color=black},
y grid style={white!69.0196078431373!black},
ylabel={Normalized RMSE},
ymin=0.01955485315, ymax=0.370296103585,
ytick style={color=black}
]
\addplot [very thick, color0, mark=asterisk, mark size=1.5, mark options={solid}]
table {%
400 0.261442573
600 0.256003707202008
800 0.203185617563666
1000 0.197718804141371
2000 0.15472493533954
4000 0.154786601536842
6000 0.102898727
8000 0.114503632
10000 0.086525123
};
\addlegendentry{mSSA}
\addplot [very thick, color1, dashed, mark=*, mark size=1.5, mark options={solid}]
table {%
400 0.256225217
600 0.25123582
800 0.186925865
1000 0.12127167
2000 0.066509288
4000 0.045553443
6000 0.036576606
8000 0.033824169
10000 0.031073316
};
\addlegendentry{hSSA}
\addplot [very thick, color2, dashed, mark=triangle*, mark size=1.5, mark options={solid}]
table {%
400 0.231812572
600 0.212062484
800 0.212402001
1000 0.211165402
2000 0.199401504
4000 0.204596591
6000 0.196629628
8000 0.203987912
10000 0.202987912
};
\addlegendentry{vSSA}
\end{axis}
\end{tikzpicture}
\caption{}
\label{fig:syn_accuracy}
 \end{subfigure}
  \hspace{4mm}
   \begin{subfigure}{0.40\textwidth}
\begin{tikzpicture}
\definecolor{color0}{rgb}{0.12156862745098,0.466666666666667,0.705882352941177}
\definecolor{color1}{rgb}{1,0.498039215686275,0.0549019607843137}
\definecolor{color2}{rgb}{0.172549019607843,0.627450980392157,0.172549019607843}

\begin{axis}[
legend cell align={left},
  width = \linewidth,
legend style={fill opacity=0.0, draw opacity=1, text opacity=1, draw=none},
log basis x={10},
tick align=outside,
tick pos=left,
title={Forecasting Accuracy - Electricity Dataset},
x grid style={white!69.0196078431373!black},
xlabel={Timesteps (T)},
xmin=328.936063770756, xmax=24320.8358131732,
xmode=log,
xtick style={color=black},
y grid style={white!69.0196078431373!black},
ylabel={Normalized RMSE},
ymin=0.36217997505, ymax=0.66139095195,
ytick style={color=black}
]
\addplot [very thick, color0, mark=asterisk, mark size=1.5, mark options={solid}]
table {%
400 0.647790453
600 0.599088419844481
800 0.573203809
1000 0.506131502
2000 0.506038926703504
4000 0.476921176529033
6000 0.437772122
8000 0.437965621148883
10000 0.401613807
20000 0.375780474
};
\addlegendentry{mSSA}
\addplot [very thick, color1, dashed, mark=*, mark size=1.5, mark options={solid}]
table {%
400 0.534651367
600 0.48066074
800 0.387990903
1000 0.410288973
2000 0.380463595
4000 0.388664761
6000 0.385501263
8000 0.397650134
10000 0.393362316
};
\addlegendentry{hSSA}
\addplot [very thick, color2, dashed, mark=triangle*, mark size=1.5, mark options={solid}]
table {%
400 0.623493682
600 0.520438475
800 0.448877317
1000 0.448873216
2000 0.460669535
4000 0.455063744
6000 0.450056081
8000 0.443863593
10000 0.440422217
};
\addlegendentry{vSSA}
\end{axis}
\end{tikzpicture}
\caption{}
\label{fig:elec_accuracy}
 \end{subfigure}
  \caption{The forecasting error of the original mSSA variants (hSSA in the orange dotted line and vSSA in the green dotted line) and the proposed mSSA variant (blue solid line) as we increase $T$.}
    \label{fig:mssa_vs_hankelmssa}
\end{figure}

\vspace{200mm}

\newpage

\end{appendix}



\bibliographystyle{imsart-number} 
\bibliography{bibliography}       

\end{document}